\documentclass[twoside,11pt]{article}

\usepackage{jmlr2e}

\usepackage[american]{babel}

\usepackage{mathtools} 
\usepackage{booktabs} 
\usepackage{tikz} 

\usepackage[utf8]{inputenc} 
\usepackage[T1]{fontenc}    
\usepackage{amsmath}
\usepackage{bbm}
\usepackage{nicefrac}       
\usepackage{microtype}      
\usepackage{xcolor}         
\usepackage[export]{adjustbox}
\usepackage{wrapfig}
\usepackage{caption}
\usepackage{subcaption}
\usepackage{enumitem}
\usepackage{algorithm}
\usepackage{algpseudocode}

\newtheorem{thm}{Theorem}
\newtheorem{lem}[theorem]{Lemma}
\newtheorem{assumption}{Assumption}

\renewcommand{\SS}{\mathcal{S}}
\renewcommand{\AA}{\mathcal{A}}
\newcommand{\PP}{\mathbbm{P}}
\newcommand{\RR}{\mathcal{R}}
\newcommand{\EE}{\mathbbm{E}}
\newcommand{\EED}{\mathbbm{E}_{\mathcal{D}}}

\newcommand{\norm}[1]{\left\|#1\right\|}
\newcommand{\tr}{\top}
\newcommand{\avg}{\bold{avg}_m(x_t)}
\newcommand{\avgx}{\bold{avg}_m(x)}

\newcommand{\oneexptd}{R1-GTD\,}

\ShortHeadings{A new Gradient TD algorithm with only one step-size}{**}
\firstpageno{1}

\begin{document}

\title{A new Gradient TD Algorithm with only One Step-size:\\  Convergence Rate Analysis using $L$-$\lambda$ Smoothness
}

\author{\name Hengshuai Yao\email hengshuai.yao@sony.com \\
       \addr SonyAI
       }

\editor{***}

\maketitle

\begin{abstract}
Gradient Temporal Difference (GTD) algorithms \citep{gtd,tdc} are the first $O(d)$ ($d$ is the number of features) algorithms that have convergence guarantees for off-policy learning with linear function approximation. 
\citet{bo_gtd_finite} and \citet{dalal2018finite_twotimescale} proved the convergence rates of GTD, GTD2 and TDC are $O(t^{-\alpha/2})$ for some $\alpha \in (0,1)$. This bound is tight \citep{dalal2020tale_twotimescale}, and slower than $O(1/\sqrt{t})$. 
GTD algorithms also have two step-size parameters, which are difficult to tune. In literature, there is a ``single-time-scale'' formulation of GTD. However, this formulation still has two step-size parameters.

This paper presents a truly single-time-scale GTD algorithm for minimizing the Norm of Expected (TD) Update (NEU) objective, and it has only one step-size parameter.  
We prove that the new algorithm, called Impression GTD, converges to the optimal solution in $O(1/t)$ rate with a constant step-size. Furthermore, based on a generalization of the expected smoothness \citep{gower2019sgd_general}, called $L$-$\lambda$ smoothness, we are able to prove that the new GTD with large constant step-sizes converges, in fact, with a {\em linear rate}, to a biased solution. 
Our rate actually also improves \citeauthor{gower2019sgd_general}'s result with a tighter bound under a weaker assumption.  
Besides Impression GTD, we also prove the rates of three other GTD algorithms, one by \citet{ptd_yao},  another called A$^\tr$TD \citep{gtd}, and a counterpart of A$^\tr$TD. 
It appears that these four algorithms are only different in how they build certain data structures from the buffers. 
The convergence rates of all the four GTD algorithms are proved in a {\em single} generic GTD framework to which $L$-$\lambda$ smoothness applies. Empirical results on Random walks, Boyan chain, and Baird counterexample show that Impression GTD converges much faster than existing GTD algorithms for both on-policy and off-policy learning problems, with well-performing step-sizes in a big range. 
\end{abstract}


\section{Introduction}\label{sec:introduction}
Off-policy learning is an important learning paradigm in reinforcement learning \citep{sutton2018reinforcement}. An agent selects actions according to a policy (called the {\em behavior policy}), and in the meanwhile, the algorithm evaluates another policy (the {\em target policy}).  
When the two policies are the same, it is called {\em on-policy learning} or simply policy evaluation if the context is clear. When they are different, the problem is called {\em off-policy learning}. Temporal Difference (TD) methods \citep{td} are guaranteed to converge for on-policy learning with linear function approximation \citep{dayan1994td,tsi_td,csaba_book,bertsekas2012dynamic}. However, bootstrapping, such as TD(0), is problematic for off-policy learning because the methods can diverge with function approximation even in the linear case \citep{bertsekas1995counterexample,boyan1994generalization,baird1995residual,gordon1995stable,tsi_td}. Off-policy learning, bootstrapping, and function approximation are problematic for reinforcement learning, and they are often referred to as the ``deadly triad'' \citep{sutton2018reinforcement}.

The Gradient TD family \citep*{gtd,tdc} is an important class of off-policy learning algorithms by stablizing the Ordinary Differential Equation (O.D.E.) underlying the TD update. For example, GTD \citep{gtd} is based on minimizing the NEU objective \citep{tdc}, which gives a stable O.D.E. whose underlying system is essentially a normal equation, transforming the TD update into a system with symmetry. The high computation efficiency of the GTD algorithms makes themselves appealing to learn many policies in parallel from a single stream of data, such as the Horde architecture \citep*{sutton2011horde} and universal off-policy evaluator \citep*{chandak2021universal_off_eval} (with an experiment for type-1 diabetes treatment via simulation).

The GTD algorithms do not suffer from high variances like importance sampling and Emphatic TD methods which are reviewed in Section \ref{sec:background}.  
However, GTD algorithms need further improvement as well because they are not easy to use. Also see \citep*{martha2020gradient} for the ``difficult-to-use'' problem of GTD algorithms from the authors who has used the algorithms for a long time  \citep{sutton2011horde}. GTD algorithms are usually $O(d)$ and this high computation efficiency is due to a two-time-scale formulation of the algorithm in order to mitigate the ``two-sample'' problem for approaching the O.D.E. solution in a stable manner.\footnote{Note the GTD algorithm was formulated in a two-time-scale update \citep{gtd}, but the proof works in a setting that has actually one time scale because the ratio between the two step-sizes is a constant instead of a standard diminishing rate in a strict two-time-scale framework. Also see, e.g., \citep{bo_gtd_finite}. In this constant-ratio formulation, which is called the ``single-time scale'' by literature \citep{martha2020gradient,dalal2020tale_twotimescale},  although the update rates of the two iterators are of the same order, there are still two step-size parameters. The performance of the algorithm is dependent on the step-size ratio as well. The same holds for GTD2 \citep{tdc}. However, the TDC proof was under the standard two-time-scale step-size condition \citep{tdc}. 
} The two-time-scale formulation 
comes with the price of an additional step-size parameter, and
intensive tuning efforts are required in practice in order to have a stable and fast convergence. This poses a great challenge to practitioners. A large family of step-size adaption methods in both supervised learning and reinforcement learning may mitigate the issue, which is however not the scope of this paper. In this paper, we focus on the problem formulation for off-policy learning, and re-examine the necessity of resorting to two time scales for Gradient TD algorithms.   

\citet{martha2020gradient} discussed that the saddle-point formulation of
the MSPBE \citep{bo_gtd_finite} can be utilized to view GTD2 as a single-time-scale update with the joint weight vector from the main and the helper iterators. This enables their TDRC to use $1/\eta$, the ratio between the main step-size and the helper step-size, to be one, i.e., essentially reducing to one step-size. This paper can be viewed as a further continuation of the motivation of TDRC. Furthermore, we examine a few issues that still remain to be solved. For example, the second condition of their Theorem 3.1 requires a condition for $\eta$ to be bigger than the negative of the minimum eigenvalue of some matrix, which is the maximum of some positive eigenvalue. This $\eta$ is much bigger than one in most cases. Thus this theorem still has the flavor of two time scales. 
Our motivation in this paper is to develop a policy evaluation algorithm that has the following desiderata:
\begin{enumerate}
    \item The algorithm should be convergent for off-policy learning. 
    \item It should be stable and guaranteed to converge with bootstrapping.
    \item The convergence is also guaranteed with linear function approximation.
    \item The complexity of the algorithm should be linear in the dimensionality of the weight vector.
\item 
The algorithm should have only one step-size parameter.\footnote{Our work is not to be confused with the GTD formulation in which the ratio between the two step-sizes is a constant.}
\item The algorithm should not suffer from high variances and it should converge much faster than existing GTD algorithms. 
\end{enumerate}
The first four are the continuing goals from existing GTD algorithms. The last two are additional goals that we seek in this paper.   
In the literature there is a belief that off-policy learning ``inherently'' has high variances and it is ``common'' when we conduct off-policy learning, e.g. see \citep*{mahmood2015etd,sutton2018reinforcement}. One point we make in this paper is that this is not necessarily true, at least for achieving the off-policy TD solution. 

Nonetheless, we agree that it seems a good summary of our hard lessons for off-policy learning, for which a few counterexamples for TD were proposed in 1990s \citep{boyan1994generalization,bertsekas1995counterexample,baird1995residual,tsitsiklis1996feature,bertsekas1996neuro}. This motivated importance sampling for off-policy learning \citep{precup2000eligibility,precup2001off}, which solves the divergence issue, however with high variances. Gradient TD algorithms were developed afterwards. They enjoy guaranteed convergence under mild conditions \citep{gtd,tdc},  without the issue of importance sampling TD algorithms. However, GTD algorithms have two step-sizes, which make the algorithms hard to use in practice. Emphatic TD is a recent new off-policy learning approach, which incrementally corrects the sample distribution towards a stable O.D.E underlying the update \citep{etd,mahmood2015etd}. However, the variances of ETD are huge \citep{sutton2018reinforcement}. Empirical results on simple domains showed that stable performance of ETD was only obtainable with tiny step-sizes \citep{ghiassian2017first_etd}, which gave extremely slow convergence.      
This paper continues the exploration of first-order off-policy learning and the linear-complexity triumph by GTD methods. We propose a single-time-scale formulation so that there is only one step-size to tune or adapt in practice, without the high-variance price of importance sampling or emphatic TD algorithms. A very recent work by \citet{shangtong_imgtd} has the same motivation as the present paper, which we are going to discuss in Section \ref{sec:imgtd}.  

The paper is organized as follows. In Section \ref{sec:background}, we review the background on the MDP framework, TD, GTD, GTD2 and TDC, and some latest progress. Section \ref{sec:imgtd} presents the basic version of Impression GTD, and Section \ref{sec:minibatchPE} extends the algorithm to the minibatch off-policy evaluation. Section \ref{sec:experiments} contains the empirical results of Impression GTD for on/off-policy learning. Finally, Section \ref{sec:conclusion} concludes the paper.

The main theoretical results of this paper are in Section \ref{sec:theory}. We analyze the convergence rates of Impression GTD under constant step-sizes. Our analysis is conducted on a generic GTD algorithm that includes Impression GTD, Expected GTD, A$^\tr$TD, and a counterpart of A$^\tr$TD. 
We first show that all the four GTD algorithms converge to the optimal solution at a rate of $O(1/t)$, with a constant step-size that depends on the largest Lipschitz constant of the stochastic gradient. 

Furthermore, we show that with larger constant step-sizes, the generic GTD algorithm converges to a neighborhood of the optimal solution, at a linear rate which depends on the variances in the feature transitions, $\ell_2$ norm of the mean of the feature transition matrix ($A$), and the batch sizes of two buffers. This result is achieved by first establishing a SGD rate under a condition for the loss function and the sampling distribution for the stochastic loss, which we call the  $L$-$\lambda$ smoothness. 
We show that our GTD problem formulation (i.e., the NEU objective and a novel sampling method called independent sampling) satisfies the $L$-$\lambda$ smoothness, to which our main SGD rate result applies.

\section{Background, Related Work \& Discussions}\label{sec:background}
In this section, we review MDPs for the problem setting. We also review GTD, GTD2 and TDC algorithms for off-policy learning, as well as some latest works for a discussion. 

\subsection{MDP}
Assume the problem is a Markovian Decision Process (MDP) \citep{bertsekas2012dynamic,sutton2018reinforcement,csaba_book}. The state space is $\SS$ with $N$ possible states. For simplicity, we denote the space by $\SS=\{1, 2, \ldots, N\}$, and a state in $\SS$ is denoted by the integer $i$, and a state sample by $s$. The action space is $\AA$ and an action is denoted by $a$. 
Let $\PP$ be a probability measure assigned to a state $i\in \SS$, which we denote as $\PP(\cdot|i,a)$. Define $R: \SS\times \AA \to \RR$ as the reward function, where $\RR$ is the real space. The reward from a state $i$ to state $j$ is denoted by $r(i,j)$.  
Let $\gamma \in (0,1)$ be the discount factor. 
 
We consider the general case of stochastic policies. Denote a stochastic policy by a probability measure $\pi$ applied to a state $s$: $\pi(\cdot|s) \to [0,1]$. 
At a time step $t$,
the agent observes the current state $s_t$ and takes an action $a_t$. The environment provides the agent with the next state $s_{t+1}$ and a scalar reward $r_{t+1}=R(s_t, a_t)$. The main task of the policy evaluation agent is to approximate the value function that is associated with a policy $\pi$:
\begin{gather*}
  V_\pi(s) = \EE _{\pi}\Big[\sum_{t=0}^{\infty } \gamma^t  r_{t+1} \Big],\notag
\end{gather*}
where $a_t \sim \pi(\cdot|s_t)$ and $s_{t+1} \sim \PP(\cdot|s_t,a_t)$ for all $t\ge 0$. 

\subsection{The NEU objective, Expected GTD and GTD}
Temporal Difference (TD) methods \citep{td} are a class of algorithms for approximating the value function of a policy $\pi$. Using a number of features, we represent the approximation by $\hat{V}^\pi (s)=\phi(s)^\tr  \theta$. Given a transition sample $(s_t, r_t, s_{t+1})$ following $\pi$, the TD($0$) algorithm updates the weight vector by
\[
\theta_{t+1} = \theta_t + \alpha \delta_t \phi_t; \quad \phi_t=\phi(s_t)
\]
where $\delta_t=r_t + \gamma \phi_{t+1}^\tr \theta_t - \phi_t^\tr \theta_t$, called the TD error. Under mild conditions, TD methods are guaranteed to converge to a solution to a linear system of equations, $
\EE[\delta_t\phi_t] = 0$. However, when the distribution of the states is not from policy $\pi$ (the so-called ``off-policy'' learning problem), TD methods can diverge, e.g., see \citep{baird1995residual,sutton1995generalization,bertsekas1996neuro,tsi_td}. 

The Gradient temporal difference (GTD) algorithm \citep{gtd} is guaranteed to converge for off-policy learning with linear function approximation. It minimizes the NEU objective \citep{tdc}, the $\ell_2$ norm of the expected TD update:
\begin{align}
\bold{NEU}(\theta) &= \norm{\EE[\delta\phi]}^2 \nonumber\\
&= (A\theta + b)^\top (A\theta + b), \label{eq:Ab}
\end{align}
where $\EE[\delta\phi]=A\theta + b$, with $A=\EE[\phi(\gamma \phi' - \phi)^\top]$ and $b=\EE[\phi r]$.\footnote{We focus on TD(0) in this paper.} The expectation operator is taken for a transition tuple, $(\phi, r, \phi')$, which follows policy $\pi$. Note that the distribution of the state underlying $\phi$ is not necessarily the stationary distribution of $\pi$ (which would be on-policy). We follow the convention of defining the matrix $A$ in this format as in \citep{tsi_td,bertsekas1996neuro}, which is negative definite in the on-policy case.  

The NEU objective function first appeared in \citep{ptd_yao}. For off-policy learning, the matrix $A$ is not necessarily negative definite. This is the source of the instability and divergence troubles with TD methods. The nice property of NEU is that it introduces a stable O.D.E., by bringing a symmetry into the underlying system. The gradient of the NEU objective is $A^\tr (A\theta +b)=0$, in contrast to the original O.D.E. of TD(0), which is, $A\theta +b=0$. As long as $A$ is non-singular (but it is not necessarily negative definite), the gradient descent update, $\theta_{\tau+1} = \theta_\tau -\alpha A^\tr (A\theta_\tau +b)$ is stable and convergent for some positive step-size $\alpha$. It gives the same solution as the original one to $A\theta+b=0$, i.e., the so-called ``LSTD solution'' \citep{lstd} or ``TD solution''.  Note that this reinforcement learning approach is very special in that this method is usually not used for solving linear system of equations in iterative methods, because the system $A^\tr (A\theta_\tau +b)=0$, called the {\em normal equation}, induces slower iterations than the original one. This is due to that $A^\tr A$ has worse conditioning than $A$ (the condition number is squared). The normal equation is usually only used in iterative methods when $A$ is not a square matrix, e.g., in an over-determined system. The context of off-policy learning makes this method meaningful because the necessity of a stable O.D.E. even though matrix $A$ is a square matrix here.        

\citeauthor{ptd_yao} has a 
 gradient descent algorithm for TD too. It is $O(d^2)$ because it builds data structures in the form of a matrix and a vector from the samples \citep{bradtke1996linear,lspe96,lstd,lspe03} or the inversion of the matrix \citep*{xu2002efficient}. GTD \citep{gtd} reduces the complexity of \citeauthor{ptd_yao}'s algorithm to $O(d)$ by a two-time scale stochastic approximation trick. 

To understand how GTD makes $O(d)$ computation possible, let's start with the least-mean-squares (LMS) algorithm. Let $\phi$ be the feature vector we observe every time step, and $y$ is the output. Then the LMS update, $\Delta \theta =-(\phi^\tr \theta - y) \phi$, converges to the solution to the linear system, $A\theta = b$ where $A =  \EE[\phi\phi^\tr ]$, and $b= \EE[\phi y]$. LMS is a stochastic gradient method that is $O(d)$ per time step. TD methods are similar in this regard. Let $\phi$ be the feature vector we observe every time step, $r$ is the reward and $\phi'$ is the next feature vector. 
For TD(0), see equation \ref{eq:Ab} for the definition of $A$ and $b$. 

Now we aim for an $O(d)$ method that approximates $\Delta \theta = A^\tr (A\theta +b)$, which is just the stationary form of \citeauthor{ptd_yao}'s gradient descent algorithm. 
What would be a good sample of $A^\tr $? That would be $(\gamma\phi' - \phi)\phi^\tr $. How to get an estimate of $A\theta +b$? We already have this in the TD algorithm. That is $\delta\phi$. So it's a question of putting these two estimations together. 
Note that here just putting them together by multiplication does not work, because the expectation of the two terms cannot be taken individually: they are dependent on each other.\footnote{ \citet{shangtong_imgtd} started with this same observation too (independent work). \citet{gtd} also said that if we sample both
of the terms to form a product, then the result will be biased by their correlation. This arises from the well-known double-sampling issue in reinforcement learning. \citet*{mengdi_stochastic_2exp} considered a general class of optimization problems that involve the composition of two expectation operators, for which SGD does not apply. }  
GTD's idea is to slow down the second estimation, $A\theta +b$, in such a way that we don't use the latest transition. 
This was done by estimating the TD update separately, introducing a helper vector $u$:\footnote{The GTD paper \citep{gtd} has a typo in their equation (4). With $\EE[\delta \phi]$ defined therein, the algorithm updates in equation 8 and equation 10 would be unstable. 
}
\begin{align}
\theta_{t+1} &= \theta_t - \alpha_t (\gamma \phi_t' - \phi_t)\phi_t^\tr u_t, \nonumber \\
u_{t+1} &= u_t + \beta_t (\delta_t \phi_t - u_t), \label{eq:gtd}
\end{align}
where $\alpha_t, \beta_t>0$ are two step-size parameters.
Note that $u_t$ can be viewed as a historical average of the TD update. 
In \citet{ptd_yao}, the following gradient descent algorithm was proposed:
\begin{equation}\label{eq:expectedGTD}
\theta_{t+1} = \theta_t - \alpha_t A_t^\tr (A_t\theta_t + b_t), 
\end{equation}
where $A_t$ and $b_t$ are consistent estimations of $A$ and $b$, which are guaranteed to converge due to the law of large numbers \citep{tadic_td}. We call the algorithm in equation \ref{eq:expectedGTD} the {\em Expected GTD} algorithm because it is a GTD algorithm from the expected update under the empirical distribution. 

In \cite{gtd}, they discussed an alternative algorithm that applies $A_t^\tr $ to the sample of $A \theta_t + b$ (which is the TD update, $\delta_t\phi_t$). It is a hybrid between TD and Expected GTD:
\begin{equation}\label{eq:attd}
\theta_{t+1} = \theta_t -\alpha_tA_t^\tr \delta_t\phi_t, 
\end{equation}
which was called $\mbox{A}^\tr$TD \citep{gtd}. There is another hybrid 
\begin{equation}\label{eq:R1-GTD}
\theta_{t+1} = \theta_t -\alpha_t (\gamma \phi_{t+1} - \phi_t)\phi_t^\tr (A_t\theta_t + b_t), 
\end{equation}
which we call {\em Rank-1 GTD} or {\em R1-GTD} for short. $\mbox{A}^\tr$TD  and R1-GTD are counterparts that apply sampling either to the TD update or to the preconditioner. The rank-1 matrix applies for the purpose of stabilizing the TD update on average.   

About the signs in the updates of $\theta$ and $u$: $u$ is the averaged TD update, which gives $\EE[\delta \phi]= A\theta + b$ in the long run. In the update for $\theta$, the rank-1 matrix is a sample of $A^\tr $. Together with $u$, the expected update provides an unbiased estimate for $A^\tr (A\theta+b)$. It thus makes sense to use the minus sign in the update of $\theta$, instead of the positive sign in the GTD paper \citep{gtd}, following the convention of gradient descent.   

The complexity reduction with GTD is nicely delivered. It is $O(d)$ and guaranteed to converge for off-policy learning. However, there is a non-trivial practical problem. In particular, to decouple the two terms in the expectation, the price we paid is an additional update, requiring us to tune two step-size parameters when using the GTD algorithm in practice. 

\cite*{maei2009convergent_nonlinear} generalized the MSPBE to the nonlinear function approximation, and proved the convergence of the generalized GTD2 and TDC to a local minima. The LMS2 algorithm is an extension of the idea of GTD to supervised learning \citep*{lms2}. 

The neural GTD algorithm \citep*{neural_gtd2} is a projected primal-dual gradient method. It has only one step-size, however, with the addition of a projection ball operation (probably needed to keep the algorithm bounded). The algorithm is actually more similar to GTD2, because the helper iterator reduces to the same as that of GTD2 in the linear case.

\subsection{The MSPBE objective, GTD2 and TDC}\label{tdc}
The algorithms of GTD2 and TDC (TD with a correction) were derived using the MSPBE (mean-squared Projected Bellman Error) \citep{tdc}. They showed that for on-policy evaluation, they are faster than GTD. However, the reason of the speedup was not explained by the paper and not known to the literature. Here we provide a simpler derivation for the two algorithms and it also explains why they are faster.  

{\bfseries GTD2}. 
In the first iterator of the GTD update, we used the averaged TD update, $u_t$. Now let's see if we can speed up $u_t$. This can be done by applying a preconditioner to the update of $u_t$. In particular, if we replace the second iterator $u_t$ with 
\[
\bar{w}_{t+1} =  C^{-1}u_t. 
\]
The expected behavior of $\bar{w}$ is then described by the following iteration:
\[
\bar{w}_{\tau+1} = C^{-1} (A \theta_\tau +b ). 
\]
Note that $h =  C^{-1} (A \theta_\tau +b ) $ is just the O.D.E. underlying LSPE \citep*{lspe96,lspe03,lspe04}. \citet{ptd_yao} also showed that LSPE is a preconditioning technique. 
GTD2 and TDC can be derived using two ways of writing the O.D.E. 

Using the form of $h$ above, we can solve $\bar{w}$ with stochastic approximation (which is just the LMS algorithm), treating the TD error, $\delta_t$, as the target signal and predicting it with the feature vector. This leads to GTD2:
\begin{align}\label{eq:gtd2}
    \theta_{t+1} &= \theta_t - \alpha_t (\gamma\phi_{t}' - \phi_t)\phi_t^\tr w_t, \nonumber \\
    w_{t+1} &= w_t - \beta_t(  \phi_t^\tr  w_t-\delta_t)\phi_t. 
\end{align}
The O.D.E. for the $\theta$ update is $h'= A^\tr  C^{-1} (A \theta_{\tau} + b)$. 
Thus the underlying matrix is symmetric and the stability of the system can be achieved provided that (1) $A$ is non-singular (but not necessarily negative definite, which is the case of general off-policy learning); and (2) $C$ is symmetric and positive definite.
This O.D.E. is just the gradient of the MSPBE objective \citep{tdc}. In the matrix-vector form, it can be written as
\[
\bold{MSPBE}(\theta)= (A \theta + b)^\tr  C^{-1} (A \theta + b).
\]

{\bfseries TDC}. 
Let's write $h$ in another form. Note that $A$ can be split into two parts, $A = D - C$, where $D=\gamma \EE[\phi \phi'^\tr ]$. Thus $h= C^{-1}((D-C)\theta +b)=-\theta + C^{-1}(D\theta + b)$. 
Thus we can simply apply stochastic approximation again to solve $h$ incrementally. This is a LMS procedure too. This time we treat the stochastic sample given by $\gamma \phi_t'^\tr \theta + r$, as the target for regression. This leads to the following update:
\begin{align}
    \theta_{t+1} &= \theta_t - \alpha_t (\gamma\phi_{t}' - \phi_t)\phi_t^\tr w_t, \nonumber \\
w_{t+1} &= w_t - \beta_t \theta_t + \beta_t(\gamma\phi_t'^\tr \theta_t +r_t - \phi_t^\tr w_t)\phi_t.  
\end{align}
This is another form of TDC \citep{tdc}. In the original form of TDC, $\theta$ has a different iteration from GTD2 while $w$ is the same as that in GTD2. This form is a ``transposed'' version of TDC: the $\theta$ update is the same as GTD2, while the $w$ update is different. 

To derive the original TDC, we start with the same transformation but this time to $h'$:
\begin{align*}
h' & =(A^\tr  C^{-1}) (A \theta_{\tau} + b) \\
& = (D^\tr  -C)C^{-1} (A \theta_{\tau} + b)\\
&= -(A \theta_{\tau} + b) + D^\tr  \left(C^{-1} (A \theta_{\tau} + b)\right).
\end{align*}
The first term is just the expected update of TD. 
The second term can be approximated by breaking the rank-1 matrix vector product, and not forming the matrix explicitly \citep{tdc}. 
Note that $\gamma \phi_t'\phi_t^\tr$ is a sample of $D^\tr$. 
This O.D.E. derives the $\theta$ update of the TDC algorithm, while the helper update remains the same as GTD2:
\begin{align}\label{eq:tdc}
    \theta_{t+1} &= \theta_t - \alpha_t \left[- \delta_t(\theta_t)\phi_t + \gamma \phi_t' (\phi_t^\tr  w_t)\right], \nonumber \\
    w_{t+1} &= w_t - \beta_t(\phi_t^\tr  w_t-\delta_t)\phi_t. 
\end{align}

These are real-world applications of the preconditioning technique from iterative algorithms and numerical analysis \citep{saad03:IMS,horn2012matrix,golub2013matrix} to reinforcement learning. Accelerated learning experiments can be found in \citep{ptd_yao}, covering TD, iLSTD \citep*{ilstd} and LSPE, which shows the spectral radius of the preconditioned iterations improves over expected GTD and TD. 
GTD2 and TDC were shown to converge faster than GTD, and TDC is slightly faster than GTD2, for on-policy learning problems \citep{tdc}. 

\citet{baird1995residual}, 
\citet*{scherrer2010should_bellerr}, 
\citet*{sutton2018reinforcement}, \citet*{zhang2019_rg} and \citet*{patterson2021investigating_etd} had good discussions on the learning
objectives for off-policy learning. In some sense, the Bellman error is indeed a tricky objective to minimize because it involves two expectation operators. In particular, take the mean squared Bellman error for example, where the transition follows the policy $\pi$ and the dynamics of the MDP, $\EE[r+\gamma V(s') - V(s)]^2= \EE[\EE(r+\gamma V(s') - V(s))^2 |s]$. 
In the nonlinear and off-policy i.i.d. case, the inside, the conditional expectation is problematic. One either needs access to a simulator, by resetting it to the same state we just proceeded from there, or hopes the environment is deterministic. This needs {\em two independently sampled successor states}, which is the so-called {\em double-sampling} problem, a well-known challenge in reinforcement learning. Sampling two independent successors is not practical in online learning and other scenarios, because we cannot go back in time. \citet{ran_gtd} circumvents this issue with a two-time scale approach, by generalizing the ``waiting'' idea of GTD and GTD2/TDC. They proposed a few algorithms that are based on residual gradient \citep{baird1995residual}, regularization \citep*{hastie2009elements}, and momentum \citep{polyak1992acceleration}. It is also possible to extend GTD, TDC and our method to add momentum. One of their algorithms performs much faster than GTD2. However, the algorithm has four hyper-parameters, and the two step-sizes used in experiments do not satisfy the two-time scale requirement and thus their empirical results are not covered by their theory. The helper iterator was actually updated much slower than the main iterator. In this paper, we focus on momentum-free algorithms and our method has convergence guarantees with only one step-size parameter.  

Policy evaluation algorithms are generally in the dimensions of first-/second-order and on-/-off policy learning. In particular, the second-order TD method is on-policy and off-policy {\em invariant}, in contrast to the diverse forms of first-order off-policy TD algorithms, which all have different updates from the online TD methods.  
To be concrete, let's consider an off-policy learning algorithm that minimizes a generic loss of the form, 
\[
E(\theta) = \EE[\delta\phi] ^\top U^{-1}\EE[\delta\phi],
\]
where $U$ is any S.P.D matrix.

The derivation by \citet*{pan2017accelerated_atd} follows through without any problem. In particular, let $H$ be the Hessian matrix of $E$. Then the Newton method minimizing $E$ takes the form of (in the expectation)
\begin{align*}
\theta_{\tau+1} &= \theta_\tau - \alpha_\tau H^{-1} \nabla_\theta E|_{\theta=\theta_\tau}\\
&=\theta_\tau -  \left(A^\top U^{-1}A  \right)^{-1} \nabla_\theta E |_{\theta=\theta_\tau}\\
& =\theta_\tau - \alpha_\tau  \left(A^\top U^{-1}A  \right)^{-1} A^\top U^{-1}(A\theta_\tau+b) \\
& = \theta_\tau -  \alpha_\tau A^{-1}  U (A^\top)^{-1}   A^\top U^{-1}(A\theta_\tau+b) \\
& = \theta_\tau -  \alpha_\tau A^{-1} (A\theta_\tau+b) \\
& = \theta_\tau -  \alpha_\tau A^{-1} \EE[\delta(\theta_\tau) \phi ].
\end{align*}
Using the stochastic approximation trick, 
this means for such a generic function $E$, Newton method has a form that is {\em invariant} in $U$:
\begin{equation}\label{eq:2ndordertd}
\theta_{\tau+1} = \theta_\tau -  \alpha_\tau A^{-1} \delta \phi. 
\end{equation}
Why is it invariant in $U$? At a high level, this is because $U$ is an artifact, in particular, a preconditioner that improves the conditioning of the underlying O.D.E. of GTD. Preconditioning, by definition, is to accelerate convergence without changing the solution.

The update \ref{eq:2ndordertd} is exactly the Newton TD method proposed and analyzed by \citeauthor{Yao_direct_preconditioning} for policy evaluation, by using an estimation of matrix $A$. 
\citet{pan2017accelerated_atd} rediscovered this algorithm by minimizing MSPBE for off-policy learning.   
In fact, their derivation will hold for minimizing NEU as well. The Newton TD method minimizing the NEU objective also leads to this update. This can be shown from $E$ by setting $U=I$, the identity matrix. This means while there are a number of diverse first-order off-policy TD algorithms, the second-order TD is invariant both in the sense of on-policy or off-policy, and a generic loss in the form of $E$. 
Thus probably we don't have to differentiate between on-policy or off-policy TD for the second-order methods, especially for Newton. No changes (like the case of the first-order TD methods) are required to make the algorithm in equation \ref{eq:2ndordertd} in order for it to converge for off-policy learning.  

Most of second order TD methods are $O(d^2)$ per time step in computation. In certain problems, when $A$ is sparse or low-rank, one can gain acceleration by taking advantage of the structure, e.g., sparse transitions \citep{ptd_yao} and low-rank approximation \citep{pan2017accelerated_atd}. However, in general, the second-order methods are not as efficient as the first-order methods in computation when deployed online. Readers are referred to a linear-complexity approximate Newton method \citep{ran_gtd}, which accelerates gradient-based TD algorithms for minimizing MSBE.
Our work in this paper is in the thread of $O(d)$, first-order TD for off-policy learning, for which there is no such invariance like that holds for the second-order TD methods. 

\subsection{The Saddle-Point Formulation}
The saddle-point or mini-max formulation of GTD, GTD2 and TDC \citep{bo_gtd} can be derived by observing that the helper iterator in the three algorithms is expected to give a good estimation of the expected TD update (GTD), or a least-squares solution (GTD2 and TDC). Take the helper iterator in GTD (equation \ref{eq:gtd}) for example. We want the helper iterator to get to $A\theta+b$ as close as possible. Thus the following loss containing the inner product will be maximized if $u=A\theta +b$:
\[
L(u|\theta) = u^\tr (A\theta + b) - \frac{1}{2} u^\tr u,
\]
which can be seen from $\nabla_u L(u|\theta)=0$. That is, $\arg\max_u L(u|\theta)=u^*=A\theta+b$. Intuitively, $u$ should be along the direction of $A\theta+b$ (from the inner product), and the magnitudes should be the same too (from the $\ell_2$-norm, which gives the length requirement). Therefore, $L(u^*|\theta) = \frac{1}{2}\norm{A\theta+b}^2=\frac{1}{2}\bold{NEU}(\theta)$. 
This gives a mini-max formulation of GTD:
\[
\theta^* = \arg\min_{\theta} \max_u L(u|\theta). 
\]
This is a very interesting formulation of GTD. Seeking the saddle-point solution is an important class of problems in optimization, e.g., see \citep{saddle_point_Nemirovski}. The problem also dates back to game theory from the beginning \citep{gametheory_Neumann}. Many later works on GTD are built on this formulation, e.g., see \citep*{saddle_point_du,mengdi_wang_gtd_like_nasa,saddle_doina_svrg,csaba_gtd_22}. Here we briefly review \citeauthor{mengdi_wang_gtd_like_nasa}'s Nested Averaged Stochastic Approximation (NASA) algorithm.    

The NASA algorithm aims to minimize a nested loss function of the form,\footnote{This nested function is a further extension of the stochastic composition problem \citep{mengdi_stochastic_2exp}.} $\min_\theta f_1(f_2(\theta))$, in a stochastic fashion.
For example, in the GTD setting, $f_1(x)=\norm{x}^2$, and $f_2=A\theta+b$. NASA features in the use of the averaging technique. Let's interpret their algorithm in the GTD setting. For minimizing the NEU objective, we can write their algorithm by the following:
\begin{align*}
 g_t &= \arg\max_g \left\{g^\tr z_t -\frac{\beta_t}{2}\norm{g}^2  \right\}\\
 \theta_{t+1} &= \theta_t - \tau_t g_t\\
 z_{t+1} &= (1-a \tau_t)z_t + a\tau_t (\gamma \phi_t' - \phi_t)\phi_t ^\tr u_t\\
 u_{t+1} & = (1-b\tau_t)u_t + b\tau_t \phi_t \delta_t(\theta_{t+1}),
\end{align*}
where $\delta_t(\theta_{t+1})$ is the TD error realized with the weight vector $\theta_{t+1}$. 
We can start understanding NASA with the simplest connection. The iteration $u_t$ is similar to GTD (equation \ref{eq:gtd}), hereby using $b\tau_t$ to smooth the TD updates. The vector $z$ provides another layer of averaging over the GTD update (i.e., the change in $\theta$ in equation \ref{eq:gtd}), using $a\tau_t$ to smooth. That is, $z$ is expected to get close to the gradient of NEU. For $g_t$, we use here the equivalent $\arg\max$ formulation instead of the original $\arg\min$, to see the saddle-point formulation clearly. The update of $\theta$, as a result, switches to the negative sign, which is the gradient descent style. The major update 2.6 in their Algorithm 1 is an averaging style.

Therefore, in the context of minimizing NEU, the major improvement of NASA over GTD is that there is an additional averaging over the GTD update, and an introduction of $\ell_2$ regularization. NASA also generalizes to minimize other nested loss functions than NEU, which include Stochastic Variational Inequality, and low rank approximation. They proved the almost sure convergence of NASA under the diminishing step-size for $\tau_k$ and constant $a, b$ and $\beta$, for the class of functions of $f_1$ and $f_2$ with Lipschitz continuity in their gradients. Algorithms and analysis of averaged updates over GTD algorithms can also be found in, e.g., \citep{csaba_lin_stochastic18,csaba_gtd_22}. For  analysis on more general averaging algorithms, one can refer to, e.g.,  \citep{polyak1992acceleration,averaging_sgd_lin,svrg}.

\subsection{Related work}
\citet*{bo_gtd_finite} performed the first finite-sample analysis for GTD algorithms, and showed that GTD and TDC/GTD2 are SGD algorithms in the formulation of minimizing a primal-dual saddle-point objective function, with a convergence rate of about $t^{-1/4}$ in terms of value function approximation. 
\citet*{dalal2018finite_twotimescale} established the convergence rates of GTD, GTD2 and TDC under diminishing step-sizes. Later they showed that, with the step-sizes scheduled by $1/t^\alpha$ and $1/t^\beta$ where $0<\beta<\alpha<1$, the convergence rates are $O(t^{-\alpha/2})$ and $O(t^{-\beta/2})$ for the two iterators, and the bounds are tight \citep*{dalal2020tale_twotimescale}. \citet*{xu2019two_time_gtd} had the first non-asymptotic convergence analysis for
TDC under Markovian sampling. 
\citet*{xu2021sample_twotimescale} analyzed the convergence rate of linear and nonlinear TDC with constant step-sizes. \citet*{xu2020finite_q_learning} analyzed the convergence rate of a Q-learning algorithm with a deep ReLU network. Their algorithm also has a projection ball applied to the TD update. \citet{yu2018convergence_gtd} had a comprehensive convergence analysis of GTD and mirror-descent GTD, with an extensive treatment of the eligibility trace under both constant and diminishing step-sizes.

\citet*{gupta2019finite} gave an error bound for stochastic linear two-time scale algorithms with fixed step-sizes. They also derived an adaptive learning rate for the faster iteration. 
The convergence rate of general two-time scale stochastic approximation is studied by \citet*{hong2020two_bilevel_optimization} and \citet*{doan2021finite}. 

An important early off-policy learning exploration is based on importance sampling \citep*{precup2000eligibility,precup2001off}. However, importance sampling algorithms have an inherent problem in the reinforcement learning context. The variance is high due to small probabilities of taking certain actions in the behavior policy, because their products appear in the denominator(s) of certain quantities. The variance of importance sampling ratios may grow exponentially with respect to the time horizon, e.g., see \citep*{xie2019towards_IS}.  
Weighted importance sampling \citep*{mahmood2014weighted} and clipped ratios \citep*{vtrace} can mitigate the issue and reduce the high variances, however, at the price of providing a biased solution. Return-conditioned importance sampling  \citep*{IS_conditioned_return} reduces variances by ruling out the actions that have no effect on the return.  
Some methods are based on the importance sampling over the stationary distributions of behavior and target policies \citep*{hallak2017consistent,liu2018breaking_IS_stationary,xie2019towards_IS,gelada2019off_IS_stationary}, instead of the product of policy ratios. 

Emphatic TD (ETD) \citep*{etd,yu2015convergence_etd}, a non-gradient-based method, has only one step-size in its update rule. However, ETD does not converge to the TD solution and it suffers from high variances. One has to use small step-sizes for ETD, which results in slow convergence \citep*{ghiassian2017first_etd,ghiassian2021empirical_etd}. ETD is still problematic on Baird counterexample due to high variances even though very small step-sizes were used \citep{sutton2018reinforcement}. Interested readers may refer to \citep{hallak2015generalized_etd,gelada2019offpolicy_etd} for bias-variance analysis, and variance reduction \citep{lyu2020variancereduced_etd} on ETD.

 \section{Impression GTD}\label{sec:imgtd}
GTD has two step-sizes. In this section, we introduce a new Gradient TD algorithm that has only one step-size, e.g., see the six design desiderata as discussed in Section \ref{sec:introduction}. 
 
 Our idea is to decouple the two estimations in GTD by a special sampling method that is going to be detailed later. To do this, we use a buffer that stores transitions. At a time step $t$, we sample two i.i.d. transitions from the buffer,   
$(\phi_1, r_1, \phi_1')$ and $(\phi_2, r_2, \phi_2')$. Note the shorthand $\phi_1=\phi(s_1)$ is for some state $s_1 \in \SS$, and $\phi_2=\phi(s_2)$ for some $s_2\in \SS $. 

Our algorithm updates the parameter vector by
\begin{equation}\label{eq:one_update}
\Delta\theta_{t} = - \alpha_t (\gamma\phi'_1 - \phi_1) \bold{sim}(\phi_1,\phi_2)\left[(\gamma\phi_2' - \phi_2)^\tr\theta_t + r_2\right],
\end{equation}
where $\alpha_t$ is a positive step-size and $\bold{sim}$ is some similarity measure for the two input feature vectors. The update is interesting that the similarity seemingly ``pairs'' the gradient of a TD error on a transition with the TD error on another transition.  

Let's understand this update. If $r_2$ is a big reward, it likely creates a large TD error (the last term in the bracket). This TD error is bridged to adjust $V(s_1)$ and $V(s_1')$. That is, a TD error {\em impresses} another (independent) sample, based on which the parameters are adjusted. 
The bigger is the similarity between the two feature vectors, the larger impression of the TD error from one sample is going to make on the other. We call this new algorithm the {\em Impression} GTD.  

In this paper, we focus on the similarity measure being the correlation between the two feature vectors. 
Let us define $\phi = (\gamma\phi'_1 - \phi_1) \phi_1^\tr\phi_2$. The update can be rewritten into
\begin{align*}
\theta_{t+1} &=  \theta_t - \alpha_t \left[\gamma{\phi_2'}^\tr\theta_t +  r_2 - \phi_2^\tr\theta_t \right]\phi\\
 &=  \theta_t - \alpha_t \left[\gamma V(s_2')+  r_2 - V(s_2) \right]\nabla_{\theta} J.
\end{align*}
where $\alpha$ is the step-size and $\nabla_{\theta} J=\phi$. The overloading notation $\nabla_{\theta} J$ will be explained shortly. 

Interestingly, most incremental $O(d)$ TD algorithms known to the authors update the parameter based on one sample. This algorithm use two independent transitions for the update. It looks like the TD update, but not exactly so (because the transposed term has $\phi_2$ in the first line instead of $\phi$). In fact, it is a modification of the TD(0) update (or the so-called bootstrapping), whose key idea is to treat $V(s_{t+1})$ as a constant target in taking the gradient of the TD error, by combing the two sample transitions to form {\em truly an SGD} algorithm that minimizes the NEU objective.\footnote{\citet{gtd} had a comment that GTD is a SGD method. This is not very precise. GTD is two-time scale, and it is not the standard, single-time-scale SGD. We noted in literature this interpretation of GTD (and GTD2 and TDC) is not rare, e.g., see \citep{td_survey}. See also the discussions on page 35 by \citet{csaba_book}. A better terminology for GTD, GTD2 and TDC may be that they are pseudo-gradient methods as suggested. The exception is when GTD uses exactly the same step-size for the two iterators in the saddle-point formulation \citep{bo_gtd_finite}. Empirical results show that in order for good convergence across domains, one has to use different ratios for the two step-sizes, e.g., see \citep{tdc,martha2020gradient}. For example, Figure 2 of \citep{martha2020gradient} shows that GTD2 generally prefers a larger step-size for the helper iterator in four out of five domains.  However, for TDC, in three domains, it prefers actually slower update for the helper iterator. This is not covered by the theory of two-time scale stochastic approximation. }

There has been a mystery about the function $J$ for decades. In particular, what form should $J$ take for the convergence guarantee of TD methods? The TD methods were developed by treating $V(s_2')$ as the target and taking just $-\nabla_\theta V(s_2)$ as $\nabla_\theta J$. For example, it is common in literature to call $V(s_2')$ (or $\gamma V(s_2')+r_2$ ) the ``TD target'', the essential quantity for bootstrapping \citep{sutton2018reinforcement}.
Treating $V(s_2')$ as the target is also the essential idea for using neural networks for TD methods. For example, \citet{tdgammon}'s TD-gammon is the first such successful example. In DQN, \citet{mnih2015human} used the target network that is a historical snapshot of the network to generate relatively stable targets.  
Counterexamples show that TD can diverge if (1) nonlinear function approximation is used (even for on-policy learning); (2) learning is off-policy (even in the linear case); and (3) bootstrapping (TD methods with the eligibility trace factor smaller than one). This is referred to as the deadly triad \citep{sutton2018reinforcement,zhang2021breaking}.
Historical efforts that research into what form of $J$ guarantees convergence include re-weighted least-squares \citep{bertsekas1995counterexample}, residual gradient \citep{baird1995residual}, and Grow-Support \citep{boyan1994generalization}, etc. These algorithms attempted to derive an algorithm that is either a contraction mapping or a stochastic gradient with the current transition. See also \citep{bo_gtd_finite} for a good discussion and the long history of seeking gradient descent methods for temporal difference learning.  

Our algorithm may imply that this cannot be done with a single sample, if one wants to achieve the TD solution. In order to achieve that, we have to use two samples, in particular, 
\[
\nabla_\theta J = \left[\gamma \nabla_\theta V(s_1') - \nabla_\theta V(s_1)\right] \nabla_\theta V(s_1)^\tr \nabla_\theta V(s_2).
\]
In contrast, residual gradient takes $\nabla_\theta J=\gamma \nabla_\theta V(s_2') - \nabla_\theta V(s_2)$,  calculated on the same transition as where the TD error is computed. 
This shows why the residual gradient algorithm does not converge to the TD solution as discussed by \citet{tdc}. In order to converge to the TD solution, one needs to {\em compute the TD error and the gradient on two different (and independent) samples, also with a similarity measure to bridge them, instead of computing the TD error and the gradient on a single sample}. While the resulting algorithm is indeed an SGD algorithm, the independence sampling mechanism of two samples is different from supervised learning. That is, in supervised learning, one i.i.d. sample suffices for a well-defined SGD update. It has guaranteed convergence (with probability one) to the correct optimum. However, in the reinforcement learning setting, only one sample is not enough for ensuring this unless for deterministic environments.  
Although this still requires (at least) two i.i.d. samples at a time, note that the two samples do not need to be the i.i.d. transitions from the same state, because it is not practical to reset our state to the previous state to start over from there, ``passed is passed''.

Note the above update does not use the reward signal $r_1$. To take advantage of the two transitions, we also perform
\[
\theta_{t+1} =  \theta_t - \alpha_t \left[\gamma{\phi_1'}^\tr\theta_t + r_1 - \phi_1 ^\tr\theta_t\right]\phi,
\]
in which $\phi = (\gamma\phi'_2 - \phi_2) \phi_2^\tr\phi_1$ this time. This is due to that in using the two samples, the operation is symmetric.  
To ensure the two transitions are independent, in sampling we also require that they are from two different episodes. This can be done by an adding 
the episode index for each transition. Ours uses this special and novel sampling method to the best of our knowledge.\footnote{\citet*{shangtong_averagereward_two_iid} considered two i.i.d. samples from a given distribution in an average-reward off-policy learning algorithm, but not in a buffer setting like our method.} To differentiate from the uniform random sampling and prioritized sampling methods widely practised in literature, we call it the {\em independence sampling} method.

The merit of independence sampling and Impression GTD is that together they remove the two steps-sizes and the resulting  tuning efforts and slow convergence. They achieve the decoupling of the two terms in GTD in a novel way. 
From a practical view, carrying a buffer is acceptable. Similar ideas appear in experience replay \citep{lin1992experiencereplay}, 
and deep reinforcement learning \citep{mnih2015human,schaul2015prioritized}. 

Recently, \citet{shangtong_imgtd} developed a GTD algorithm that is very similar to ours as in equation \ref{eq:one_update}. They started with the same observation as ours, in that the gradient $A^\tr$ and the expected TD error in GTD's O.D.E. can be estimated separately. Their algorithm has a buffer as well, but the buffer length does not need to grow linearly as learning proceeds, while our analysis does have such a limitation.  
They also focused on the infinite-horizon setting, and the analysis is very much involved in the discretization of the underlying O.D.E. Our analysis is focused on the episodic problems though occasionally there are also discussions about infinite horizon problems as well. Our independence sampling is also an important ingredient, which facilitates an SGD analysis framework.  
In general, their direct GTD and our Impression GTD can be viewed as algorithms in the same family, with the same motivation and similar algorithmic flavour.   

In a summary, Impression GTD is guaranteed to converge under the same conditions on the MDP and linear features as GTD. Together with direct GTD \citep{shangtong_imgtd}, ours is the first theoretically sound, truly single-time-scale SGD off-policy learning algorithm, with $O(d)$ complexity and one step-size. In Section \ref{sec:theory} and Section \ref{sec:experiments}, we conduct theoretical analysis and empirical studies to show that the new algorithm converges much faster than GTD, GTD2 and TDC. 

We will detail the sampling process in the next section, which also introduces a more general form of this algorithm. 

\section{Mini-batch Policy Evaluation}\label{sec:minibatchPE}
This section further extends the Impression GTD. It is common to use mini-batch training in deep learning and deep reinforcement learning. There the mini-batch training paradigm is necessary mostly because the size of the data sets and the high dimensional inputs. Here we show that it also makes sense to use mini-batch training for off-policy learning, even in the linear case and even the problem size is not big. The motivation of using a buffer here has a different motivation from in deep learning and deep reinforcement learning though it also has the merit of improving sample efficiency and scaling to large problems. In short, the buffer is a tool for decoupling the error and gradient estimations in GTD. 

Let's start with on-policy learning. Suppose we maintain a buffer that is large enough. At each time step, we take an action according to the policy that is evaluated, observing a transition, $(\phi_t, \phi_t', r_t)$. We put the sample into the buffer. Next we sample a mini-batch of samples, $\{(\phi_i, \phi_i', r_i)\}, i=1, 2, \ldots, m$, where $m$ is the batch size.  We then update the parameter vector by the averaged TD update:
\[
\theta_{t+1} = \theta_t + \alpha_t \frac{1}{m}\sum_{i=1}^m \left(\gamma {\phi_i'}^\tr\theta_t +r_t - \phi_i^\tr\theta_t\right)\phi_i.
\]
We call this algorithm the {\em mini-batch TD}.

We follow by extending the Impression GTD for off-policy learning to work with mini-batch sampling. The buffer saves for each sample also the episode index within which a sample is encountered. At a time step, we sample two batches of samples,
\begin{equation}\label{eq:buffer1}
b_1=\{(\phi_i, \phi_i', r_i, e_i)|i=1, 2, \ldots, m_1\}, \quad b_2=\{(\phi_j, \phi_j', r_j, e_j)|j=1, 2, \ldots, m_2\}
\end{equation}
where $e_k$ is the episode index for the $k$th sample. In order for the samples in $b_1$ and $b_2$ to be independent, for any sample index pair, $i$ of $b_1$ and $j$ of $b_2$, we require that they are from different episodes:
\begin{equation}\label{eq:buffer2}
e_i\neq e_j, \quad for \quad \forall i=1, 2, \ldots, m_1; j=1, 2, \ldots, m_2.
\end{equation}
We first generate the averaged TD update from the samples in $b_2$, just like in the mini-batch TD:
\[
\bar{u}_{t} = \frac{1}{m_2} \sum_{j=1}^{m_2} \left(\gamma \phi_j'^\tr\theta_t +r_j - \phi_j^\tr\theta_t\right)\phi_j.
\]
Then for each sample in $b_1$, we compute $\bar{\delta}_t(i)=\phi_i^\tr \bar{u}_t$. 
Finally, the {\em mini-batch Impression GTD} update is
\begin{equation}\label{eq:imgtd}
\theta_{t+1} = \theta_t - \alpha_t \frac{1}{m_1}\sum_{i=1}^{m_1}(\gamma \phi_i' - \phi_i) \bar{\delta}_t(i).
\end{equation}

In the lookup table case,\footnote{In this case, with batch sizes $m_1=m_2=1$, the algorithm is a variant of Baird's RG, equipped with double-sampling. The algorithm converges to the true value function, while RG does not because RG only converges to the correct value function for deterministic MDPs. 
In fact, this is the place where double-sampling and independence sampling meet. Update to the weights happens only when $\phi_i$ and $\phi_j$ are the same, or, the two i.i.d. transitions are from the same state of the MDP. This is rare though, which also shows why mini-batch sampling leads to faster convergence than using batch sizes equal one. This observation was due to James MacGlashan.} this means the bigger is this $\bar{\delta}_t(i)$, the more eligible is this sample for a big update. Thus the update for $\theta(s_i)$ (or $V(s_i)$), and $\theta(s_i')$ (or $V(s_i')$) is big if $\bar{\delta}_t(i)$ is large. Note because $\bar{\delta}_t(i)= u_t(s_i)$ in this case, this largely agrees with prioritized sweeping \citep{moore1993prioritized}. Consider the table lookup case. When $|u_t(s_i)|$ is large, it means the TD update for the the component, $\theta(s_i)$, is big. Thus we can view Impression GTD as a way of adjusting the magnitude of the TD update in the original TD(0) algorithm and update based on the adjusted.

Consider for batch $b_1$, we have only one sample, e.g., the latest online sample, and $b_2$ has $m$ samples. This in fact is the standard online learning paradigm, hereby aided with some historical samples:\footnote{This is actually a ``shrinked'' version of R1-GTD. }
\begin{align*}
\Delta \theta_t &=  - \alpha_t (\gamma \phi_t' - \phi_t) \phi_t^\tr\frac{1}{m} \sum_{j=1}^m \left(\gamma \phi_j'^\tr\theta_t +r_j - \phi_j^\tr\theta_t\right)\phi_j\\
&=  - \alpha_t (\gamma \phi_t' - \phi_t) \frac{1}{m} \sum_{j=1}^m \left(\gamma \phi_j'^\tr\theta_t +r_j - \phi_j^\tr\theta_t\right)\phi_t^\tr\phi_j\\
&=  - \alpha_t (\gamma \phi_t' - \phi_t) \frac{1}{m} \sum_{j=1}^m \delta_j(\theta_t) \phi_t^\tr\phi_j=  - \alpha_t (\gamma \phi_t' - \phi_t) \frac{1}{m} \sum_{j=1}^m \delta_j(\theta_t) \mbox{sim}(s_t, s_j).
\end{align*}
The first line means, if the current feature vector is greatly correlated the averaged TD update from historical samples, the update for $\theta$ is likely to be big for the current transition, to reduce the difference between  $V(s_t)$ and $\gamma V(s_t')$. We could also say that it reduces the difference between $V(s_t)$ and $\gamma V(s_t')+r_t$ because the reward is a constant bias whose gradient is zero. The reward does not appear in $(\gamma \phi_t' - \phi_t)$ because it is already taken care of in the averaged TD update, which will be driven to zero as the update proceeds. The effect is that we use the averaged TD update (estimated independently) {\em projected} on the current feature vector for the parameter update.  

The second and third lines give a different interpretation of the algorithm. The algorithm replaces the TD error in the standard TD with an average TD error, {\em similarity weighted}. In particular, instead of using the current TD error, $\delta_t$, calculated on the latest transition, to trigger learning, as in the standard TD(0), it uses an average of the TD errors that are computed on independent samples, weighted by the similarity of the sampled historical feature vectors to the current feature vector, for learning. Thus our algorithm takes an approach that comes with an improved estimation for the error signal to prevent the divergence of TD(0) for off-policy learning, for which using the latest TD error is problematic.

Notably, this interpretation gives a connection to \citet{baird1995residual}'s Residual Gradient (RG) algorithm. If we replace the weighted averaged TD error in the  third line with the latest TD error, it becomes exactly RG. RG is guaranteed to converge, however, not to the TD solution, e.g., see \citep{tdc}. TD(0) uses the latest TD error in another way, however, it suffers from divergence for off-policy learning. This update is guaranteed to converge to the TD solution under general and the same conditions as GTD. Furthermore, the convergence is orders faster than GTD, as we will show in Section \ref{sec:theory}.

The complexity of mini-batch Impression GTD is $O((m_1+m_2)d)$ per step, where $m_1$ and $m_2$ are the batch sizes. It is more complex than the Impression GTD in Section \ref{sec:imgtd} and GTD. However, it is still a linear complexity that is scalable to large problems.

An easier implementation for independence sampling is to have two buffers. Before the start of an episode, we can choose a random number that is either zero or one with equal probability. If it’s zero, then all the samples in this episode will be saved to the first buffer; otherwise, they will be saved to the second buffer. At sampling time, we just sample a batch from the first buffer and another batch from the second. In this way, we can also save extra memory for the episode index in each sample. Using the odd-even episode number for switching the buffers also works. This two-buffer implementation is shown in Algorithm \ref{alg:imGTD}. The similarity computation is also consumed so as to vectorize. 

\begin{algorithm}[t]
\caption{Impression GTD for off-policy learning, with independence sampling.}\label{alg:imGTD}
\begin{algorithmic}
\Require $\gamma \in (0, 1)$, the discount factor; $\alpha>0$, the step-size; $\phi(\cdot): \SS \to \RR^d$, the features
\State $\theta \gets \theta_0$
\Comment{Initialize the parameter vector}
\State buffer $B_1 \gets []$
\State buffer $B_2 \gets []$
\Comment{Initialize the buffers}
\For{episode $e=0, 1, \ldots$}

\State Environment resets to an initial state, $s_0$, drawn i.i.d. from some distribution 

\State $s\gets s_0$

\For{time step $t=0, 1, \ldots$}

    \State Observe $\phi(s)=\phi$, and take an action according to the behavior policy $\pi_b$

    \State Observe the next feature vector $\phi(s')=\phi'$ and reward $r$
    
\If{$e$ is odd} 
\Comment{Append the data in the same episode to same buffer}
    \State $B_1.\mbox{append}((\phi, \phi', r))$
\Else
    \State $B_2.\mbox{append}((\phi, \phi', r))$
\EndIf

\If{len$(B_2)>M$}
\State  
Sample a batch of $m_2$ samples from $B_2$, $\{(\phi_j, \phi_j', r_j)\}$, and compute
\[
\bar{u} \gets \frac{1}{m_2} \sum_{j=1}^{m_2} \left(\gamma \phi_j'^\tr\theta +r_j - \phi_j^\tr\theta\right)\phi_j.
\]

\State Sample $\{(\phi_i, \phi_i', r_i), i=1, \ldots, m_1\}$ from $B_1$ 

\State Form a feature matrix $\Phi$ with $\Phi[i,:] = \phi_i^\tr$
\Comment{The $i$th row of the matrix is $\phi_i^\tr$}

\State Compute $\bar{\bold{\delta}}=\Phi \bar{u}$

\State Update the parameters by
\[
\theta \gets \theta - \alpha \frac{1}{m_1}\sum_{i=1}^{m_1}(\gamma \phi_i' - \phi_i) \bar{\bold{\delta}}(i).
\]

\EndIf

\State $s\gets s'$
\EndFor
\EndFor
\end{algorithmic}
\end{algorithm}

In terms of the similarity measure used by our algorithm, the most relevant work is a recent new loss,  called the ``K-loss'' function \citep*{lihong_kernel_sim}, 
defined by the product of the Bellman errors calculated on two i.i.d. transition samples, weighted by a kernel encoding of the similarity between the two samples. They are probably the first to find that considering the similarity interplay between i.i.d. transitions can circumvent the double-sampling problem for reinforcement learning. Using their method, we can actually derive our algorithm in a second way. In particular, the NEU objective is
\begin{align}
\bold{NEU} &= \norm{\EE\delta\phi}^2\label{eq:neu_L2}\\
&=\EE[\delta\phi]^\top \EE[\delta\phi] \nonumber\\ 
&= \EE[\delta_1\phi_1]^\top \EE[\delta_2\phi_2] \Longleftrightarrow \EE[\delta_1\phi_1^\top \delta_2\phi_2] \nonumber \\
&= \EE[\phi_1^\top \phi_2 \delta_1 \delta_2] \nonumber\\
&= \EE[\bold{sim}(s_1, s_2) \delta_1 \delta_2].\label{eq:neu_2deltas}
\end{align}
This is exactly the place where \citeauthor*{lihong_kernel_sim} and we converge to. The  double-sampling problem arises when one aims to optimize using the single, online sample (the first line), which has held back the off-policy learning field for decades. 
The third line means we are realizing the $\ell_2$ norm on two independent transitions instead of on a single transition. The arrow annotated equality is due to the independence sampling.\footnote{Strictly speaking, the K-loss was defined using the Bellman errors, e.g., equation (3) of \citep{lihong_kernel_sim} uses the Bellman operator. The proof of their Corollary 3.5 mentioned "TD error". However the proof was done for deterministic MDPs for which TD error is the same as Bellman error.  Writing the loss in terms of  the weighted independent TD errors is more direct, also easier to interpret (without a model) and it entails direct optimization for practitioners. This is also interesting because minimizing the usual, online TD error via gradient descent has pitfalls \citep{tdc}, in particular the way represented by residual gradient \citep{baird1995residual}.} 
Therefore besides minimizing NEU, another way of interpreting our method is that it is a SGD method for minimizing the {\em expected product of two i.i.d. TD errors}, weighted by the similarity between the two feature vectors where the TD errors happened.

\begin{proposition}
Using independence sampling, we sample two independent transitions, $(s_1, r_1, s_1')$ and $(s_2, r_2, s_2')$ from the two buffers that have an infinity length. 
Consider a generic loss $\bold{N}(\theta) = \EE[\bold{sim}(s_1, s_2) \delta_1 \delta_2]$, where $\bold{sim}(s_1, s_2)$ is some similarity measure. Define $C= \EE[\phi(s)\phi(s)^\tr]$, where the expectation is taken with respect to the behavior policy (i.e., the distribution of $s$). Assume $C$ is non-singular. If $\bold{sim}(s_1, s_2)= \phi(s_1)^\tr C^{-1} \phi(s_2)$, then we have, $\bold{N}(\theta) = \bold{MSPBE}(\theta)$.  
\end{proposition}
\begin{proof}
We have 
\begin{align*}
\bold{N}(\theta) &= \EE[\bold{sim}(s_1, s_2) \delta_1 \delta_2] \\
&= \EE[\phi_1^\tr C^{-1} \phi_2 \delta_1\delta_2]\\
&= \EE[\delta_1 \phi_1^\tr C^{-1} \delta_2\phi_2 
]\\
&= \EE[\delta_1 \phi_1^\tr]  C^{-1} \EE[\delta_2\phi_2 
]\\
&= \bold{MSPBE}(\theta), 
\end{align*}
where the last second line is because of independence sampling, and $C^{-1}$ is a constant. The last line is because the buffers are sufficiently long so that the empirical distribution is the true data distribution. 
\end{proof}
Thus this shows that NEU and MSPBE belong to the same family of objective functions that are only different in a similarity measure, under independence sampling. Note this observation actually holds for any S.P.D matrix $U$ besides $C$. In particular, the generic loss $E(\theta)$ discussed in Section \ref{tdc} is also a special case of $\bold{N}(\theta)$.  
While these observations are interesting, we focus on minimizing NEU in this paper. 

Our method of deriving the Impression GTD algorithm by decoupling the estimations of $A^\top$ and $A\theta+b$ in GTD also entails an empirical form of the NEU loss, given multiple samples:
\[
\widehat{\bold{NEU}}(\theta|B_1, B_2) =\sum_{s_1 \in B_1}\sum_{s_2 \in B_2}\bold{sim}(s_1, s_2) \delta_1(\theta) \delta_2(\theta).
\]
For episodic problems, $B_1$ and $B_2$ are from our two-buffer implementation, for which samples in $B_1$ are always independent from the samples in $B_2$. 

For infinite horizon problems, $B_1$ and $B_2$ can be collected such that samples in them have a sufficiently large time window. For example, every $10000$ steps, we switch the collection buffer. In the first $T_0$ time steps, all the samples are inserted into $B_1$ and for the next $T_0$ time steps, the samples go into $B_2$; etc. A large $T_0$ ensures that no samples, for which the similarities are computed, happened close to each other in time, thus controlling their dependence strength at sampling time.    

Writing the NEU loss in terms of multiple samples is more reminiscent of the general machine learning problem where one minimizes an empirical loss on data sets. This is especially interesting because it transforms off-policy learning, an important problem of reinforcement learning, into a supervised learning problem, except that the data still needs to be collected for which there is the issue of exploration, etc. Nonetheless, we think it is an important connection to establish between reinforcement learning and supervised learning. This view is also interesting because the $\bold{sim}$ is a matrix form now, which measures inter-similarity between independent samples across the two buffers.\footnote{Note that \citeauthor{lihong_kernel_sim} did not have this form of the loss. Instead, they estimated the loss and the gradient using V-statistics. It has a problem that is discussed later in this section.}
Suppose the buffers keep adding the data and never drop any sample. Taking the gradient descent for the empirical NEU gives
\begin{align*}
\theta_{t+1} &= \theta_t - \frac{\alpha}{2} \nabla \widehat{\bold{NEU}} \\
&= \theta_t - \frac{\alpha}{2} \sum_{s_1 \in B_1}\sum_{s_2 \in B_2}\bold{sim}(s_1, s_2) \nabla(\delta_1 \delta_2) \\
&= \theta_t - \frac{\alpha}{2} \sum_{s_1 \in B_1}\sum_{s_2 \in B_2}\phi_1^\tr\phi_2 (\delta_1\nabla\delta_2 + \nabla\delta_1 \delta_2 )\\
&= \theta_t - \frac{\alpha}{2} \sum_{s_1 \in B_1}\sum_{s_2 \in B_2}\phi_1^\tr\phi_2 \left[ (r_1+\gamma\phi_1'^\tr \theta_t - \phi_1^\tr \theta_t) \nabla\delta_2 + (r_2+\gamma\phi_2'^\tr \theta_t - \phi_2^\tr \theta_t) \nabla\delta_1 \right]
\end{align*}
The two terms in the bracket is similar. For example, the first one is  (dropping the subscript of $\theta$ for simplicity)
\begin{align*}
 \phi_1^\tr \phi_2(r_1+\gamma\phi_1'^\tr \theta - \phi_1'^\tr \theta) \nabla\delta_2
&=\phi_2^\tr \phi_1(r_1+\gamma\phi_1'^\tr \theta - \phi_1^\tr \theta) (\gamma \phi'_2 - \phi_2). 
\end{align*}
Suppose $T_1$ samples are stored in buffer $B_1$ and $T_2$ samples are in buffer $B_2$. We have 
\begin{align*}
&\sum_{s_1 \in B_1}\sum_{s_2 \in B_2} 
    \phi_2^\tr \phi_1(r_1+\gamma\phi_1'^\tr \theta - \phi_1^\tr \theta) (\gamma \phi'_2 - \phi_2) \\
    =& \sum_{t_2=1}^{T_2} \sum_{t_1=1}^{T_1}\phi_2^\tr\phi_1(r_1+\gamma\phi_1'^\tr \theta - \phi_1^\tr \theta)
     (\gamma \phi'_2 - \phi_2) \\
    =& \sum_{t_2=1}^{T_2} \phi_2^\tr(\tilde{A}_1 \theta+\tilde{b}_1)
  (\gamma \phi'_2 - \phi_2) \\
      =& \sum_{t_2=1}^{T_2}   (\gamma \phi'_2 - \phi_2)\phi_2^\tr(\tilde{A}_1 \theta+\tilde{b}_1)
 \\
      =& \tilde{A}_2^\tr (\tilde{A}_1 \theta+\tilde{b}_1),
\end{align*}
in which we define $\tilde{A}_1 = \sum_{t_1=1}^{T_1} \phi_1(\gamma \phi_1' - \phi_1)^\tr$, and $\tilde{A}_{2} = \sum_{t_2=1}^{T_2} \phi_2(\gamma \phi_2' - \phi_2)^\tr$. 
The normalized matrices, i.e., $\tilde{A}_1/T_1$ and $\tilde{A}_2/T_2$, are both consistent estimations of the matrix, $A=E[\phi(\gamma \phi' - \phi)^\tr]$. Note this algorithm can be implemented in a complexity that is linear in the number of samples ($n$), i.e., $O(d^2)$ per sample, where $d$ is the number of features, by forming the matrices explicitly.

Therefore, if we go for a direct approach of minimizing the empirical NEU, it ends up with a variant of the expected GTD \citep{ptd_yao}, 
\[
  \theta_{t+1}   = \theta_t -\frac{\alpha}{2} \left[\tilde{A}_1^\tr (\tilde{A}_2 \theta_t+\tilde{b}_2) + \tilde{A}_2^\tr (\tilde{A}_1 \theta_t+\tilde{b}_1) \right],
\]
which is $O(d^2)$ per step (the two matrices can be aggregated incrementally).
Though an interesting variant of the expected GTD, this algorithm is presented purely for the understanding of Impression GTD. Our convergence and convergence rate analysis apply to this variant in a straightforward way. 

 The Impression GTD applies the successful mini-batch training in deep learning to off-policy learning and reduces to a linear complexity in the number of features, without resorting to two-time scale stochastic approximation. This observation was also made by \citet{lihong_kernel_sim}. They noted that their loss function ``coincides'' with NEU in the linear case (see their Section 3.3). However, the reason was not well understood or explained. Hopefully it is clear that our derivation above showed that this is not an coincidence. In matrix notation, the minibatch Impression TD uses the batch samples to build two matrices, $\tilde{A}_1^\tr = \sum_{b=1}^m (\gamma\phi_1' - \phi_1)\phi_1^\tr $, from the batch samples in $B_1$, and  $\tilde{A}_2 = \sum_{b=1}^m \phi_2  (\gamma\phi_2' - \phi_2)^\tr$ (and $\tilde{b}_2 = \sum_{b=1}^m \phi_2 r_2$), from the minibatch samples of $B_2$. The terms were transformed equivalently using the $\bold{sim}$ measure such that these matrices do not form explicitly and thus avoid the $O(d^2)$ complexity, e.g., see 
equation \ref{eq:imgtd}.\footnote{We found it's interesting that the two implementations, one that forms the matrices explicitly, and the other that doesn't (instead using $\bold{sim}$), gives the flexibility of switching for the higher computation efficiency given different numbers of samples (e.g., the batch sizes). }

\citeauthor{lihong_kernel_sim} also had a batch version of their algorithm, e.g., see their equation 4 and Section B.1 therein. However, the implementation is not technically sound because the independence of samples would break.  
The nature is a bit tricky.\footnote{We also refer the readers to \citep{shangtong_imgtd} for more detailed discussions about this problem.} Random sampling from the buffer does not necessarily means the samples in the buffer are i.i.d. 
Let's say we have two samples, $(s_1, s_1', r_1)$ and $(s_2, s_2', r_2)$. They are sampled i.i.d. from the buffer. However, what if they occurred in the same episode when we inserted them? Let's say $(s_2, s_2', r_2)$ was inserted into the buffer right after $(s_1, s_1', r_1)$. That is, $s_2=s_1'$. The second sample is dependent on the first one. In general, as long as the two samples are from the same episode, the one that happens at a later time depends on the former one and they are not independent.
It may be easier to understand in the infinite horizon setting. Suppose the Markov chain is irreducible and aperiodic and there exists a unique stationary distribution under the behavior policy. Then the samples are only independent of each other if the empirical distribution of the states in the buffer gets sufficiently close to the stationary distribution. Before this happens, the samples in the buffer are all dependent on each other. This depends on how fast the chain is mixing. It can take a very long time to reach the stationary distribution for slowly mixing chains. After the stationary distribution is reached, the Markovian argument for the above two samples still holds. However, because the distribution of states becomes stationary, the samples from the chain exhibit independence: the distribution of a state is just a property on its own. 
Consider a simple example. Assume that $s_2$ can only be reached from $s_1$. Note, however, from $s_1$ one can reach other state(s) than $s_2$. 
Let $\mu$ be the empirical distribution of the states in the buffer (a single buffer that stores all the samples up to the current time step). 
Then we have 
\[
\mbox{Prob}(s_2|s_1) = \mu(s_1) \mbox{Prob}(s1\to s_2)
\]
As long as $\mu(s_1) \mbox{Prob}(s1\to s_2)\neq \mu(s_2)$, $\mbox{Prob}(s_2|s_1)$ is not equal to $\mbox{Prob}(s_2)$, and thus the dependence between the two states holds. Before the chain reaches the stationary distribution, $\pi_0$, we don't have the equality. Only after $\mu$ gets sufficiently close to $\pi_0$, we have $\mu(s_1)\mbox{Prob}(s_1\to s_2)\approx \mu(s_2)$, and the independence between the states starts to exhibit. 

For episodic problems, one can define a similar chain from the distribution of the initial states (where the episodes are started), the behavior policy and the transition dynamics of the MDPs. If a unique stationary distribution exists, similar argument holds for the episodic problems. Most reinforcement learning problems in practice are episodic. Luckily, our independence sampling ensures the samples are independent even when the underlying chain has not reached the stationary distribution yet. This is shown by Lemma \ref{lem:independence} in Section \ref{sec:theory}.

\section{Analysis}\label{sec:theory}
This section contains convergence rate analysis of Impression GTD 
 with constant step-sizes. The first result is an $O(1/t)$ rate. For the second result, we first give a new condition of smoothness, called $L$-$\lambda$ smoothness. Under this weaker smoothness condition, we establish a tighter convergence rate for SGD than Theorem 3.1 of \citet{gower2019sgd_general}. 
Then by showing that the NEU objective and the independence sampling satisfies $L$-$\lambda$ smoothness, we prove that Impression GTD converges at a linear rate.

Our algorithm analysis is conducted in a generic GTD algorithmic framework. The $O(1/t)$ rate and the linear rate are both applicable to Expected GTD, A$^\tr$TD, and R1-GTD.  

Both the $O(1/t)$ rate and the linear rate depend on the i.i.d. sampling ensured by our independence sampling method. Thus we first introduce a lemma for that. 

\begin{lem}[Independence Sampling]\label{lem:independence}
For episodic problems, our sampling method according to equations \ref{eq:buffer1} and \ref{eq:buffer2} ensures that the transition samples from the two mini-batches are independent: 
\[
Pr(i_{t_1}  =s_1 \cap j_{t_2} = s_2) =Pr(i_{t_1}  =s_1 ) Pr(j_{t_2} = s_2),
\]
where $i_{t_1}$ and $j_{t_2}$ are the time steps that we insert the two samples into buffer $B_1$ and buffer $B_2$, respectively. 
\end{lem}
\begin{proof}
Without loss of generality, let us consider the batch size equal to 1. 
Let $(s_1, r_1, s_1')$ and $(s_2, r_2, s_2')$ be two samples drawn at time step $t$ by the sampling method. Then it suffices to prove that $s_1$ and $s_2$ are independent. For notation convenience, let $i_{t_1}$ be $s_1$, and the state sequence up to $s_1$ is $\{i_0, i_1, \ldots, i_{t_1}\}$, in the episode where we put $s_1$ into the buffer. Similarly, $j_{t_2}$ is for aliasing $s_2$. 
We just need to prove that $Pr(i_{t_1}=s_1 \cap j_{t_2} = s_2) = Pr(i_{t_1}=s_1) Pr(j_{t_2} = s_2)$. To see this, we first have
\begin{align*}
&Pr(i_{t_1}  =s_1 \cap j_{t_2} = s_2) \\
=& \sum_{i_0, i_1, \ldots, i_{t_1}-1}  \quad \sum_{j_0, j_1, \ldots, j_{t_2}-1}
Pr(\underline{i_0, i_1, \ldots, i_{t_1}  =s_1 } \cap \underline{j_0, j_1, \ldots, j_{t_2} = s_2})  \\
=&\sum_{i_0, i_1, \ldots, i_{t_1}-1}  Pr(i_0, i_1, \ldots, i_{t_1}  =s_1 )  \sum_{j_0, j_1, \ldots, j_{t_2}-1} Pr( j_0, j_1, \ldots, j_{t_2} = s_2)  
\end{align*}
The first equality is according to the law of total probability, which sums over all possible trajectories that lead to these two observations.  
The second equality is because the two episodes are independent due to that $i_0$ and $j_0$ are i.i.d. samples (which is ensured by the environment). 

It suffices to focus on the first term in the second equality. The second term  can be calculated similarly. We have 
\begin{align*}
&\sum_{i_0, i_1, \ldots, i_{t_1}-1}  Pr(i_0, i_1, \ldots, i_{t_1}  =s_1 )  \\
=& \sum_{i_0, i_1, \ldots, i_{t_1}-1}  Pr(i_{t_1}  =s_1 ) Pr(i_0, \ldots, i_{t_1}-1| i_{t_1}  =s_1 )   \\
=&   Pr(i_{t_1}  =s_1 ) \sum_{i_0, i_1, \ldots, i_{t_1}-1} Pr(i_0, \ldots, i_{t_1}-1)    \\
=&   Pr(i_{t_1}  =s_1 ) \sum_{i_0, i_1, \ldots, i_{t_1}-1} Pr(i_{t_1}-1)Pr(i_0, \ldots, i_{t_1}-2| i_{t_1}-1)    \\
=&   Pr(i_{t_1}  =s_1 ) \sum_{i_0, i_1, \ldots, i_{t_1}-2}\sum_{i_{t_1}-1}  Pr(i_{t_1}-1)Pr(i_0, \ldots, i_{t_1}-2| i_{t_1}-1)   \\
=&   Pr(i_{t_1}  =s_1 ) \sum_{i_0, i_1, \ldots, i_{t_1}-2}Pr(i_0, \ldots, i_{t_1}-2)   \\
=&   Pr(i_{t_1}  =s_1 ) \sum_{i_0} d_0(i_0)    \\
=& Pr(i_{t_1}  =s_1 ).
\end{align*}
The first equality is according to the conditional probability formula. The next equality is because historical observations are independent of later ones.

The remaining of the derivation breaks down according to the conditional probability formula. The third equality applies this one step,   
Then the next equality splits the sum over $i_{t_1}-1$, and the law of total probability follows. 
We recursively apply to the beginning to get the last second equality. Note $d_0$ is the sampling distribution of the initial state, and $\sum_{i_0}d_0(i_0)=1$.

\end{proof}

We analyze the convergence rates of Impression GTD under constant step-sizes. 
Many SGD analysis is conducted in the setting of the finite-sum loss function, e.g., see \citep{gower2019sgd_general,sps_stepsize}. In that setting, the function is of the form, $f(x)=\frac{1}{n}\sum_{i=1}^nf_i(x)$. This setting covers important applications in machine learning, especially supervised learning problems, where there are $n$ training samples and each $f_i$ is the  loss on sample $i$. However, it does not cover the application we consider in this paper, because the NEU objective is not a finite-sum form in a straightforward sense. Towards this end, we consider the ``expected form'' of the loss. That is, the loss function can be sampled via simulation, in particular, 
\[
f(x)=\EE[f_t(x)],
\]
where $f_t$ is the loss on the sample drawn at simulation step $t$,  according to a distribution $\mathcal{D}$. This covers the finite-sum loss and it is general enough to cover our GTD setting. Let $x^*$ be the optimum and $f(x^*)=\min_x f(x)$.   

We assume the gradient of the loss can be queried for each stochastic sample. Thus equivalently, we can also say that our simulation process keeps drawing the stochastic gradient. In particular, we draw $f_t'(x)$ to get a random sample for the true gradient $f'(x)$. 
At drawing step $t$, denote the gradient sample as $g_t(x)=f_t'(x)$ for a given $x\in \RR^d$.
\begin{lem}\label{lem:ED_avg}
Let us draw a batch of $m$ i.i.d. samples according to $\mathcal{D}$. 
Let $\avgx$ be the average of the sampled gradients in this batch, i.e.,  
\[
\avgx=\frac{1}{m}\sum_{t=1}^m g_t(x).
\]
We have, for any distribution $\mathcal{D}$ that satisfies $\EED [g_t(x)]=f'(x)$, the following holds:
    \[
    \EED \norm{\avgx}^2 = \frac{1}{m}\EED\norm{g_t(x)}^2+ \left(1-\frac{1}{m}\right) \norm{f'(x)}^2.
    \]
\end{lem}
The proof is in Appendix \ref{appendix:ED_avg}. 

Instead of analyzing Impression GTD and each of the three GTD algorithms that are discussed in Section \ref{sec:background} individually, 
we use a generic algorithmic framework that enables us to study their convergence rates at one time. 
Define $\tilde{A}_m=\frac{1}{m}\sum_{i=1}^m \phi_i (\gamma \phi_{i+1}- \phi_i)^\tr$ as a normalized matrix from $m$ samples.
Consider this algorithm:
\begin{equation}\label{alg:generic_gtd}
\theta_{t+1} = \theta_t - \alpha_t\tilde{A}_{m_1}^\tr (\tilde{A}_{m_2} \theta_t + \tilde{b}_{m_2}). 
\end{equation}
The above the algorithm is for mathematical definition only.  
Note the matrix and the transpose may not be explicitly formed or computed using matrix-vector product for certain algorithms. 
We compute $\tilde{A}_{m_1}$ and $\tilde{A}_{m_2}$ for different algorithms as follows:
\begin{itemize}
\item Impression GTD. 
$\tilde{A}_{m_1}$ and $\tilde{A}_{m_2}$ are computed
from buffer $B_1$ and buffer $B_2$, respectively. 

\item Expected GTD. For the algorithm that is discussed in Section \ref{sec:background} (equation \ref{eq:expectedGTD}), a single matrix is built from all the samples in the two buffers. To fit into the independence sampling and the generic TD framework, we consider here the version described in Algorithm \ref{alg:imGTD}. Thus $m_1=|B_1|$ and $m_2=|B_2|$. For simplicity of argument and without loss of generality, we assume $|B_1|=|B_2|=t/2$.

\item A$^\tr$TD. $\tilde{A}_{m_1}$ is computed from both buffers and $\tilde{A}_{m_2}$ is the rank-1 matrix from the latest transition, $\phi_t(\gamma \phi_{t+1} - \phi_t)^\tr$. Thus  
$m_1=t$ and $m_2=1$ for A$^\tr$TD.  

\item \oneexptd. $\tilde{A}_{m_1}$ is the rank-1 matrix from the latest transition, and $\tilde{A}_{m_2}$ is computed from both buffers instead. Thus $m_1=1$ and $m_2=t$ for \oneexptd.

\end{itemize}

We first introduce an assumption that is fairly general in the analysis of TD methods, e.g., see \citep*{tsi_td,bertsekas1996neuro,gtd,zhang2020gradientdice,zhang2020provably}.

\begin{assumption}\label{assumption:phir}
The feature functions in $\phi(\cdot): \SS\to \RR^d$, are independent. All the feature vectors and rewards are finite. 
\end{assumption}

We show that all the four discussed GTD algorithms are faster than GTD2 and TDC, even though the latter two were developed to improve the convergence rate of the GTD algorithm. Note that the above four GTD algorithms all have the same O.D.E. as the GTD algorithm. Thus the convergence is accelerated by them not by improving the conditioning of the problem (like GTD2 and TDC do). Instead, improvement is achieved by a single-time scale formulation of minimizing NEU. 

First consider this term, $\delta_i \phi_i^\tr\phi_j \delta_j$, which is a sample of NEU using two independent sample transitions from $\phi_i$ and $\phi_j$. Its gradient is $(\gamma \phi_{i+1}-\phi_i) \phi_i^\tr \phi_j \delta_j + (\gamma \phi_{j+1}-\phi_j) \phi_j^\tr \phi_i \delta_i $. For simplicity, we only consider the first term in the following analysis. The second term is symmetric and our proof can be extended to including it in a straightforward way. Define $f'_{i,j}=(\gamma \phi_{i+1}-\phi_i) \phi_i^\tr \phi_j \delta_j$. One can show that $\EE[f'_{i,j}]=\EE\nabla[\delta_i \phi_i^\tr\phi_j \delta_j]$. 
We have
\begin{align*}
\norm{f'_{i,j}(x) - f'_{i,j}(y)} &= \norm{(\gamma \phi_{i+1}-\phi_i) \phi_i^\tr \phi_j (\gamma \phi_{j+1}-\phi_j)^\tr (x-y)}  \\
&\le \underbrace{\norm{(\gamma \phi_{i+1}-\phi_i) \phi_i^\tr \phi_j (\gamma \phi_{j+1}-\phi_j)^\tr}}_\text{$L_{i,j}$}\norm{x-y}   = L_{i,j} \norm{x-y}.
\end{align*}
That is, each $f_{i,j}(x)$ is $L_{i,j}$ smooth. Given that all the feature vectors are finite according to Assumption \ref{assumption:phir}, $L_{i,j}$ must be finite.  

We are now ready to give the $O(1/t)$ rate. 
\begin{thm}\label{thm:rates_1_over_t}
Let Assumption \ref{assumption:phir} hold. Also assume matrix $A$ is non-singular. 
Impression GTD, Expected GTD, A$^\tr$TD and \oneexptd converge at a rate of $O(1/t)$ with $\alpha \le \frac{2}{L_{\max}}$, 
where $L_{\max}=\max_{i,j}L_{i,j}$. 
In particular, 
\[
\min_{k=0, \ldots, t-1} 
 f(\theta_k) \le \max\left\{\frac{2}{t \alpha \left(2- \alpha L_{\max} \right)\sigma^2_{\min}(A)}f (\theta_{0}) - \frac{1}{m_1m_2\sigma^2_{\min}(A)} \sigma_v^2, \, 0\right\}.
\]
\end{thm}
\begin{proof}
Consider the generic GTD update in \ref{alg:generic_gtd}. 
Because each $f_{i,j}$ is $L_{i,j}$-smooth, we have
\begin{align*}
f_{i,j}(\theta_{t+1}) &\le f_{i,j}(\theta_t) + f_{i,j}'(\theta_t)^\tr (\theta_{t+1}-\theta_t) +\frac{L_{i,j}}{2} \norm{\theta_{t+1} - \theta_t}^2 \\
&= f_{i,j}(\theta_t) - \alpha_t f_{i,j}'(\theta_t)^\tr  \tilde{A}_{m_1}^\tr (\tilde{A}_{m_2} \theta_t + \tilde{b}_{m_2}) +\frac{\alpha_t^2 L_{i,j}}{2} \norm{\tilde{A}_{m_1}^\tr (\tilde{A}_{m_2} \theta_t + \tilde{b}_{m_2})}^2. 
\end{align*}
Summing above for all the samples $i$ in batch $b_1$ and all the samples $j$ in batch $b_2$ gives
\begin{align*}
\sum_{i,j} f_{i,j}(\theta_{t+1}) &\le \sum_{i,j}f_{i,j}(\theta_t)- \alpha_t \sum_{i,j}f_{i,j}'(\theta_t)^\tr  \tilde{A}_{m_1}^\tr (\tilde{A}_{m_2} \theta_t + \tilde{b}_{m_2}) \\
&\quad +\frac{\alpha_t^2 \sum_{i,j}L_{i,j}}{2} \norm{\tilde{A}_{m_1}^\tr (\tilde{A}_{m_2} \theta_t + \tilde{b}_{m_2})}^2.
\end{align*}
Note that 
\begin{align*}
\frac{1}{m_1m_2}\sum_{i,j}f_{i,j}'(\theta_t)
&=\frac{1}{m_1} \sum_{i=1}^{m_1} (\gamma \phi_{i+1}-\phi_i) \phi_i^\tr\frac{1}{m_2}\sum_{j=1}^{m_2}  \phi_j \delta_j\\
&= \tilde{A}_{m_1}^\tr (\tilde{A}_{m_2} \theta_t + \tilde{b}_{m_2}).
\end{align*}
Thus 
\[
\frac{1}{m_1m_2}\sum_{i,j} f_{i,j}(\theta_{t+1}) \le \frac{1}{m_1m_2} \sum_{i,j}f_{i,j}(\theta_t)- \alpha_t\left(1 -\frac{\alpha_t\sum_{i,j}L_{i,j}}{2m_1m_2}\right) \norm{\tilde{A}_{m_1}^\tr (\tilde{A}_{m_2} \theta_t + \tilde{b}_{m_2})}^2
\]
Let $L_{\max}=\max_{i,j}L_{i,j}$, then for $\alpha_t \le \frac{2}{L_{\max}}$, the averaged pair-wise loss across the two batches is guaranteed to reduce because the following also holds:
\begin{equation}\label{eq:fb1b2}
\frac{1}{m_1m_2}\sum_{i,j} f_{i,j}(\theta_{t+1}) \le \frac{1}{m_1m_2} \sum_{i,j}f_{i,j}(\theta_t)- \alpha_t\left(1 -\frac{\alpha_tL_{\max}}{2}\right) \norm{\tilde{A}_{m_1}^\tr (\tilde{A}_{m_2} \theta_t + \tilde{b}_{m_2})}^2.
\end{equation}
For notation convenience, let $\bar{f}_{b_1, b_2}(x)= \frac{1}{m_1m_2}\sum_{i\in b_1,j\in b_2} f_{i,j}(x)$. We have, for a constant step-size $\alpha \le \frac{2}{L_{\max}}$,
\begin{align}
\bar{f}_{b_1,b_2}(\theta_t) &\le  \bar{f}_{b_1,b_2}(\theta_{t-1}) - \alpha\left(1 -\frac{\alpha L_{\max}}{2}\right) \norm{\tilde{A}_{m_1}^\tr (\tilde{A}_{m_2} \theta_{t-1} + \tilde{b}_{m_2})}^2\label{eq:fbatch}\\
&= \bar{f}_{b_1,b_2} (\theta_{0}) - \alpha\left(1 -\frac{\alpha L_{\max}}{2}\right)  \sum_{k=0}^{t-1}\norm{\tilde{A}_{m_1}^\tr (\tilde{A}_{m_2} \theta_{k} + \tilde{b}_{m_2})}^2 \nonumber\\
&\le \bar{f}_{b_1,b_2} (\theta_{0}) - \alpha\left(1 -\frac{\alpha L_{\max}}{2}\right)  t \min_{k=0, \ldots, t-1} \norm{\tilde{A}_{m_1}^\tr (\tilde{A}_{m_2} \theta_{k} + \tilde{b}_{m_2})}^2. \nonumber
\end{align}
Thus 
\begin{align*}
\min_{k=0, \ldots, t-1} \norm{\tilde{A}_{m_1}^\tr (\tilde{A}_{m_2} \theta_{k} + \tilde{b}_{m_2})}^2 &\le \frac{2}{t \alpha \left(2- \alpha L_{\max} \right)}\left(\bar{f}_{b_1,b_2} (\theta_{0}) - \bar{f}_{b_1,b_2} (\theta_{t})\right)\\
&\le \frac{2}{t \alpha \left(2- \alpha L_{\max} \right)}\left(\bar{f}_{b_1,b_2} (\theta_{0}) - \bar{f}_{b_1,b_2} (\theta^*)\right), 
\end{align*}
where the second line is because the averaged loss keeps decreasing according to equation \ref{eq:fbatch}, and furthermore, as $t$ goes to infinity, the loss is bounded below and thus $\theta^* = \lim_{t\to\infty}\theta_t$. 

This proves the $\ell_2$ norm of the update of the generic GTD converges at a rate of $O(1/t)$. The above steps can also be conducted after taking expectation of equation \ref{eq:fb1b2} with respect to the sampling. This gives 
\begin{align*}
\min_{k=0, \ldots, t-1} \EED \norm{\tilde{A}_{m_1}^\tr (\tilde{A}_{m_2} \theta_{k} + \tilde{b}_{m_2})}^2 &\le \frac{2}{t \alpha \left(2- \alpha L_{\max} \right)}f (\theta_{0}). 
\end{align*}
Note that in the context of generic GTD, the stochastic gradient is $f'_{i,j}(\theta)$, and the batch size is actually $m_1m_2$ because the generic GTD  essentially uses this number of pairs of the correlated TD errors from the two buffers (one has $m_1$ samples and the other $m_2$). Thus we can use Lemma \ref{lem:ED_avg} (which depends on the i.i.d. sampling that is ensured by Lemma \ref{lem:independence}) to get
\begin{align*}
\EED \norm{\tilde{A}_{m_1}^\tr (\tilde{A}_{m_2} \theta_{k} + \tilde{b}_{m_2})}^2 &=
\frac{1}{m_1m_2}\EED\norm{f'_{i,j}(x)}^2+ \left(1-\frac{1}{m_1m_2}\right) \norm{f'(x_t)}^2\\
&\ge \norm{f'(x_t)}^2 +\frac{1}{m_1m_2} \sigma_v^2 \\
&= \norm{A^\tr (A\theta_k+b)}^2+\frac{1}{m_1m_2} \sigma_v^2\\
&\ge \sigma_{\min}^2(A) \norm{A\theta_k+b}^2 +\frac{1}{m_1m_2} \sigma_v^2,
\end{align*}
where the first inequality is because of Lemma \ref{lem:g2andf2}.

Therefore, 
\begin{align*}
\min_{k=0, \ldots, t-1} 
 f(\theta_k) &= \min_{k=0, \ldots, t-1} 
 \norm{A\theta_k+b}^2 \\
 &\le \max\left\{\frac{2}{t \alpha \left(2- \alpha L_{\max} \right)\sigma^2_{\min}(A)}f (\theta_{0}) - \frac{1}{m_1m_2\sigma^2_{\min}(A)} \sigma_v^2, \, 0\right\}.
\end{align*}

\end{proof}
Theorem \ref{thm:rates_1_over_t} shows that out of the historical learning steps, we are guaranteed to find a moment with an $O(1/t)$ reduction of the initial loss.\footnote{It also shows that the minibatch update is helpful, because it enables more reduction than $O(1/t)$. For larger batch-sizes $m_1$ and $m_2$, this benefit grows smaller, indicating that smaller batch-sizes should converge faster. 
Although this is interesting, we found this is contradictory to our empirical results which we cannot explain why.} 

\citet{bo_gtd_finite} proved that certain variants of GTD and GTD2 converge at a rate of $O(t^{-1/4})$ with a high probability. The algorithm variants apply projections to the iterators to keep them bounded. The rate was proved by applying the rate analysis of the saddle point problem \citep{saddle_point_Nemirovski}. A key condition that guarantees this rate is the use of a {\em fixed} step-size, $1/\sqrt{t}$, by knowing the total number of iteration steps before hand. For example, if we want to learn $10000$ steps, at all the learning steps, the step-size is $0.01$. 

\citet*{dalal2018finite_twotimescale} established the convergence rates of a variant of GTD that projects the update back to a ball sparsely. With diminishing step-sizes $\alpha_t=t^{-(1-\tau)}$ and $\beta_t=t^{-2/3(1-\tau)}$, where $\kappa$ is some constant in $(0,1)$, they showed that the convergence rate is $O(t^{-1/3 + \kappa/3})$, which is roughly $O(t^{-1/3})$ at best. This is slightly faster than \citet{bo_gtd_finite}'s rate if $\kappa$ is small. One can understand that this is due to the use of a bigger step-size than the fixed $1/\sqrt{t}$ step-size.  
The rate also applies to GTD2 and TDC. This result was obtained by drawing inspiration from single-time-scale stochastic approximation \citep{borkar2008book}, in bounding the distance of the two-time-scale iterations to the trajectories that are generated by the O.D.E. 
An important condition for this distance to remain bounded is that the two step-sizes are scheduled to satisfy the two-time-scale condition.     

Later with the step-sizes $t^{-\alpha}$ and $t^{-\beta}$ respectively for the two iterators, which satisfy $0<\beta<\alpha<1$, they showed that the convergence rates are $O(t^{-\alpha/2})$ and $O(t^{-\beta/2})$ for the two iterators, and the bounds are tight \citep*{dalal2020tale_twotimescale}.\footnote{Note that in this paper, analysis of the GTD algorithm was applied with a projection operator very $2^i$ ($i=0, 1, \ldots$) steps to keep the update bounded. Our analysis does not use any projection. 
}  Given that GTD learns slower than GTD2 and TDC as found by empirical studies \citep{tdc,sutton2018reinforcement}, it is probably safe to say that GTD (without projection) converges no faster than this rate. 
In short, all the three GTD algorithms converge slower than $O(1/\sqrt{t})$. In fact, $O(1/\sqrt{t})$ is the theoretical rate limit of stochastic saddle-point problem \citep{bo_gtd_finite}. This means even if one uses advanced optimizers such as Stochastic Mirror-Prox \citep*{Mirror_prox}, GTD, GTD2 and TDC will not converge any faster than $O(1/\sqrt{t})$. 

In contrast, 
our Theorem \ref{thm:rates_1_over_t}  shows that the Impression GTD algorithm together with the three other GTD algorithms converge at least as fast as $O(1/t)$, much faster than GTD, GTD2 and TDC. \cite*{shangtong_imgtd} proved an $O(\xi(t)ln(t)/t)$ rate for their Direct GTD under the diminishing step-size that scales with $1/t$, where $\xi(t)$ is some slowly growing function such as $ln^2(t)$. This rate is almost $O(1/t)$, and it does not need to take the minimum over the historical learning steps.

We further show that the four algorithms actually converge in a linear rate to a biased solution. For that purpose, we first establish a convergence rate result for SGD, under the $L$-$\lambda$ smoothness. 

\begin{definition}[$L$-$\lambda$ smoothness] \label{def:L-lambda-smooth}  
If for all $x\in \RR^d$, function $f$ and $\mathcal{D}$ satisfy
\[
\EE_{\mathcal{D}}\norm{g_t(x)}^2 \le 2L (f(x) - f(x^*)) + \lambda \norm{x-x^*}^2 +  \sigma^2, 
\] 
we say that $f$ is $L$-$\lambda$ smooth smooth under distribution $\mathcal{D}$, or simply, $(f, \mathcal{D})\sim L$-$\lambda(  \sigma^2)$,
\end{definition}
This new definition of expected smoothness has a background in our Impression GTD setting. 
Note that in this definition, $\sigma^2$ can be any positive real number. \citeauthor{gower2019sgd_general} used $\sigma^2=\EED\norm{g_t(x^*)}^2$. 
We will show that in our analysis, $\sigma^2$ is some different number.  
The new definition adds a term of $\lambda\norm{x-x^*}$ to allow for convergence analysis of GTD algorithms. This term improves the expected smoothness to be more {\em noise tolerant}, and thus the induced smoothness is more general. 

\begin{lem}\label{lem:utrongly_norm_grad}
If $f$ is $\mu$-strongly quasi-convex, then we have for any $x\in \RR^d$, 
\[
\norm{f'(x)} \ge \mu \norm{x-x^*}.  
\]
\end{lem}
Appendix \ref{appendix:u_norm_grad} has the proof. 

\begin{lem} \label{lem:ES_qLsmooth}
    If $(f, \mathcal{D})$ $\sim$ $L$-$\lambda$ $( \sigma^2)$ for some $\sigma^2\ge 0$, then for any $x\in \RR^d$, we have 
    \[
   f(x) - f(x^*) \ge \frac{\norm{f'(x)}^2- \lambda\norm{x-x^*}^2 - (\sigma^2-\sigma_v^2)}{2L}, 
    \]
    where $\sigma_v^2 = \min_{x}\EED \norm{g_t(x)-f'(x)}^2$.
\end{lem}
The proof is in Appendix \ref{appendix:es_qls}.

The following theorem improves \citet*{gower2019sgd_general}'s Theorem 3.1 by removing the factor of two in the bias term because of the use of a refined definition of expected smoothness. The rate is also tightened for a faster rate with a $\mu^2$ term. Analysis on SGD usually drops $\EE \norm{\nabla f(x_t)}^2$ by relating it to $f(x_t)-f(x^*)$ first, and then drops $f(x_t)-f(x^*)$ due to $L$-smoothness and $f(x)\ge f(x^*)$, e.g., see \citep{gower2019sgd_general} and \citep*{sps_stepsize}.
This means their bounds on the convergence rate can be further tightened.  
Our proof keeps $f(x_t)-f(x^*)$, relates it to $\EE \norm{\nabla f(x_t)}^2$, and bounds the latter.  
This can be done by noting that $f(x) - f(x^*)$ can be lower bounded by the norm of the gradient together with the perturbation and the constant. 
By using Lemma \ref{lem:utrongly_norm_grad} for the strongly quasi-convexity of $f$, we have further
$\EE \norm{\nabla f(x_t) }^2\ge \mu^2 \EE \norm{x_t-x^*}^2$. 
\begin{thm}\label{thm:sgd_rate}
Assume $(f, \mathcal{D})$ $\sim$ $L\mbox{-}\lambda(\sigma^2)$ and $f$ is $\mu$-strongly quasi-convex. For SGD with batch update:
\[
x_{t+1} = x_t - \alpha_t\avg,
\]
we have 
\begin{align*}
\EE \norm{x_{t+1}-x^*}^2  &\le   \left[1-\left(\mu -\frac{\lambda}{L} \right) \alpha_t-\mu^2\alpha_t\left(\frac{1}{L}-\alpha_t  \right)\right] \EE\norm{\Delta_t}^2  + \frac{\alpha_t }{L}(\sigma^2-\sigma^2_v) + \frac{\alpha_t^2\sigma^2_v}{m},
\end{align*}
A linear convergence rate 
can be guaranteed for $\lambda\le L\mu$. 
Specifically, for a constant step-size $\alpha\le \frac{1}{L}$, we have 
\[
\EE \norm{x_{t}-x^*}^2 \le q^t\EE \norm{x_0-x^*}^2 +  \alpha\frac{   m(\sigma^2-\sigma^2_v) + L\alpha\sigma^2_v }{L m\left[\left(\mu -\frac{\lambda}{L} \right)+\mu^2\left(\frac{1}{L}-\alpha\right)\right]},
\]
where 
\[
q= 1-\left(\mu -\frac{\lambda}{L} \right) \alpha-\mu^2\alpha\left(\frac{1}{L}-\alpha  \right).
\]
\end{thm}
\begin{proof}
Let $\Delta_t = x_t - x^*$.
We have $\Delta_{t+1}= \Delta_t -  \alpha_t \avg$. 
Taking the squared $\ell_2$ norm and the conditional expectation gives 
\begin{align*}
\EED  \norm{\Delta_{t+1}}^2  & = \EED(\Delta_t -  \alpha_t \avg)^\top (\Delta_t -  \alpha_t \avg\nonumber \\
&= \norm{\Delta_t}^2  - 2 \alpha_t \EED\left[ \avg^\top \Delta_t|x_t\right] + \alpha_t^2 \EED \norm{\avg}^2\nonumber\\
&= \norm{\Delta_t}^2  - 2 \alpha_t \EED\left[ \avg\right]^\top \EED\left[\Delta_t\right] + \alpha_t^2 \EED \norm{\avg}^2\nonumber\\
&=
\norm{\Delta_t}^2  - 2 \alpha_t  \nabla f(x_t)^\top \Delta_t + \alpha_t^2 \EED \norm{\avg}^2
\end{align*}
where the third line is because $\Delta_t$ is independent of $\avg$ given $x_t$. The last line is due to the expected form of $f$, which gives $\EE[ g_t(x)]=\nabla f(x)$ for any $x$. 

Taking expectation over $x_t$ gives
\begin{align*}
\EE \norm{\Delta_{t+1}}^2 
&=\EE\norm{\Delta_t}^2  - 2 \alpha_t  \EE\left[\nabla f(x_t)^\top \Delta_t\right] + \alpha_t^2\EE \EED \norm{\avg}^2\\
&\le 
\EE \norm{\Delta_t}^2  - 2 \alpha_t  \EE\left(f(x_t)-f(x^*)+\frac{\mu}{2}\norm{\Delta_t}^2\right)  + \alpha_t^2 \EE\EED \norm{\avg}^2
\end{align*}
where the inequality is by the $\mu$-strongly quasi-convexity of $f$.

We have
\begin{align*}
\EE \norm{\Delta_{t+1}}^2 &\le (1-\mu\alpha_t) \EE \norm{\Delta_t}^2  - 2 \alpha_t   \EE\left(f(x_t)-f(x^*)\right)  + \alpha_t^2 \EE\EED \norm{\avg}^2\\
&= (1-\mu\alpha_t) \EE \norm{\Delta_t}^2  - 2\alpha_t  \EE\left(f(x_t)-f(x^*)\right) \\
& \quad + \alpha_t^2 \left( \frac{1}{m}\EE\EED\norm{g_t(x)}^2+ \left(1-\frac{1}{m}\right) \EE\norm{f'(x_t)}^2 \right)\\
&\le (1-\mu\alpha_t) \EE \norm{\Delta_t}^2  - 2\alpha_t\EE  \left(f(x_t)-f(x^*)\right) \\
& \quad + \alpha_t^2 \left( \frac{1}{m}\left( 2L\EE(f(x_t)-f(x^*)) + \EE \lambda\norm{\Delta_t}^2 + \sigma^2 \right)+ \left(1-\frac{1}{m}\right) 
\EE\norm{f'(x_t)}^2 \right)\\
&= \left(1-\mu\alpha_t +\frac{\lambda \alpha_t^2}{m} \right) \EE \norm{\Delta_t}^2 - 2\alpha_t\left(1 - \frac{\alpha_tL}{m} \right)\EE(f(x_t)-f(x^*)) \\
& \quad + \alpha_t^2\left(1-\frac{1}{m}\right) \EE\norm{f'(x_t)}^2 + \frac{\alpha_t^2\sigma^2}{m}
\end{align*}
in which line 2 is by Lemma \ref{lem:ED_avg}, and line 3 is according to $L$-$\lambda$ smoothness.
The above holds for any $\alpha_t$. 
Then with $\alpha_t\le \frac{m}{L}$,
\begin{align*}
\EE \norm{\Delta_{t+1}}^2&\le \left(1-\mu\alpha_t+ \frac{\lambda \alpha_t^2}{m}  \right) \EE \norm{\Delta_t}^2 +
\alpha_t^2\left(1-\frac{1}{m}\right) \EE\norm{f'(x_t)}^2 + \frac{\alpha_t^2\sigma^2}{m} \\
& \quad  - 2\alpha_t\left(1 - \frac{\alpha_tL}{m} \right)\frac{1}{2L}\left( \EE \norm{f'(x_t)}^2 -\EE\lambda \norm{\Delta_t}^2 -(\sigma^2 - \sigma^2_v) \right)\\
&= \left(1-\mu\alpha_t+ \frac{\lambda \alpha_t}{L}\right) \EE \norm{\Delta_t}^2  - \alpha_t\left( \frac{1}{L} -\alpha_t\right)   \EE \norm{f'(x_t)}^2 +   \frac{\alpha_t }{L}(\sigma^2-\sigma^2_v) + \frac{\alpha_t^2\sigma^2_v}{m},  
\end{align*}
where line 1 is by Lemma \ref{lem:ES_qLsmooth}.
Furthermore, if $\alpha_t\le \frac{1}{L}$, we can use Lemma \ref{lem:utrongly_norm_grad} to get 
\begin{align*}
\EE \norm{\Delta_{t+1}}^2&\le 
\left(1-\mu\alpha_t+ \frac{\lambda \alpha_t}{L}\right) \EE \norm{\Delta_t}^2  - \alpha_t\left( \frac{1}{L} -\alpha_t\right)   \mu^2 \EE\norm{\Delta_t}^2 +   \frac{\alpha_t }{L}(\sigma^2-\sigma^2_v) + \frac{\alpha_t^2\sigma^2_v}{m} \\
&= 
\left(1-\left(\mu -\frac{\lambda}{L} \right)\alpha_t  - \mu^2 \alpha_t\left( \frac{1}{L} -\alpha_t\right)  \right) \EE \norm{\Delta_t}^2    +   \frac{\alpha_t }{L}(\sigma^2-\sigma^2_v) + \frac{\alpha_t^2\sigma^2_v}{m},
\end{align*}

In the constant step-size case, with the choice of $\alpha\le \frac{1}{L}$, a linear rate is guaranteed because
\begin{align*}
0\le q&\stackrel{def}{=}1-\left(\mu -\frac{\lambda}{L} \right)\alpha  - \mu^2 \alpha\left( \frac{1}{L} -\alpha\right) \le 1- \mu^2 \alpha\left( \frac{1}{L} -\alpha\right)\le 1,
\end{align*}
due to that $\lambda\le L\mu$.

We run the recursion repeatedly until the beginning and get
\begin{align*}
\EE\norm{\Delta_{t}}^2& \le q^t \EE\norm{\Delta_0}^2 +  \left( \frac{\alpha }{L}(\sigma^2-\sigma^2_v) + \frac{\alpha^2\sigma^2_v}{m} \right)  \sum_{s=0}^\infty q^s \\
&= q^t \EE\norm{\Delta_0}^2 + \alpha\frac{   m(\sigma^2-\sigma^2_v) + L\alpha\sigma^2_v }{L m\left[\mu - \frac{\lambda}{L}+\mu^2\left(\frac{1}{L}-\alpha\right)\right]}.
\end{align*} 
\end{proof}

This theorem extends Theorem 3.1 of \citep{gower2019sgd_general} in three ways. First, the SGD rate is established under $L$-$\lambda$ smoothness, which is weaker than expected smoothness. Second, the linear rate is tightened with a $\mu^2$ term even for $\lambda=0$. Third, the bias term is more refined in the numerator too, with the difference between $\sigma^2$ and the minimum variance.

In the extreme case of $\lambda=L\mu$, although it guarantees a linear rate, for problems where $\mu$ is small, the rate can be very slow.
The factor $\lambda$ can be understood as the amount of perturbation to the expected smoothness condition. In particular, the more perturbation, the slower the rate. If the perturbation reduces, e.g.,  
if $\lambda \le (1-\rho) L\mu$ where $\rho \in [0, 1]$, then a much faster rate can be achieved:
\[
q\le 1- \rho\mu\alpha -\mu^2 \alpha\left( \frac{1}{L} -\alpha\right). 
\]
Luckily, as we will show later, in our Impression GTD setting, $\lambda$ can be reduced by increasing the batch sizes.  





Now we are ready to prove the linear rate result of Impression GTD. This is achieved by applying the SGD rate in Theorem \ref{thm:sgd_rate}.
First we introduce a lemma to show that in the Impression GTD problem, the loss function (NEU) and the independence sampling is $L$-$\lambda$ smooth, which is required by Theorem \ref{thm:sgd_rate}. 

\begin{lem}\label{thm_item:Llambdasmoooth}
Let $\mu=\sigma^2_{\min}(A)$, i.e., the minimum singular values of $A$. Assume $\mu>0$. Let Assumption \ref{assumption:phir} hold. 
Let $\Sigma_A$ be the matrix of the standard deviations of the rank-1 sample matrix $\phi(\gamma \phi'-\phi)^\tr$. That is, $\Sigma_A(i,j)= \sqrt{Var(\phi(i)(\gamma \phi(j)-\phi(j)))}$. Let $\Sigma_b$ be the vector of the standard deviations of $\phi r$, i.e., $\Sigma_b(i)=\sqrt{Var(\phi(i)r)}$.\footnote{These are all properties of the considered MDP, feature functions, the behavior policy and the target policy.} 

Define the following constants due to NEU and the independence sampling, respectively:
\[
L_1 =4\left(\frac{\norm{\Sigma_A}^2}{m_1} +  \norm{A}^2\right), \quad \sigma^2 = 16\left(\frac{\norm{\Sigma_A}^2}{m_1} +  \norm{A}^2\right) \left(\frac{\norm{\Sigma_A}^2}{m_2} \norm{\theta^*}^2 + \frac{\norm{\Sigma_b}^2}{m_2}\right).
\]
and
\[
L_2 = \frac{\norm{\Sigma_A}^2}{m_2}, \quad \lambda = \frac{2\norm{\Sigma_A}^4}{m_1m_2}; 
\]
The NEU objective function and the independence sampling satisfy the $L$-$\lambda(\sigma^2)$ smoothness with $L=L_1+L_2$, $\lambda=\lambda$, and $\sigma^2=\sigma^2$.  
\end{lem}

\begin{proof}

First we have $x^\tr H x = x^\tr H^\tr x$ holds even for a non-symmetric matrix $H$. This is because 
\begin{align*}
x^\tr H x  = \sum_{i}\sum_j H_{i,j}x_ix_j= \sum_{j}\sum_i H_{i,j}x_ix_j= \sum_{i}\sum_j H_{j,i}x_ix_j=x^\tr H^\tr x,
\end{align*}
where equality 2 is by switching the order of the two sums, and equality 3 is by swapping $i$ and $j$. Thus $\norm{Hx} = \norm{H^\tr x}$ holds for any real matrix $H$ and real vector $x$. This equality is crucial in the analysis below. 

We have
\begin{align*}
&\EED\norm{\tilde{A}_{m_1}^\tr (\tilde{A}_{m_2} \theta_t + \tilde{b}_{m_2})}^2 \\
&= \EED\norm{(\tilde{A}_{m_1}-A+A) ^\tr \left((\tilde{A}_{m_2}-A) \theta_t + {A} \theta_t + b + (\tilde{b}_{m_2}-b)\right)}^2\\
&= \EED\norm{(\Delta^A_{m_1} + A)^\tr \left(\Delta^A_{m_2}(\theta_t-\theta^*) + \Delta^A_{m_2}\theta^* + (A\theta_t +b) + \Delta^b_{m_2} \right)}^2\\
&\le 2\EED\norm{(\Delta^A_{m_1} + A)^\tr \Delta^A_{m_2}(\theta_t-\theta^*)}^2 + 2\EED \norm{(\Delta^A_{m_1} + A)^\tr}^2 \norm{ \Delta^A_{m_2}\theta^*  +  (A\theta_t +b) + \Delta^b_{m_2} )}^2
\end{align*}
where we define $\Delta^A_{m}= \tilde{A}_m -A$, and $\Delta^b_{m}= \tilde{b}_m -b$.
Let's first examine the second term in the above equation: 
\begin{align*}
&\EED\norm{(\Delta^A_{m_1} + A)}^2 \norm{ \Delta^A_{m_2}\theta^*  +  (A\theta_t +b) + \Delta^b_{m_2} )}^2\\
&=\EED\norm{(\Delta^A_{m_1} + A)}^2 \EED \norm{ \Delta^A_{m_2}\theta^*  +  (A\theta_t +b) + \Delta^b_{m_2} )}^2\\
&\le 2
\EED \norm{\Delta^A_{m_1} + A}^2 \left(\EED \norm{ \Delta^A_{m_2}\theta^* + \Delta^b_{m_2}}^2   + \norm{A\theta_t +b}^2 \right)\\
&\le 8 \left(\frac{\norm{\Sigma_A}^2}{m_1} +  \norm{A}^2\right) \left(2\EED\norm{\Delta^A_{m_2}}^2 \norm{\theta^*}^2 + 2\EED \norm{\Delta^b_{m_2}}^2     + f(\theta_t) \right)\\
&\le \underbrace{16\left(\frac{\norm{\Sigma_A}^2}{m_1} +  \norm{A}^2\right) \left(\frac{\norm{\Sigma_A}^2}{m_2} \norm{\theta^*}^2 + \frac{\norm{\Sigma_b}^2}{m_2}\right)}_\text{$\sigma^2$}     + \underbrace{8 \left(\frac{\norm{\Sigma_A}^2}{m_1} +  \norm{A}^2\right)}_\text{2$L_1$}\left(f(\theta_t) -f(\theta^*) \right)\\
& = \sigma^2 + 2L_1 \left(f(\theta_t)-f(\theta^*) \right), 
\end{align*}
where the equality is due to the independence sampling. The first inequality uses Jensen's inequality. The second inequality uses Jensen's inequality, $Var(\frac{1}{m_1}\sum_{i=1}^m X_i) = \frac{1}{m_1}Var(X)$, where $\{X_i\}$ are i.i.d. samples of the random variable $X$; and the triangle inequality. The third inequality uses the above variance relationship again. 

The first term is 
\begin{align*}
&\quad\EED\norm{(\Delta^A_{m_1} + A)^\tr \Delta^A_{m_2}(\theta_t-\theta^*)}^2\\
&\le 2\EED\norm{{\Delta^A_{m_1}}^\tr \Delta^A_{m_2}(\theta_t-\theta^*)}^2 + 2\EED \norm{A^\tr \Delta^A_{m_2}(\theta_t-\theta^*)}^2\\
&\le 2\EED\norm{{\Delta^A_{m_1}}}^2 \norm{ \Delta^A_{m_2}}^2 \norm{\theta_t-\theta^*}^2 + 2\EED \norm{ {\Delta^A_{m_2}}^\tr A(\theta_t-\theta^*)}^2\\
&\le 
2\EED\norm{{\Delta^A_{m_1}}}^2 \EED\norm{ \Delta^A_{m_2}}^2 \norm{\theta_t-\theta^*}^2 + 2\EED \norm{\Delta^A_{m_2}}^2\norm{A(\theta_t-\theta^*)}^2\\
&= 2\frac{\norm{\Sigma_A}^2}{m_1}\frac{\norm{\Sigma_A}^2}{m_2}\norm{\theta_t-\theta^*}^2  + 2\frac{\norm{\Sigma_A}^2}{m_2}\norm{A\theta_t+b}^2\\
&= \underbrace{2\frac{\norm{\Sigma_A}^4}{m_1m_2}}_\text{$\lambda$}\norm{\theta_t-\theta^*}^2  + 2\underbrace{\frac{\norm{\Sigma_A}^2}{m_2}}_\text{$L_2$}\left(f(\theta_t) -f(\theta^*)  \right)\\
&= \lambda \norm{\theta_t-\theta^*}^2  + 2L_2 \left(f(\theta_t) -f(\theta^*)\right). 
\end{align*}
The first inequality uses Jensen's inequality. The second inequality uses the triangle inequality, and $\norm{Hx}= \norm{H^\tr x}$ due to that $x^\tr H x = x^\tr H^\tr x$. The third inequality uses the independence sampling and the triangle inequality. The first equality uses the variance equality that was used in proving the second term, and $b= -A\theta^*$. 

Therefore, $\EED\norm{\tilde{A}_{m_1}^\tr (\tilde{A}_{m_2} \theta_t + \tilde{b}_{m_2})}^2 \le 2(L_1+L_2) \left(f(\theta_t) -f(\theta^*)\right) + \lambda \norm{\theta_t-\theta^*}^2  +\sigma^2$. This proves that the NEU objective function and the independence sampling satisfy the $L$-$\lambda(\sigma^2)$ smoothness with the specified constants. 

\end{proof}

The following theorem is shows that, for Impression GTD, the linear rate can be obtained by large batch sizes, and we show a sufficient choice is $m_1=m_2\ge \lceil \frac{1}{\sqrt{2\mu}}\frac{\norm{\Sigma_A}^2}{\norm{A}^2} \rceil$. 
For Expected GTD, linear rate can be achieved after a key metric, $\frac{\norm{\Sigma_A}^2}{t^2}$ gets small for Expected GTD. For A$^\tr$TD and \oneexptd, the rate becomes linear once $\frac{\norm{\Sigma_A}^2}{t}$ gets small. This can be understood as that after we have a big enough number of samples that is proportional to the variance of the problem (or simply put, our buffers are {\em representative} of the true data distribution in the sense of the variances), the algorithms converge fast. The results also show that Expected GTD is faster than \oneexptd and A$^\tr$TD, \oneexptd is faster than A$^\tr$TD. 

\begin{thm}[Conv. Rates of Impression GTD, Expected GTD, A$^\tr$TD, \oneexptd]\label{thm:rates_all}  

We have the following convergence rate results. 

\begin{enumerate}

\item \label{thm_item:imGTD}
Impression GTD (\ref{eq:imgtd}). 
With batch sizes $m_1=m_2 \ge \lceil \frac{1}{\sqrt{2\mu}}\frac{\norm{\Sigma_A}^2}{\norm{A}} \rceil =m$,\footnote{The $1/\norm{A}$ can be roughly interpretted as the condition number of NEU. Thus this shows that the batch sizes should increase with the condition number of NEU and the variances of the feature transitions. The constant $\frac{1}{\sqrt{\mu}}$ is a good sign because it is much smaller than $1/\mu$, if $\mu$ is very small. }
the algorithm converges linearly and
the rate is given by Theorem \ref{thm:sgd_rate} by using a step-size
\[
\alpha \le \frac{1}{5\frac{\norm{\Sigma_A}^2}{m} +  4\norm{A}^2}.
\]  

\item \label{thm_item:expGTD}
Expected GTD (\ref{eq:expectedGTD}). 
There exists $t_0$, such that $t>t_0$, we have
$\frac{2\norm{\Sigma_A}^2}{t}\le \epsilon$, and with $\alpha \le \frac{1}{4\norm{A}^2}$, 
\begin{align*}
\EE \norm{x_{t+1}-x^*}^2  &\le   q \EE\norm{x_{t}-x^*}^2  + \frac{\alpha }{4\norm{A}^2}(\sigma^2-\sigma^2_v) + \frac{4\alpha^2\sigma^2_v}{t^2},
\end{align*}
where 
\[
q= 1- \mu \alpha - \mu^2  \alpha\left(\frac{1}{4\norm{A}^2+5\epsilon} -\alpha \right)  + \frac{2\alpha}{4\norm{A}^2}\epsilon^2.   
\]

\item \label{thm_item:attd} 
A$^\tr$TD (\ref{eq:attd}). For $t>t_0$ such that
$\frac{\norm{\Sigma_A}^2}{t}\le \epsilon$, with $\alpha\le \frac{1}{\max\{ 4\norm{A}^2, \norm{\Sigma_A}^2\}}$, we have
\begin{align*}
\EE \norm{x_{t+1}-x^*}^2  &\le   q \EE\norm{x_{t}-x^*}^2  + \frac{\alpha }{\max\{ 4\norm{A}^2, \norm{\Sigma_A}^2\}}(\sigma^2-\sigma^2_v) + \frac{\alpha^2\sigma^2_v}{t},
\end{align*}
where 
\[
q=  1- \mu \alpha - \mu^2  \alpha\left(\frac{1}{4\norm{A}^2 + {\norm{\Sigma_A}^2} + 4\epsilon} -\alpha \right)  + \frac{\norm{\Sigma_A}^2}{\max\{ 4\norm{A}^2, \norm{\Sigma_A}^2\}}\alpha\epsilon.  
\]

\item \label{thm_item:r1etd} 
\oneexptd (\ref{eq:R1-GTD}).\footnote{It is straightforward to extend this result to the shrinked R1-GTD algorithm that is discussed in Section \ref{sec:minibatchPE}. The result remains the same by just replacing $t$ with $m_2$ and requiring $m_2$ to be sufficiently large. Similarly, this can be done for a shrinked version of A$^\tr$TD. }
After $t>t_0$ such that
$\frac{\norm{\Sigma_A}^2}{t}\le \epsilon$, the algorithm converges linearly with $\alpha \le \frac{1}{4\left({\norm{\Sigma_A}^2} +  \norm{A}^2\right)}$: 
\begin{align*}
\EE \norm{x_{t+1}-x^*}^2  &\le   q \EE\norm{x_{t}-x^*}^2  + \frac{\alpha }{4\left({\norm{\Sigma_A}^2} +  \norm{A}^2\right)}(\sigma^2-\sigma^2_v) + \frac{\alpha^2\sigma^2_v}{t},
\end{align*}
where 
\[
q= 1- \mu \alpha - \mu^2  \alpha\left(\frac{1}{4\left({\norm{\Sigma_A}^2} +  \norm{A}^2\right) + \epsilon} -\alpha \right)  + \frac{\norm{\Sigma_A}^2}{4\left({\norm{\Sigma_A}^2} +  \norm{A}^2\right)}\alpha\epsilon.   
\]

\end{enumerate}

\end{thm}
\begin{proof}
\ref{thm_item:Llambdasmoooth}.

\ref{thm_item:imGTD}. Impression GTD.  
Consider $m_1=m_2=m$.  
With 
\[
m \ge \left\lceil \frac{1}{\sqrt{2\mu}}\frac{\norm{\Sigma_A}^2}{\norm{A}} \right\rceil, 
\]
we have 
\[
\frac{4\norm{A}^2}{\norm{\Sigma_A}^2}m^2  + 5m -\frac{2\norm{\Sigma_A}^2}{\mu}> 0 
\]
This gives
\begin{align*}
 \frac{2\norm{\Sigma_A}^4}{\mu} &< 4\norm{A}^2m^2  + 5{\norm{\Sigma_A}^2}m \\
 &= 4\left(\norm{A}^2 + \frac{\norm{\Sigma_A}^2}{m} \right)m^2+\frac{\norm{\Sigma_A}^2}{m}m^2 \\
&=
{(L_1 + L_2)m^2}, 
\end{align*}
or equivalently, $\lambda < (L_1+L_2)\mu=L\mu$. Thus Theorem \ref{thm:sgd_rate} is applicable. The step-size condition can be derived by requiring that $\alpha \le \frac{1}{L}$. 

Next let's get the $\mu$ constant in the context of Impression GTD. 
Because $f'(\theta) = A^\tr (A\theta + b)$, 
we have  
\begin{align*}
 (x-y)^\tr (f'(x) -f'(y)) &= (x-y)^\tr  A^\tr A (x-y)\ge \sigma_{\min}^2(A)\norm{(x-y)}^2. 
\end{align*}
Thus $\mu = \sigma_{\min}^2(A)$. Thus we can apply Theorem \ref{thm:sgd_rate} and completes the proof for Impression GTD. 

\ref{thm_item:expGTD}. Expected GTD.  
For a sufficiently large $t>t_0$, we have
$\frac{2\norm{\Sigma_A}^2}{t}\le \epsilon$.
Note that 
\[
4\norm{A}^2\le L_1<L=L_1+L_2\le L_1 + \epsilon\le  4\norm{A}^2 + 5\epsilon,
\] 
which gives
\begin{align*}
\frac{\lambda}{L} &\le \frac{8\norm{\Sigma_A}^4}{4\norm{A}^2t^2} \le \frac{\epsilon^2}{2\norm{A}^2}; \quad -\frac{1}{L} \le - \frac{1}{4\norm{A}^2+5\epsilon}.
\end{align*}
According to Theorem \ref{thm:sgd_rate}, with $\alpha \le \frac{1}{4\norm{A}^2}$, 
\begin{align*}
\EE \norm{x_{t+1}-x^*}^2  &\le   q \EE\norm{\Delta_t}^2  + \frac{\alpha_t }{L}(\sigma^2-\sigma^2_v) + \frac{4\alpha_t^2\sigma^2_v}{t^2},
\end{align*}
where for the linear rate we have
\begin{align*}
q&=1-\left(\mu -\frac{\lambda}{L} \right)\alpha  - \mu^2 \alpha\left( \frac{1}{L} -\alpha\right) \\
&\le 1- \mu \alpha - \mu^2  \alpha\left(\frac{1}{4\norm{A}^2+5\epsilon} -\alpha \right)  + \frac{\alpha}{2\norm{A}^2}\epsilon^2.   
\end{align*}

\ref{thm_item:attd}. A$^\tr$TD. $m_1=t$ and $m_2=1$.
Note that $L_1$ still has a diminishing term but $L_2$ itself does not any more: 
\[
\quad L_1 = 4\left(\frac{\norm{\Sigma_A}^2}{t} +  \norm{A}^2\right), \quad 
L_2 = {\norm{\Sigma_A}^2}, \quad \lambda = \frac{2\norm{\Sigma_A}^4}{t},
\]
For $t>t_0$ such that
$\frac{\norm{\Sigma_A}^2}{t}\le \epsilon$, we have
\[
\max\{ 4\norm{A}^2, \norm{\Sigma_A}^2\}<L=L_1+L_2 \le  4\norm{A}^2 + {\norm{\Sigma_A}^2} + 4\epsilon.
\] 
Thus 
\begin{align*}
\frac{\lambda}{L} &< \frac{2\norm{\Sigma_A}^4}{\max\{ 4\norm{A}^2, \norm{\Sigma_A}^2\}t} \le \frac{2\norm{\Sigma_A}^2}{\max\{ 4\norm{A}^2, \norm{\Sigma_A}^2\}}\epsilon; \quad -\frac{1}{L} \le - \frac{1}{4\norm{A}^2 + {\norm{\Sigma_A}^2} + 4\epsilon}.
\end{align*}
Therefore, with $\alpha\le \frac{1}{\max\{ 4\norm{A}^2, \norm{\Sigma_A}^2\}}$,  the linear rate for A$^\tr$TD satisfies
\[
q= 1- \mu \alpha - \mu^2  \alpha\left(\frac{1}{4\norm{A}^2 + {\norm{\Sigma_A}^2} + 4\epsilon} -\alpha \right)  + \frac{2\norm{\Sigma_A}^2}{\max\{ 4\norm{A}^2, \norm{\Sigma_A}^2\}}\alpha\epsilon.   
\]
The bias in the rate can be bounded according to the lower bound of $L$. 

\ref{thm_item:r1etd}. \oneexptd. $m_1=1$ and $m_2=t$. The constants are now 
\[
L_1= 4\left({\norm{\Sigma_A}^2} +  \norm{A}^2\right), \quad L_2 =\frac{\norm{\Sigma_A}^2}{t}; \quad \lambda =  \frac{2\norm{\Sigma_A}^4}{t}.
\]
and the lower and upper bounds of $L$ are thus
\[
4\left({\norm{\Sigma_A}^2} +  \norm{A}^2\right)\le L_1<L=L_1+L_2\le L_1 + \epsilon\le  4\left({\norm{\Sigma_A}^2} +  \norm{A}^2\right) + \epsilon,
\] 
which gives
\begin{align*}
\frac{\lambda}{L} &< \frac{2\norm{\Sigma_A}^4}{4\left({\norm{\Sigma_A}^2} +  \norm{A}^2\right)t} \le \frac{\norm{\Sigma_A}^2}{2\left({\norm{\Sigma_A}^2} +  \norm{A}^2\right)}\epsilon; \quad -\frac{1}{L} \le - \frac{1}{4\left({\norm{\Sigma_A}^2} +  \norm{A}^2\right) + \epsilon}.
\end{align*}
With $\alpha \le \frac{1}{4\left({\norm{\Sigma_A}^2} +  \norm{A}^2\right)}$, the linear rate of \oneexptd  is thus
\[
q= 1- \mu \alpha - \mu^2  \alpha\left(\frac{1}{4\left({\norm{\Sigma_A}^2} +  \norm{A}^2\right) + \epsilon} -\alpha \right)  + \frac{\norm{\Sigma_A}^2}{2\left({\norm{\Sigma_A}^2} +  \norm{A}^2\right)}\alpha\epsilon.   
\]
\end{proof}
Comparing Theorem \ref{thm:rates_all} and Theorem \ref{thm:rates_1_over_t}, we can see that in Theorem \ref{thm:rates_1_over_t}, the step-size is much smaller, because in practice $L_{\max}$ can be very large. In that case, we are guaranteed to converge to the optimal solution, but with a slower rate. By using a much larger step-size in Theorem \ref{thm:rates_all}, we are only guaranteed to converge to a neighborhood of the optimal solution, however, with a much faster, linear rate.

Theorem \ref{thm:rates_all}.\ref{thm_item:imGTD} shows that the step-size of Impression GTD for the fastest convergence depends on three factors: the variance in the transition, the $\ell_2$ norm of $A$ and the batch size. In particular, the higher is the variance of the transition or the bigger is $\norm{A}$, the smaller the step-size we need to use for Impression GTD. A bigger batch size enables a larger step-size and faster convergence. The side effect of a larger step-size, though, is that the bias term in the convergence rate increases, which means the final convergence point may be located in a larger neighborhood of the optimal solution.    

\begin{table}
\centering
\begin{tabular}{l|l| l| l| l|l} 
 \hline
   & Batch size & $L_1-  4\norm{A}^2$ & $L_2$ & $\lambda$ & Bias\\
 \hline
 Im.GTD &$m^2$ & ${4\norm{\Sigma_A}^2}/{m} $ & ${\norm{\Sigma_A}^2}/{m}$ & ${2\norm{\Sigma_A}^4}/{m^2}$ & ${\alpha^2\sigma^2_v}/{m}$ \\ 
 Expected GTD & $t^2/4$ & ${8\norm{\Sigma_A}^2}/{t} $ & ${2\norm{\Sigma_A}^2}/{t}$ & ${8\norm{\Sigma_A}^4}/{t^2}$ & ${4\alpha^2\sigma^2_v}/{t^2}$\\
 A$^\tr$TD & $m_1=t, m_2=1$ & ${4\norm{\Sigma_A}^2}/{t}$ & $\norm{\Sigma_A}^2$ & ${2\norm{\Sigma_A}^4}/{t}$  & ${\alpha^2\sigma^2_v}/{t}$ \\
 R1-GTD& $m_1=1, m_2=t$ & $4{\norm{\Sigma_A}^2} $ & ${\norm{\Sigma_A}^2}/{t}$ & ${2\norm{\Sigma_A}^4}/{t}$& ${\alpha^2\sigma^2_v}/{t}$\\ 
 \hline
\end{tabular}
\caption{GTD Algorithm factors. For Impression GTD, we consider $m_1=m_2=m$. The first column is the effective batch size. }
\label{table:alg_L_mu_lambda}
\end{table}

The constants are summarized in Table \ref{table:alg_L_mu_lambda} for comparison.
Let's take a look at the smoothness constants of A$^\tr$TD and \oneexptd. 
For A$^\tr$TD, 
\[
\quad L_1 = 4\left(\frac{\norm{\Sigma_A}^2}{t} +  \norm{A}^2\right), \quad 
L_2 = {\norm{\Sigma_A}^2}, \quad \lambda = \frac{2\norm{\Sigma_A}^4}{t},
\]
For \oneexptd,
\[
L_1= 4\left({\norm{\Sigma_A}^2} +  \norm{A}^2\right), \quad L_2 =\frac{\norm{\Sigma_A}^2}{t}; \quad \lambda =  \frac{2\norm{\Sigma_A}^4}{t}.
\]
Clearly, the two algorithms have the same $\lambda$. In one extreme (A$^\tr$TD), $L_1$ is small but $L_2$ is large. In the other extreme (\oneexptd), $L_1$ is big but $L_2$ is small and in fact diminishing. Thus Impression GTD can be viewed as a balance between the two algorithms in $L_1$ and $L_2$. Its complexity is much lighter than the two algorithms, but it is still linear in the number of features. Though still higher than GTD, the order is the same, both in $O(d)$. The storage of Impression GTD is much higher due to the buffers. However, memory is not usually not a concern in modern computers, with a wide application in deep learning and deep reinforcement learning. 
After a sufficiently large number of learning steps, Impression GTD converges slower than A$^\tr$TD and \oneexptd, but the rate is still a linear rate, which is much faster than GTD, GTD2 and TDC. Our result also shows that A$^\tr$TD is slower than R1-GTD and with a larger bias term.     

Comparing Expected GTD, A$^\tr$TD and \oneexptd, we can see that there is wait time for the algorithms to be converge linearly. In particular, the wait time is proportional to $\lambda/L$, i.e., the ratio of the perturbation to the expected smoothness. For Expected GTD, this perturbation in the rate $q$ is $\frac{\norm{\Sigma_A}^4}{4\norm{A}^2t^2}$.\footnote{
The constant of the perturbation is also interesting. In particular, this ratio shows that the condition number of NEU and the variances in the feature transitions all contribute to the perturbation.} For the latter two algorithms, the perturbations are  
\[
\mbox{A$^\tr$TD: }
\frac{2\norm{\Sigma_A}^4}{\max\{ 4\norm{A}^2, \norm{\Sigma_A}^2\}t}; \quad \mbox{\oneexptd: } \frac{2\norm{\Sigma_A}^4}{4\left({\norm{\Sigma_A}^2} +  \norm{A}^2\right)t}. 
\]
We can see that the perturbation is diminishing in time. For the case of Expected GTD, the diminishing rate is very fast, which is $O(1/t^2)$. 
Thus the wait time for Expected GTD to converge linearly is much shorter than A$^\tr$TD and \oneexptd.\footnote{One can show that the wait time of Expected GTD is $1/\sqrt{\epsilon}$ for achieving a bias proportional to $\epsilon$. In the theorem, we let the algorithm wait the same amount of time as A$^\tr$TD and \oneexptd, for which case, the bias of Expected GTD is $O(\epsilon^2)$. The two presentation forms are equivalent.} A$^\tr$TD and \oneexptd have similar perturbation, both in the order of $O(1/t)$. At a constant scale, the perturbation in R1-GTD is smaller, and thus it waits shorter than A$^\tr$TD for the linear rate to arrive. 
The results also show that there is no guarantee that the three GTD algorithms (and Impression GTD as well) would converge fast before a sufficiently large number of samples in the buffers. This can be understood as that we need a sufficient amount of statistics built in our buffers and it takes time to grow it.  

Note that if we use mini-batch versions for GTD, GTD2 and TDC, their convergence rate may be expected to converge faster as well. Algorithm 1 by \citet{xu2021sample_twotimescale} shows how such update can be done for TDC. They showed that this mini-batch TDC also converges at a linear rate. The rate was established by requiring that the two step-sizes are smaller than some upper bound number. The number for $\alpha$ is fairly complex, containing quite a few terms from the minimum eigenvalues of $A^\tr C^{-1}A$ and $C$, the maximum importance sampling ratio, the ergodicity factor of the underlying Markov chain and $\beta$, the other step-size as well. On one hand, their result and ours show that mini-batch training is indeed a very useful tool for accelerating stochastic approximation, and effective for both single-time scale and two-time scale algorithms. However, on the other hand, regardless of that the mini-batch TDC algorithm also has two step-sizes, which is hard to use in practice just like TDC, the rate they proved is a fairly slow one even though it is linear. To be concrete, in their Theorem 1, let $\lambda_1 = \lambda_{\min} (A^\tr C^{-1}A)$ and $\lambda_2 = \lambda_{\min}(C)$. The $\lambda_1$ and $\lambda_2$ factors correspond to $\mu$, the strong convexity factor for solving the underlying O.D.Es of the main iterator and helper iterator, respectively. Also let $\rho_{\max} = \max_{s, a} \frac{\pi(a|s)}{\pi_b(a|s)}$, the maximum importance sampling ratio across all state-action pairs. The condition for the theorem requires that $\alpha$ should be at least as small as the minimum of $\frac{\lambda_1\lambda_2}{12}$ and $\frac{\lambda_1\lambda_2^2}{256 \rho_{\max}^2} $. Both numbers are extremely small because in practice the minimum eigenvalues are usually small. The factor $\rho_{\max}$ is very large in off-policy learning, and scales like $1000$  or even much higher aren't uncommon. 

In contrast, our result does not depend on $\rho_{\max}$ (at least not explicitly, it may still play a role in the conditioning number of $A$). In addition to the condition number of $A$, the (single) step-size in our result depends on the ratio between the variances in the feature transitions and the batch size(s), which means we can increase the step-size for larger batch sizes and also for problems in which feature transition variances are small.         

\footnote{This paragraph is due to a discussion with Csaba Szepesvari.} In literature, there is a result of $O(1/t)$ rate established for linear stochastic approximation algorithms, which also holds for a variant of GTD \citep*{csaba_lin_stochastic18}. The technique they used is iterate averaging \citep{polyak1992acceleration}, which iteratively averages the weight vector over all historical steps. Later, by adapting the step-size or using constant step-sizes that require prior knowledge of certain problem-dependent data structures, the same rate is also established for TDC with iterative averaging (over both the main and the helper iterators) \citep*{csaba_gtd_22}. This $O(1/t)$ is known to be information-theoretically near-optimal, e.g., see \citep{1_over_t_information_optiaml}. Thus it appears that our linear rate is contradictory to this well-known result. The catch is that our result has a bias because of the use of constant step-sizes. Although the results by \citeauthor{csaba_gtd_22} also contain the case of a special constant step-size, the averaging on the top of iterations provides a similar effect to the diminishing step-size, which enables their solution to converge to the true solution without a bias. To have a closer look of why our result is not contradictory, take the main result of \citeauthor{1_over_t_information_optiaml} (their Theorem 1) for example. The result states that, for any algorithm that comes up with a solution $\hat{x}$, there exists a data distribution (underlying the expectation operator in $f$) such that 
\begin{equation}\label{eq:worst_rate_1_over_t}
f(\hat{x}) - f(x^*) \ge c\, \min \left\{Y^2, \frac{B^2+dY^2}{t}, \frac{BY}{\sqrt{t}}
\right\}
\end{equation}
holds, 
where $B$ and $Y$ are some constants, $B\ge 2Y$, and $c$ is some positive constant. Now if $t$ is sufficiently large, the $O(1/t)$ term is the minimum of the three. Thus the result quantifies {\em the worst convergence rate to the optimal solution}. Precisely, the distance from any algorithmic solution to the optimal solution (in terms of the loss) cannot be anywhere closer than $O(1/t)$ for certain data distributions. For our result in Theorem \ref{thm:sgd_rate}, when $t$ is sufficiently large, for a constant step-size $\alpha\le \frac{1}{L}$, the first term becomes negligible, and we are left with the bias term, which is usually bigger than zero. Thus our theorem states that SGD converges to a {\em neighborhood} of the optimal solution $x^*$ {\em linearly fast}, but caution that it does not necessarily converge to $x^*$ linearly fast. The $O(1/t)$ rate to $x^*$ still applies to SGD and Impression GTD with diminishing step-size or iterate averaging. Note that the $O(1/t)$ information-theoretically near-optimal rate is the worst case, and it is realized on certain data distributions. For example, the proof of Theorem 1 in \citep{1_over_t_information_optiaml} is constructed by using an example in which the data distributions depend on the sample size. In practice, we are usually not that unlucky that our data distributions are screwed like such, and we may often get a faster rate than $O(1/t)$ when using SGD with mini-batch update. 

There is a special case that SGD will converge to the optimal solution $x^*$ linearly fast, no longer to just a neighbourhood of $x^*$. This corresponds to $B=Y=0$ in equation \ref{eq:worst_rate_1_over_t}. In this case, this bound is only an obvious fact instead of a rate. 
Our bound such as Theorem \ref{thm:sgd_rate} correctly covers a subclass of this case, with $L$-$\lambda$ smoothness for the loss and the sampling, the convergence of SGD is linear, with a zero bias. 
Interestingly, in Baird counterexample (\ref{exp:baird}), we actually see this linear rate to the optimal solution in experiments, because the bias term there is zero due to that $B=Y=0$. 

\section{Experiments}\label{sec:experiments}
This section contains empirical results of Impression GTD, for on-policy learning on Boyan chain, and off-policy learning on Random Walks (with tabular representation). Experiments on the inverted- and the dependent- representation for Random Walks are in Appendix \ref{exp:rwinv} and Appendix \ref{exp:rwdep}, respectively. Baird counterexample is in Appendix \ref{exp:baird}. All  the curves reported were averaged over 100 independent runs.

\subsection{Boyan Chain}
The problem is the same as \citep{lstd}. It has 13 states and the rewards are all -3.0 except that the transition from state 1 to state 0 incurs a reward of -2.0. The features are generated by a linear interpolation from four unit basis vectors at states $4i$, $i=0, 1, 2, 3$. Each episode starts from state $12$, and from state $i$ it goes to either $i+1$ or $i+1$, with an equal probability of 0.5. The features can represent the value function for this policy accurately. 

The compared algorithms include GTD, HTD, Vtrace, GTD2, TDC, TDRC, Impression GTD and mini-batch TD. 
At time step $t$, an algorithm gives $\theta_t$, and the metrics is computed by
\[
\mbox{RMSVE}(\theta_t) = \sqrt{\frac{1}{N}\sum_{s=1}^N (V^\pi(s) - \phi(s)^\tr \theta_t)^2}. 
\]

\begin{figure}[t]
\caption{Boyan chain: algorithm comparisons.}\label{fig:boyan_rmse}
\centering
\includegraphics[width=0.8\textwidth]{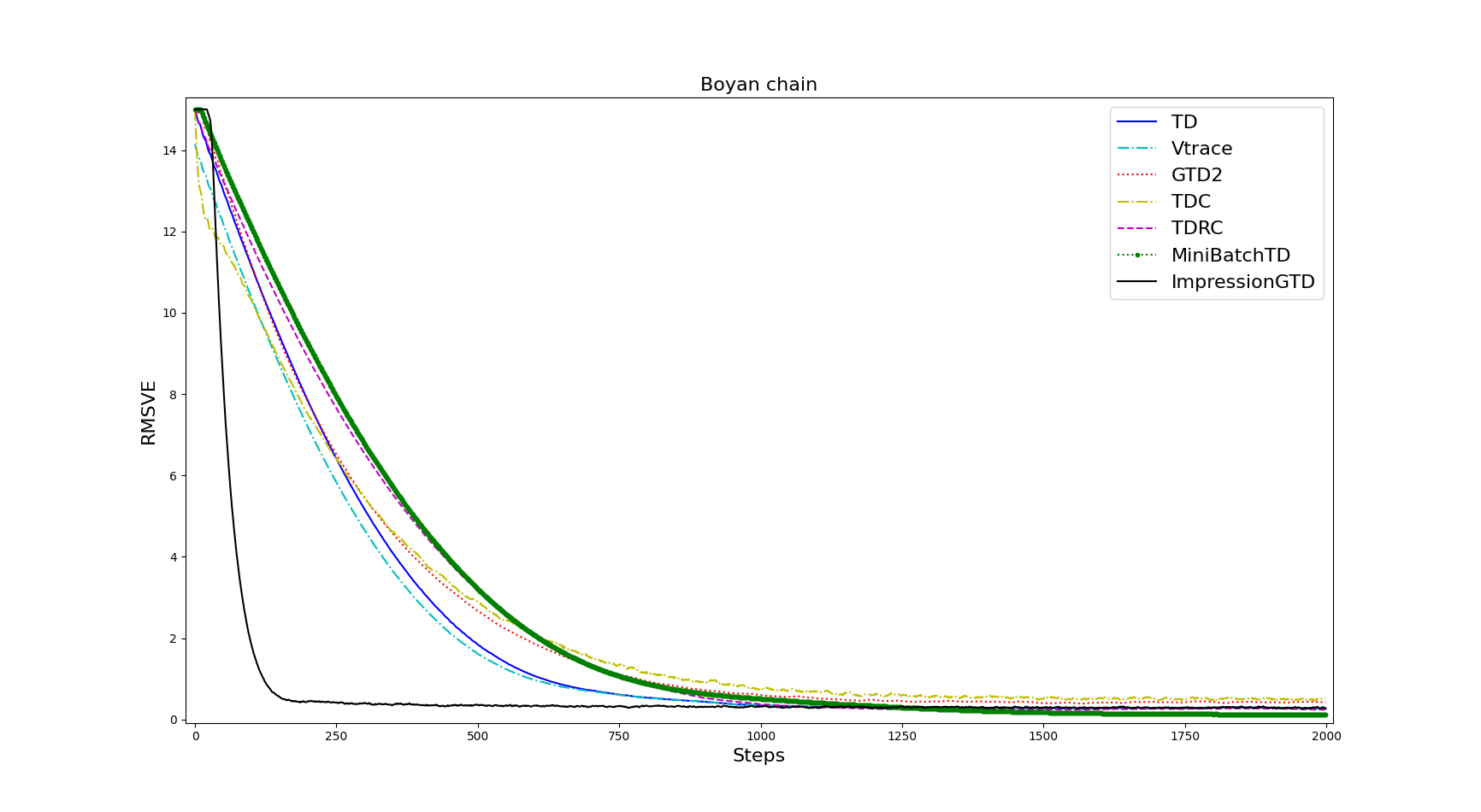}
\end{figure}

Figure \ref{fig:boyan_rmse} compares the RMSVE of the algorithms.
The batch size for mini-batch TD and Impression GTD are both 10. For Impression GTD, it converges with large step-sizes for this example. So a step-size of 10.0 is used. MiniBatchTD used a step-size of 0.05. All the hyper-parameter of the other algorithms were the same as in \citep{martha2020gradient}. Impression GTD waited until both buffers are bigger than the batch size. So there is a flat curve in the beginning. 
HTD’s curve was almost the same as TD and thus it is not shown.

Figure \ref{fig:boyan_rmse_imGTD_bsz} shows the effect of the batch size for Impression GTD. We select the top two baselines after about 1,500 steps in Figure \ref{fig:boyan_rmse}, which are TD and TDRC. Because this problem is on-policy learning, TD converges fast and it stands for the ceiling  for $O(n)$ gradient TD methods in the convergence rate. 
Comparing to TD and TDRC, all the impression GTD algorithms have a steeper drop in the loss, though bigger batch sizes need to wait a bit longer to kick start learning.    Impression GTD with bigger batch sizes (e.g, 32, 64, and 128) is able to learn significantly faster than TD and TDRC. The acceleration in convergence rate seems to decrease after batch size 32.  

\begin{figure}[t]
\caption{Boyan chain: Batch size effect for Impression GTD.}\label{fig:boyan_rmse_imGTD_bsz}
\centering
\includegraphics[width=0.9\textwidth]{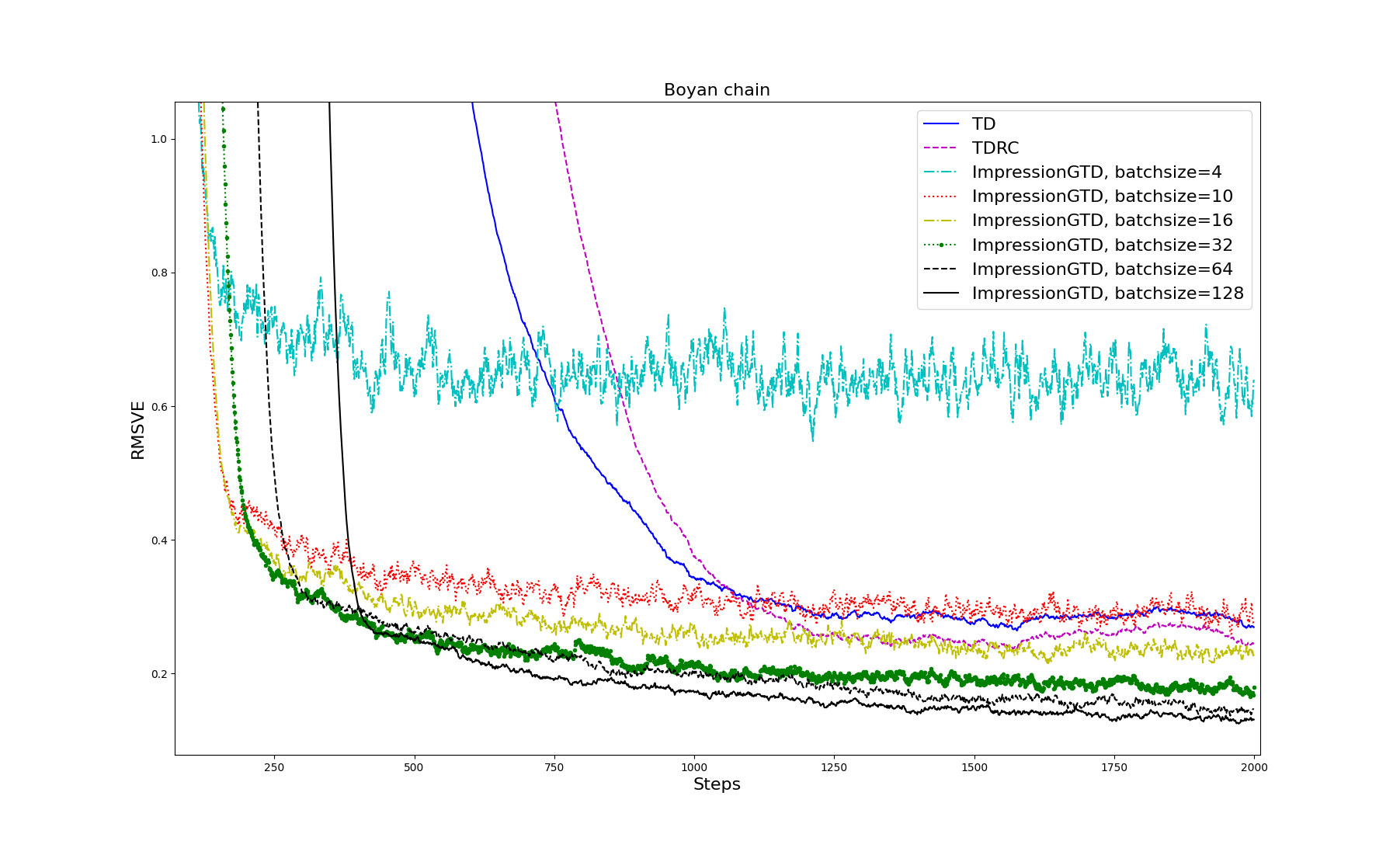}
\end{figure}

Figure \ref{fig:boyan_rmse_imGTD_stepsize} shows the effect of the step-size for Impression GTD, which all used batch sizes equal to 16. All the four step-sizes performed faster than TD whose step-size was tuned near optimal by \citeauthor{martha2020gradient}, which was 0.0625. The four step-sizes used for plotting this figure were 0.1, 1.0, 5.0 and 10.0. Their value range being big whilst learning all faster than TD means that tuning the step-size for Impression GTD is not as sensitive as the GTD algorithms.  

\begin{figure}[t]
\caption{Boyan chain: Step-size effect for Impression GTD.}\label{fig:boyan_rmse_imGTD_stepsize}
\centering
\includegraphics[width=0.9\textwidth]{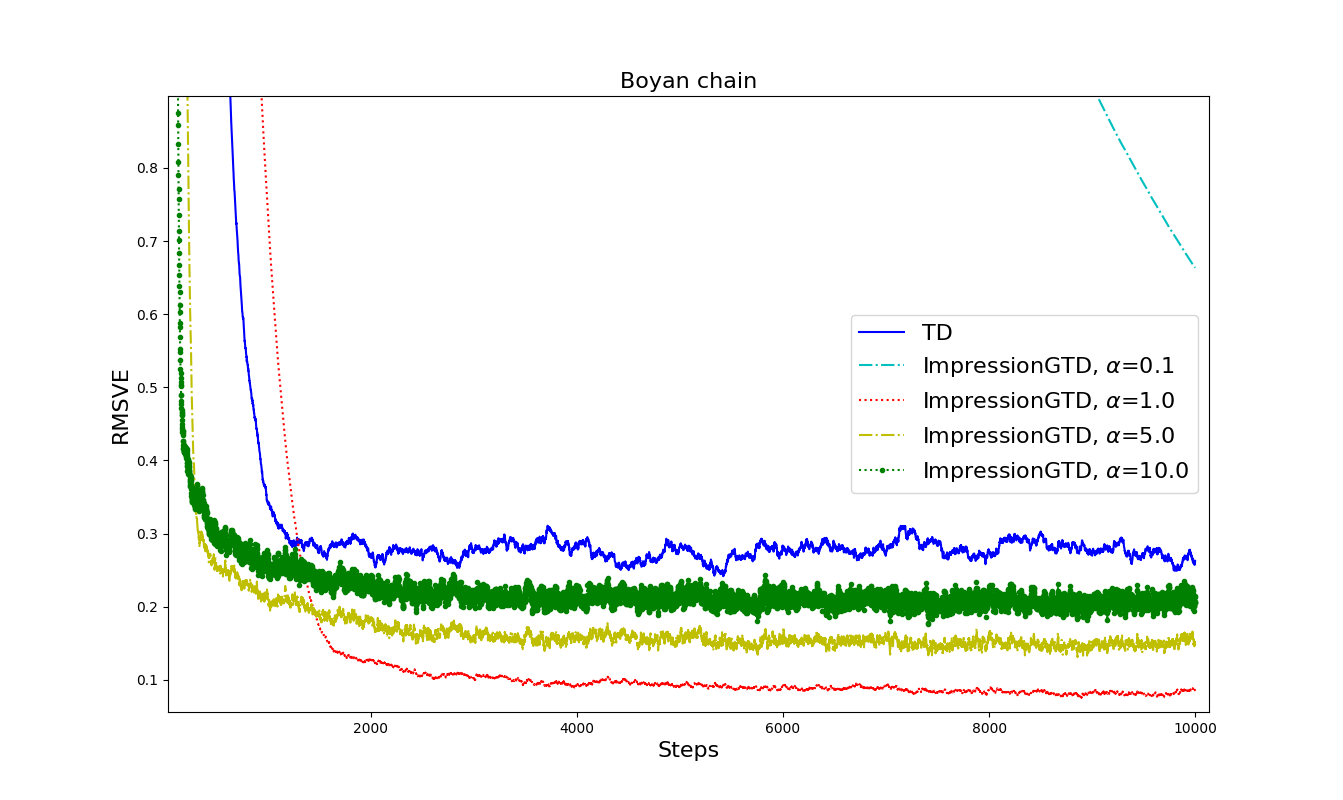}
\end{figure}

\subsubsection{\bfseries Are we getting a linear rate?} By just looking at the rate plots of Impression GTD, e.g., Figure \ref{fig:boyan_rmse_imGTD_bsz} and Figure \ref{fig:boyan_rmse_imGTD_stepsize}, it is hard to see whether the algorithm converges at a linear rate or a rate of $O(1/t)$. The overall curve looks more like the rate of $O(1/t)$ actually.  
However, note that there is a nonzero bias that plays a role in the end. Thus the convergence rate curves shown in the figures can still be a linear convergence, to a biased solution. 

To investigate which rate Impression GTD is getting, we re-plot Figure \ref{fig:boyan_rmse_imGTD_bsz}, using exactly  the same data. The only change in the re-plotting is that we subtracted a bias estimate, which is taken to be the average of the last 100 steps for each curve, discounted by a factor of 0.8, to consider the errors shown in our running steps are still decreasing in the end. Thus this just shifts the curves of the algorithms in Figure \ref{fig:boyan_rmse_imGTD_bsz} by an algorithm-dependent constant, in the $y$-axis. Then finally, we use the log scale for the $y$-axis. This is shown in Figure \ref{fig:rmsve_boyan_batchsize_linear_rate}. First, we can take a look at the curve of TD. After the subtraction, there are still significant errors (around $0.08$) throughout the most learning steps. Thus {\em TD does not have a linear convergence rate}.  Now take a look at Impression GTD. For small batch sizes like 4, the algorithm has a similar rate to TD. As we increase the batch size, the rate becomes approximately linear. This is reflected in that in linear-$x$ and log-$y$ plot, the error curve is almost linear. To confirm this in another way, we additionally plotted the learning curve of Expected GTD, using also a step-size of 5.0. After some number of sufficient samples, Expected GTD basically iterates using a good matrix that is close to the one underlying the NEU objective, and thus the rate is approximately linear. We see in Figure \ref{fig:rmsve_boyan_batchsize_linear_rate} that Impression GTD with a large batch size 128 gets very close to the rate of Expected GTD. For clarity of the presentation, only a subset of curves from Figure \ref{fig:boyan_rmse_imGTD_bsz} are shown in Figure \ref{fig:rmsve_boyan_batchsize_linear_rate}. The omitted curves of the other batch-sizes are between the shown ones, and they support this conclusion as well.

\begin{figure}[t]
\caption{Boyan chain: a re-presentation of Figure \ref{fig:boyan_rmse_imGTD_bsz}. An algorithm-dependent bias is subtracted from Figure 
\ref{fig:boyan_rmse_imGTD_bsz}. This shows Impression GTD algorithms converged with a linear rate (to a biased solution). }\label{fig:rmsve_boyan_batchsize_linear_rate}
\centering
\includegraphics[width=0.8\textwidth]{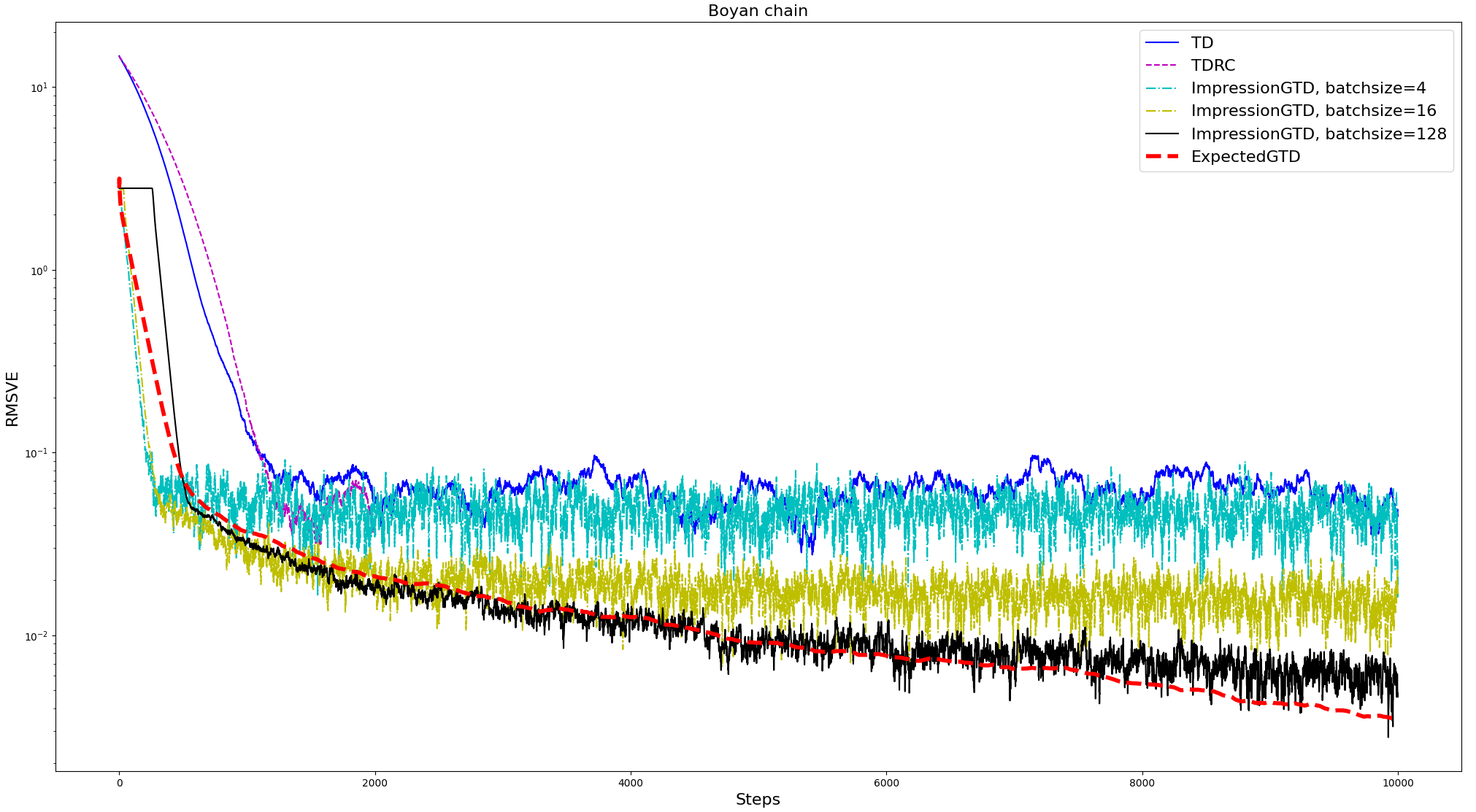}
\end{figure}

\subsection{Random Walks}
There are five intermediate states, and two terminal states (which can be treated as one terminal state). The problem is off-policy learning. The target policy goes to left with probability 40\% and to the right with 60\% probability. The behavior policy chooses the left and right actions with equal (50\%) probabilities. For the experiments in this section, the tabular representation is used.  

\begin{figure}[t]
\caption{Results on Random Walk with tabular representation (RW-tabular).}\label{fig:rwtab}
\centering
\includegraphics[width=0.9\textwidth]{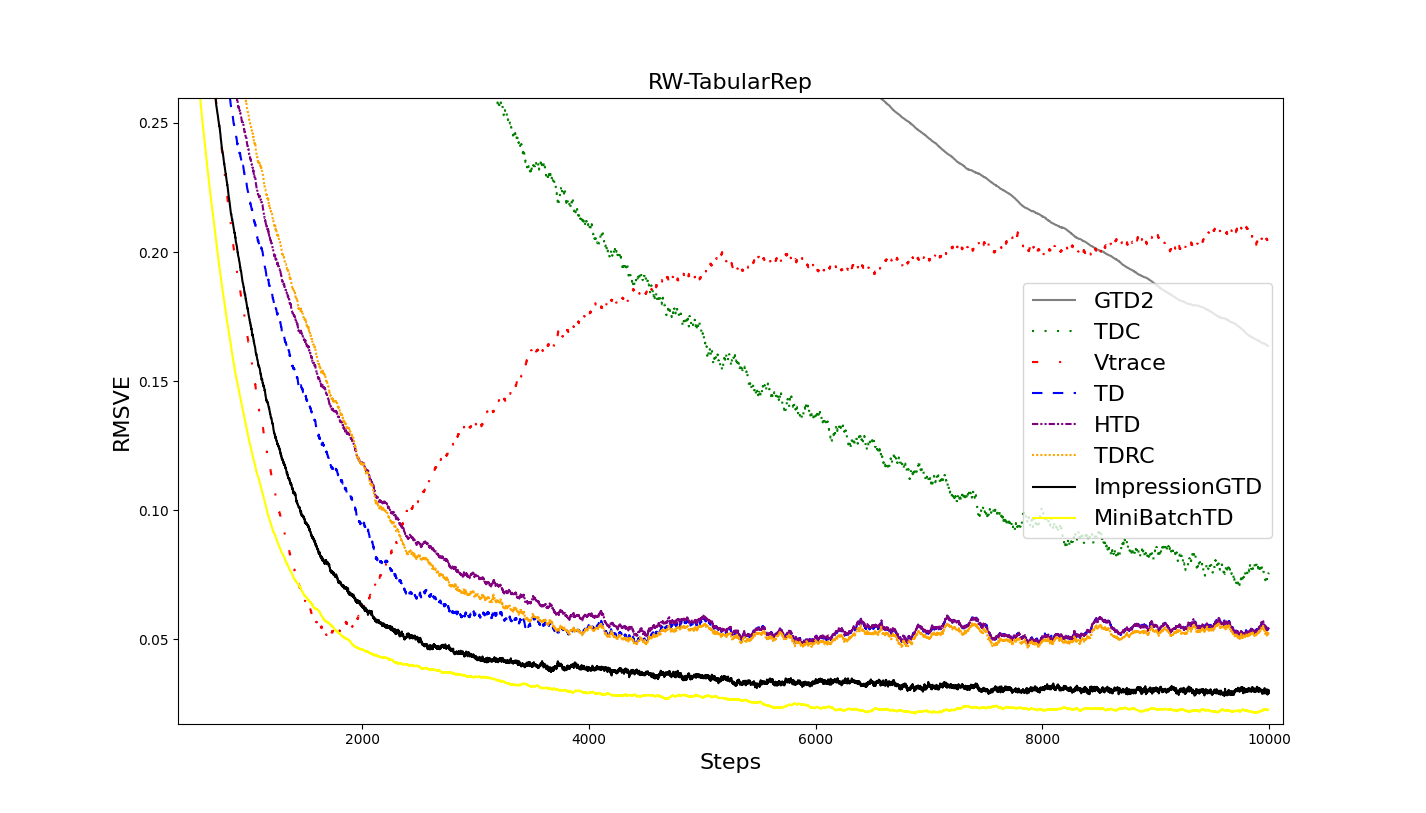}
\end{figure}

Figure \ref{fig:rwtab} shows the results. 
All the algorithms were run with the same, near-optimal hyper-parameters as used in the git repository provided by \citet{martha2020gradient}. Impression GTD used step-size 1.0, and MiniBatchTD used a step-size of 0.05.  The batch size for both algorithms is 32. Impression GTD converges much faster than the baselines including TDRC. Vtrace is a very simple algorithm. It just modifies the importance sampling ratio so that it is upper clipped at 1.0. The motivation of the algorithm is to control the variances caused by importance sampling ratios. It looks Vtrace introduces a bias with the variance reduction. 
TD, HTD and TDRC are  faster than the other baselines, and the gap among the three are small. 

\begin{figure}[t]
\caption{Batch size effect of Impression GTD on RW-tab.}\label{fig:rwtab_batch}
\centering
\includegraphics[width=0.9\textwidth]{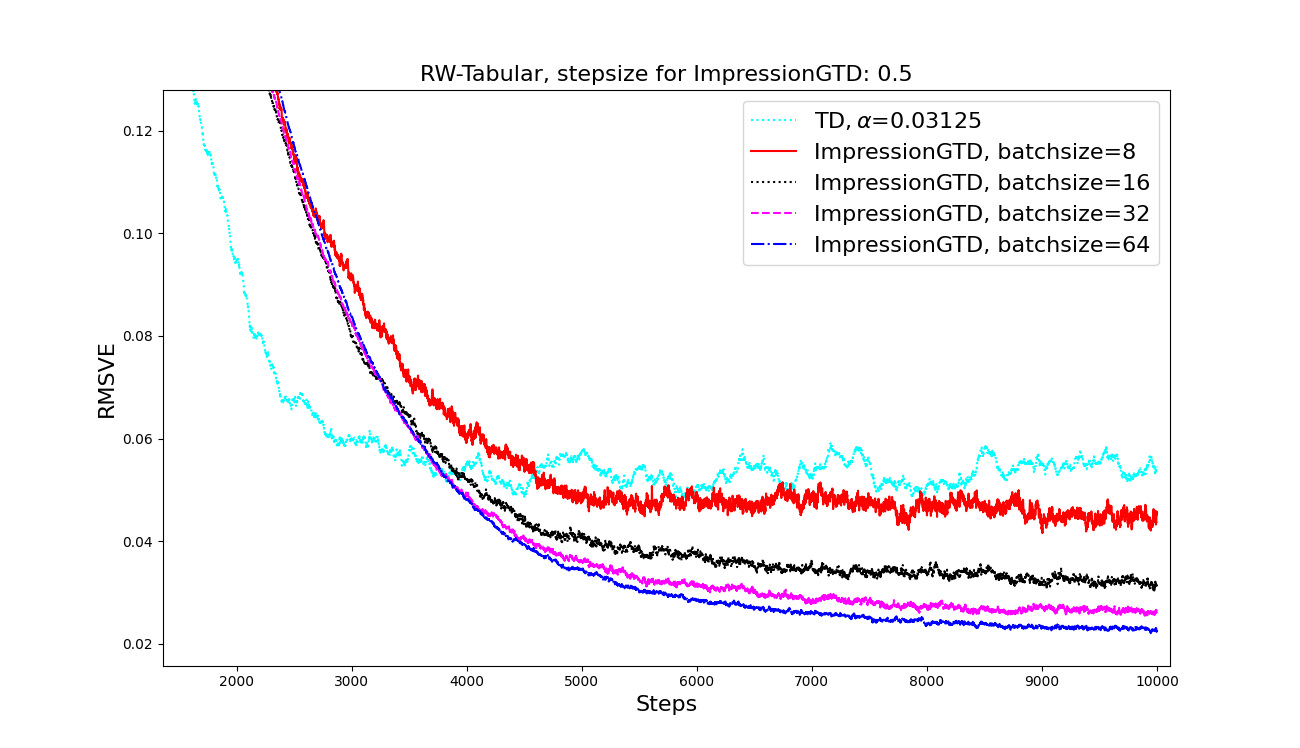}
\end{figure}
The effect of the batch size for Impression GTD is shown in Figure \ref{fig:rwtab_batch}. The step-size of Impression GTD is uniformly 0.5.   After 4500 steps, all the Impression GTD agents were faster than TD. The acceleration is more with a bigger batch size. However, note that a bigger batch size also means more computation complexity per time step.  This can be accelerated with GPU computation for the mini-batch policy evaluation procedure.

\begin{figure}[t]
\caption{Step-size effect for Impression GTD on RW-tab.}\label{fig:rwtab_alpha}
\centering
\includegraphics[width=0.9\textwidth]{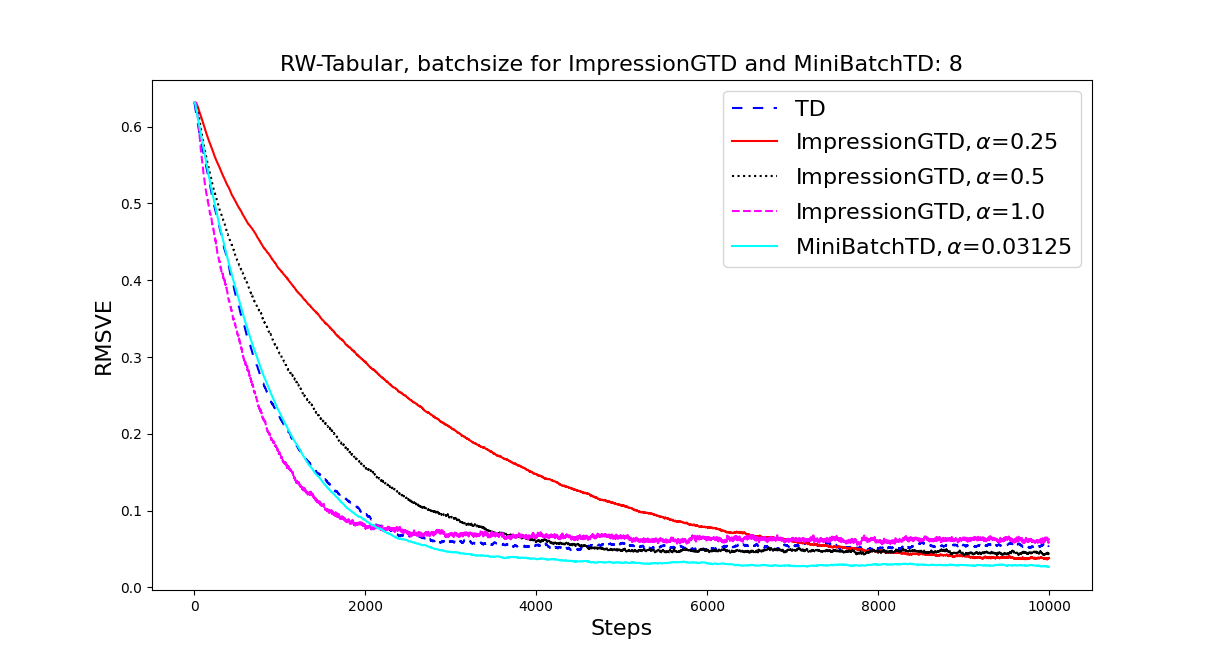}
\end{figure}

Figure \ref{fig:rwtab_alpha} shows the effect of the step-size for Impression GTD. This shows that a bigger step-size is faster in the beginning. However, there is a {\em convergence rate  overturn}. For example, Impression GTD with $\alpha=1.0$ crosses with $\alpha=0.25$ and 0.5 at about 8100 steps and 7000 steps, respectively.  After the crossing points, $\alpha=1.0$ is slower than a smaller step-size.  This is consistent with the Boyan chain results. We also plotted the MiniBatchTD, which used the same step-size as the TD algorithm. It converges faster than TD after about 2000 steps. In the end, MiniBatchTD was fastest algorithm. However, keep in mind that TD and MiniBatchTD just converges for this off-policy learning task by chance, and they may diverge for general off-policy learning. The batch size used for this plot is 8.

\subsection{Discussion on the Empirical Results}
In a summary, Impression GTD is fastest in all the compared algorithms, with a single step-size that is much easier to tune than the two-time scale GTD algorithms. Our results show that MSPBE should be interpreted with care, as the literature seems to favor this measure over NEU or RMSVE \citep{sutton2018reinforcement,martha2020gradient}. Our results show that a larger MSPBE does not necessarily mean a worse RMSVE. Vice versa. A small MSPBE does not necessarily imply good learning either. 
What this finding means is that {\em MSPBE is a surrogate loss for NEU}. 

At the time of writing this, we are not sure which one is better or not than the other. However, there is one thing we are sure: If the preconditioner is not well conditioned,  MSPBE does not delegate NEU well. In numerical analysis and iterative algorithms, people would normally avoid choosing an ill-conditioned preconditioner, at least not worsening the conditioning of the original problem. To accelerate TD and GTD, the spectral radius of the underlying iteration should be reduced \citep{ptd_yao}. With ill-conditioned preconditioner, this is hardly achieved.  \citet{ran_gtd} had similar comments on the influence of the conditioning of matrix $C$ on learning objectives.
In the off-policy learning setting, the preconditioner in MSPBE is not directly chosen by the users like in numerical analysis (where the preconditioner matrix is chosen on the fly). Instead, the matrix $C$ is dependent on the behavior policy and the features, and its preconditioning effect is realized in a {\em stochastic} way (this is very different from iterative algorithm in numerical analysis), because only samples of $C$ are applied in the learning update. Given this distinctiveness, the preconditioner is often not obvious to algorithm users in off-policy learning.  
We should perhaps avoid possible pitfalls of using preconditioning such as MSPBE in off-policy learning, and also take caution in interpreting the learning results. An example is a debug analysis of TDC on Baird counterexample performed by the authors \citep{}, which thoroughly revealed why it performed so slow as illustrated by \citep{sutton2018reinforcement}. We found it was due to the use of a singular preconditioner.   

\section{Conclusion and Discussion}\label{sec:conclusion}
We proposed a new Gradient TD algorithm for off-policy learning in reinforcement learning. 
Because of the use of two or more independent TD errors, weighted by the similarity between the feature vectors of the states where the TD errors happen, off-policy learning becomes a fully supervised learning problem once the data has been collected. 
In this paper, we have focused on the formulation and optimization parts of the problem, and established the convergence rates of the resulting {\em truly single-time-scale} SGD algorithm for off-policy learning. With only one step-size parameter, the new algorithm is much easier to use for practitioners.  


Under the constant step-size case, our first rate shows that Impression GTD converges at a rate of $O(1/t)$. The rates of GTD, GTD2 and TDC were established previously in a saddle-point formulation  \citep{bo_gtd_finite}, and a two-time-scale formulation \citep{dalal2018finite_twotimescale,dalal2020tale_twotimescale}. Both analysis show the three GTD algorithms converge slower than $1/\sqrt{t}$. 

Our second rate analysis is a linear rate result for Impression GTD under constant step-sizes. The analysis draws on a SGD result that is novel in two ways. First, the expected smoothness condition \citep{gower2019sgd_general} is relaxed to a new class of loss functions, in the sense that the smoothness measure is allowed to have some extra noise. Second, our rate yields a tighter bound, introducing a squared term of the strong quasi-convexity factor. This technique can be used to improve bounds of the SGD rate in literature, e.g., Theorem 3.1 of \citep{gower2019sgd_general}, as we showed in the paper. 
Besides Impression GTD, we also proved the convergence rates of three other GTD algorithms, including the one by \citet{ptd_yao} and another discussed by \citet{gtd}.

The empirical results of on/off-policy learning on three problems show that Impression GTD learns much faster than exiting off-policy TD algorithms. Our parameter studies of the step-size and the batch size shows that the new algorithm is very easy to use. Tuning the single step-size for Impression GTD is much easier than tuning the two step-sizes of existing GTD algorithms. 

Empirical results show that larger step-sizes make Impression GTD converge faster. However, after increasing to certain step-size values, the acceleration gets smaller. The phenomenon is well explained by our theory. In particular, Theorem \ref{thm:rates_all} shows that the step-size should be inverse proportional to the ratio between the variance of the transition features and the batch size. 
This means two things. First, larger batch sizes induce less disturbance into the smoothness measure and as a result, a larger step-size gives a faster convergence. 
Second, increasing the batch size beyond a certain threshold (proportional to the variance of the transition features) does not help with the convergence rate much any more. The rate will then be bounded by the convergence factors of the full gradient descent, assuming the true loss function is given. 

Another phenomenon we observed in the experiments is that larger step-sizes induced larger biases in the end, though faster convergence in the beginning. This can be explained by Theorem \ref{thm:rates_all} as well. In particular, the bias is proportional to 
\[
\frac{\alpha^2}{m}\propto \left(\frac{1}{5\frac{\norm{\Sigma_A}^2}{m} +  4\norm{A}^2}\right)^2 \frac{1}{m}.
\]
Thus the bias reduces as we decrease the batch size $m$. Unless the transition variance is small (comparing to $\norm{A}^2$), the bias also increases when we increase the batch size. 

A limitation of the present paper is that all the theory and experiments are about policy evaluation. It is interesting to extend the results to control and nonlinear function approximation in the future.



\acks{
We appreciate Thomas Walsh, James MacGlashan, and Peter Stone for insightful discussions on the topics of off-policy learning and deep reinforcement learning, who also helped improve the paper in many ways. Tom spotted a problem in an early draft of Theorem 2. Tom and Peter also gave lots of advice that greatly helped improve the presentation of the paper.   
We would like to thank Declan Oller for the pointer to the paper by \citet*{lihong_kernel_sim}, which helped improve our understanding of Impression GTD. 
We also thank Varun Kompella, Dustin Morrill and Ishan Durugkar for interesting questions and helpful discussions. 
We appreciate Sina Ghiassian, Andrew Patterson, Shivam Garg, Dhawal Gupta, Adam White and Martha White for making their TDRC code available, and Shangtong Zhang for the Baird counterexample, both of which greatly facilitate the experiment studies in this paper. We appreciate Shangtong Zhang also for helpful discussions on importance sampling and off-policy learning. We thank Csaba Szepesvari for helpful discussions on the information-theoretically near-optimal rate for linear stochastic approximation algorithms. We thank Bo Liu for helpful discussions on the saddle-point problem, the single-time-scale formulation of GTD, and ETD.    
}










\vskip 0.2in
\bibliography{ref}

\appendix
\section{}
To give the proofs of lemmas in the paper, we first introduce a few results including the definition of a convexity and a lemma for it. 
\begin{definition}[q-convex]\footnote{
The notion of q-convex here does not imply the convexity of $f$, nor is it limited to the uniqueness of $x^*$.
This definition is different from ``quasi-convex'', which means something else, e.g., see \citep{Kiwiel_quasi_convex,quasi_convex_hu2019convergence}. In particular, $f(\lambda x + (1-\lambda)y)\le \max\{f(x), f(y)\}$ holds for all $x,y\in \mathcal{D}(f)$ and any $\lambda \in[0,1]$. }  
Let a differentiable function $f$ be defined by $f:\RR^d \to \RR$. 
Let $x^*=\min_{x\in \RR^d }f(x)$. In addition, $f'(x^*)=0$.
If 
\begin{equation}\label{eq:qconvex1}
f(x^*) \ge f(x) + f'(x)^\tr (x^*-x)
\end{equation}
holds for all $x\in \RR^d$, then we call $f$ {\em q-convex}. 
\end{definition}

\begin{lem}\label{lem:qconvexiif}
Assume $f$ is q-convex. Recall that $x^*$ is the optimum with $f(x^*)\le f(x)$ for all $x$. 
For any $ x\in \RR^d$, let $y = \lambda x^* + (1-\lambda)x$. We have,
\[
-f'(y)^\tr (x^*-x)
\begin{cases}
        \ge 0, & \text{if}\ \lambda\in[0,1]; \\
      <0, & \text{otherwise.}
    \end{cases}
\]

\end{lem}
That is, any intermediate point between a current point $x$ and the optimum $x^*$ is guaranteed to have a negative gradient that is positively correlated with the direction of $x^*-x$. Extrapolation outside of the two points gives a reverse relationship. 

\begin{proof}
Setting $x=y$ in equation \ref{eq:qconvex1}, this still holds. Thus
\begin{align*}
    f(x^*) &\ge f(y) + f'(y)^\tr (x^*-y)\\
    &= f(y) + f'(y)^\tr (x^*-(\lambda x^* + (1-\lambda)x))\\
    &= f(y) + (1-\lambda) f'(y)^\tr (x^*-x).
\end{align*}
Note that $f(y)\ge f(x^*)$. 
Thus 
\[
 -(1-\lambda)f'(y)^\tr (x^*-x) \ge 0.
\]

\end{proof}

\begin{lem}\label{lem:g2andf2}
Refer to Definition \ref{def:L-lambda-smooth} for the generation process of $g_t(x)$. 
Let $\mathcal{D}$ be any distribution that satisfies $\EE_{\mathcal{D}} [g_t(x)|x] = f'(x)$ for all $x \in \RR^d$. Given $f$ and $\mathcal{D}$, we define $
\sigma_v^2 = \min_x \EE_{\mathcal{D}}\norm{g_t(x)-f'(x)}^2$.
For all $x$, we have 
\[
\EED \norm{g_t(x) - f'(x)}^2=\EED \norm{g_t(x)}^2 - \norm{f'(x)}^2;  \quad \EE_{\mathcal{D}}\norm{g_t(x)}^2\ge \norm{f'(x)}^2 + \sigma_v^2.
\]
\end{lem}
\begin{proof}
This can be seen from 
\begin{align*}
\sigma_v^2 &\le 
\EE_{\mathcal{D}}\left[\norm{g_t(x) - f'(x)}^2|x\right] \\
&=
\EE_{\mathcal{D}}\left[\norm{g_t(x)}^2|x\right] - 2 \EE_{\mathcal{D}}\left[g_t(x)^\tr f'(x)|x\right] + \EE_{\mathcal{D}}\left[\norm{f'(x)}^2 |x\right]\\
&=
\EED\norm{g_t(x)}^2 - 2 \EED[g_t(x)]^\tr \EED f'(x) + \norm{f'(x)}^2 \\
&=
\EED\norm{g_t(x)}^2 - 2 \EED[g_t(x)]^\tr f'(x) + \norm{f'(x)}^2 \\
&=
\EED\norm{g_t(x)}^2 - 2 \norm{f'(x)}^2 + \norm{f'(x)}^2 \\
&=
\EED\norm{g_t(x)}^2 -  \norm{f'(x)}^2,  
\end{align*}
where the conditional on $x$ is omitted whenever there is no need to emphasize.
\end{proof}

\subsection{Proof for Lemma \ref{lem:ED_avg}}\label{appendix:ED_avg}
\begin{proof}
Using the familiar notation of $\bar{X}$ to denote the empirical mean of multiple i.i.d (gradient) samples and $X$ refers to the stochastic gradient in this context.  We have
\begin{align*}
\EED \norm{\avgx}^2 &=  \EED \norm{\frac{1}{m}\sum_{i=1}^m
g(x|X_i,Y_i)}^2 \\
&=  \EED   
\norm{\bar{X}}^2 \\
&= \EED   
\sum_{i=1}^d\bar{X}[i]^2 \\
&=  
\sum_{i=1}^d\EED \bar{X}[i]^2 \\
&=  
\sum_{i=1}^d Var(\bar{X}[i]) + (\EED\bar{X}[i])^2 \\
&=  
\sum_{i=1}^d \frac{1}{m}Var({X}[i]) + (\EED{X}[i])^2 
\end{align*}
where we used $Var(\bar{X})=\frac{1}{m}Var(X)$ because all the samples are i.i.d. It follows that 
\begin{align*}
\EED \norm{\avgx}^2&=  
\sum_{i=1}^d \frac{1}{m}\left[\EED{X}[i]^2 - (\EED X[i])^2\right] + (\EED{X}[i])^2 \\
&=  
\frac{1}{m}\sum_{i=1}^d \EED{X}[i]^2  + \frac{m-1}{m}\sum_{i=1}^d(\EED{X}[i])^2 \\
&= \frac{1}{m}\EED \norm{g_t(x)}^2 + \left(1-\frac{1}{m}\right) \norm{f'(x_t)}^2. 
\end{align*}
where in the third line we plugged in the original notations. 
\end{proof}
This lemma shows that the expected squared norm of the averaged gradient is not just shrinking by a factor of one over the batch size, but also it adds $(1-1/m)$ times the squared norm of the true gradient. The appearance is a convex sum of the two. 

\subsection{Proof of Lemma \ref{lem:utrongly_norm_grad}}\label{appendix:u_norm_grad}
\begin{proof}
Because $f$ is $\mu$-strongly quasi-convex, we have 
\[
f(x^*) \ge f(x) + \nabla f(x)^\top (x^* -x ) + \frac{\mu}{2} \norm{x^* -x}^2.
\]
Let $g(x) = f(x) -\frac{\mu}{2}\norm{x-x^*}^2$. Note $g'(x^*)=0$. 
We have, $g(x^*)\ge g(x)+g'(x)^\tr (x^*-x)$ holds if and only if $f$ is $\mu$-strongly quasi-convex. Thus $g$ is q-convex and we have $-g'(y)^\tr (x^*-x) \ge 0$, for $y=\lambda x^*+ (1-\lambda) x$, $\forall \lambda\in[0,1]$, according to Lemma \ref{lem:qconvexiif}. Note that $y-x^*=(1-\lambda) (x-x^*)$. Thus $g'(y)^\tr (y-x^*) \ge 0$. 
Expanding this inequality we have
\begin{align*}
    0&\le  g'(y)^\tr (y-x^*) \\
    &= (f'(y) - \mu(y-x^*))^\tr (y-x^*). 
\end{align*}
This gives $f'(y)^\tr (y-x^*)\ge \mu\norm{y-x^*}^2$. Using Cauchy–Schwarz inequality, we have 
\begin{align*}
\norm{f'(y)} \norm{y-x^*}\ge f'(y)^\tr (y-x^*)\ge \mu\norm{y-x^*}^2. 
\end{align*}
If $y=x^*$, the lemma holds; otherwise, 
dividing by $\norm{y-x^*}$, we have 
$\norm{f'(y)}\ge \mu\norm{y-x^*}$. Since this holds for $y$ generated for any $\lambda\in[0,1]$, setting $\lambda=1$ also holds. Thus this holds for any $x\in \RR^d$.   
\end{proof}

\subsection{Proof of Lemma \ref{lem:ES_qLsmooth}}

\begin{proof}\label{appendix:es_qls}
We have
\[
\sigma^2_v + \norm{f'(x)}^2 \le \EE_{\mathcal{D}}\norm{g_t(x)}^2 \le 2L (f(x) - f(x^*))  + \lambda \norm{x-x^*}^2+ \sigma^2, 
\]
where the first inequality uses Lemma \ref{lem:g2andf2} and the second the $L$-$\lambda$ smoothness condition.  
\end{proof}

\subsection{RW-inv}\label{exp:rwinv}

In this random-walk problem, RW-inv, the representation inverts the tabular representation, and switches zero for one, and one for zero for the features. Then normalization is applied row-wise to the feature matrix.
First we run the baselines and measured the RMSPBE, and the results are shown in Figure \ref{fig:rwinvpbe}. The hyper-parameters were the same as \citep{martha2020gradient}. If we compare with the Figure 1 for this problem in their paper,  they showed that TDRC performed the best. However, note that there the metrics were taken for 3000 steps. As learning extends more steps,  TDC gets much better than TDRC. HTD is pretty close to TD. GTD2 and Vtrace have the largest RMSPBE after 4000 steps.

\begin{figure}[t]
\caption{RW-inv: TDC is better than TDRC after 3000 steps (RMSPBE).}\label{fig:rwinvpbe}
\centering
\includegraphics[width=0.9\textwidth]{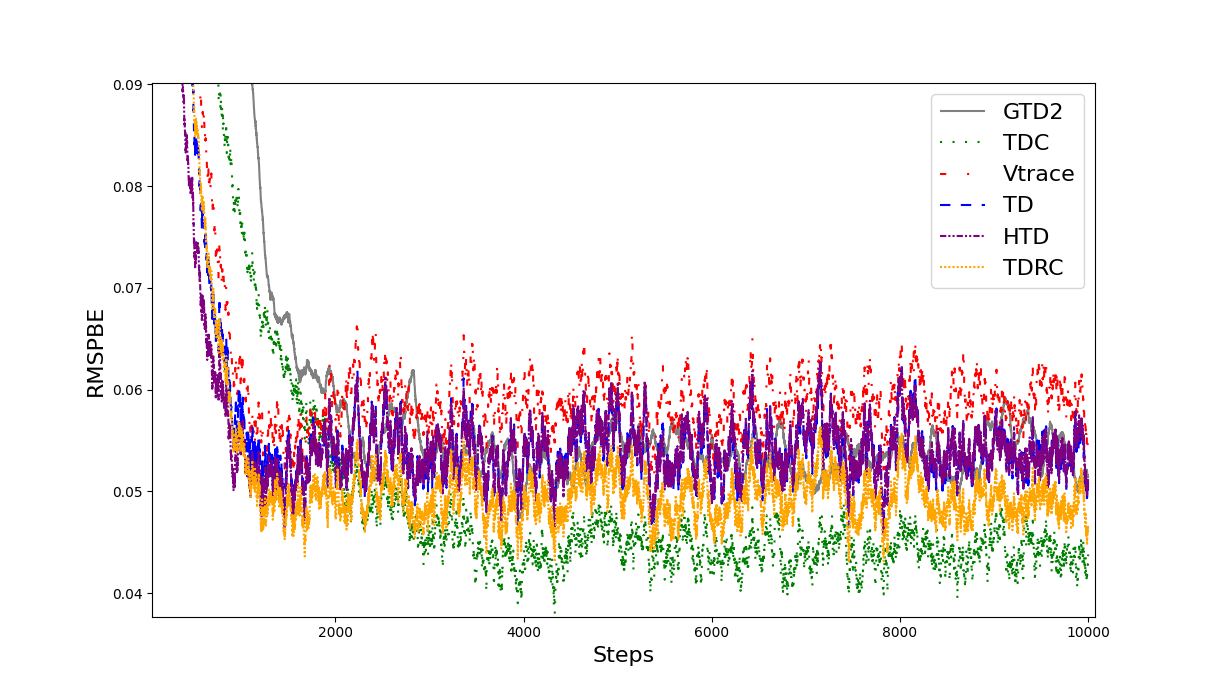}
\end{figure}

Next we run all the algorithms and measure the RMSVE in Figure \ref{fig:rwinvrmsve}. Surprisingly, after the initial learning stage, GTD2 and TDC perform the best in all the baselines, much better than TD, TDRC and the other baselines. Thus we should take care in interpreting MSPBE. A large MSPBE does not necessarily mean bad learning, e.g., GTD2, as shown in Figure \ref{fig:rwinvpbe}. Impression GTD used a batch size of 32 and a step-size of 1.0. After the initial learning stage, Impression GTD performs the best, even better than the unusually fast GTD2 and TDC in this case.

\begin{figure}[t]
\caption{RW-inv: Impression GTD is still faster than GTD2 and TDC,  although they are unusually fast for this problem.}\label{fig:rwinvrmsve}
\centering
\includegraphics[width=0.9\textwidth]{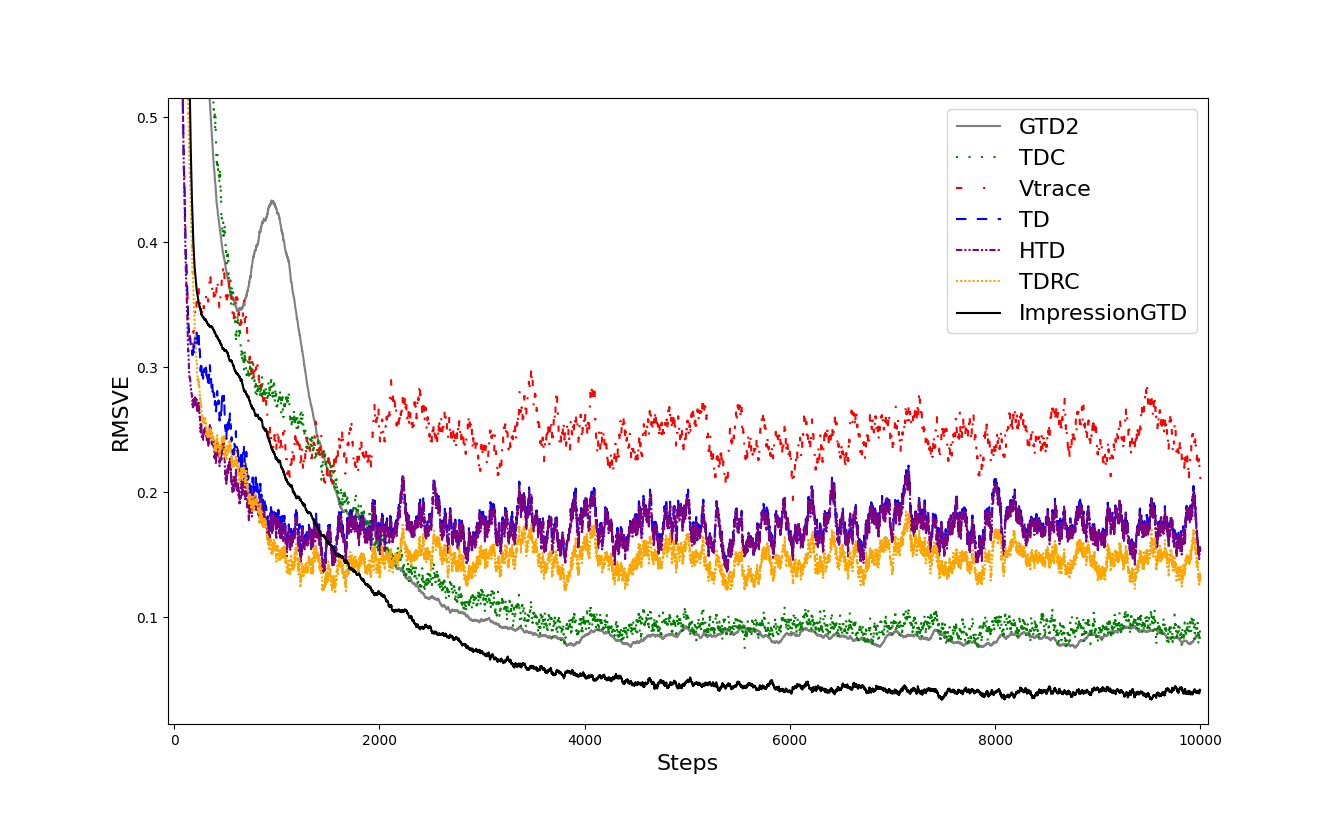}
\end{figure}

\begin{figure}[t]
\caption{RW-inv: The batch size effect for Impression GTD.}\label{fig:rwinv_batchsize}
\centering
\includegraphics[width=0.9\textwidth]{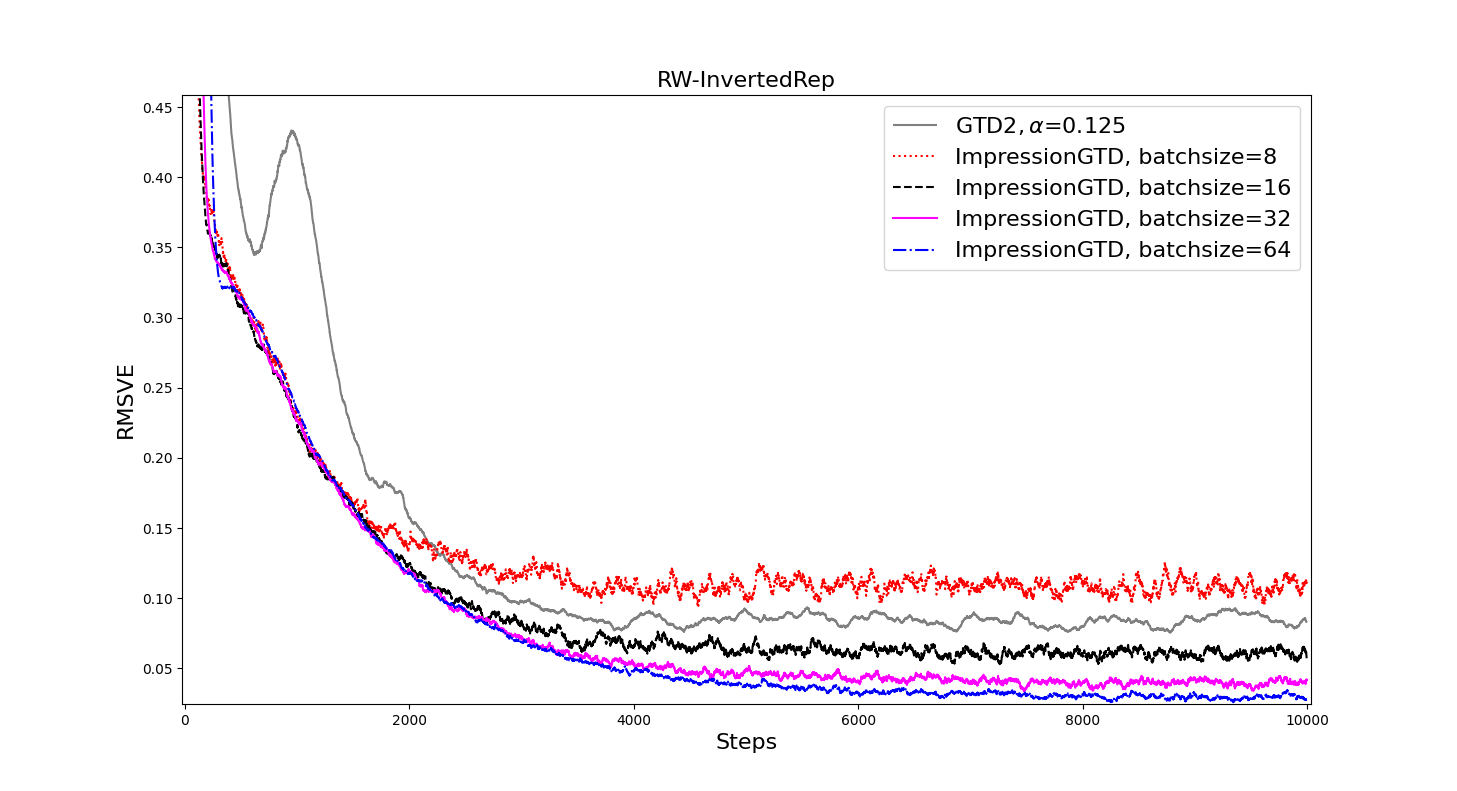}
\end{figure}

The batch size effect for Impression GTD is shown in Figure \ref{fig:rwinv_batchsize}. 
The step-size for Impression GTD agents is uniformly 1.0.   We also plotted GTD2, the best performing agent for this problem. 
This shows that a smaller batch size like 8 is slower in learning. Impression GTD agents with batch sizes 16, 32 and 64 all perform much better than GTD2 after the initial learning stage.

\begin{figure}[t]
\caption{RW-inv: The step-size effect for Impression GTD.}\label{fig:rwinv_alpha}
\centering
\includegraphics[width=0.9\textwidth]{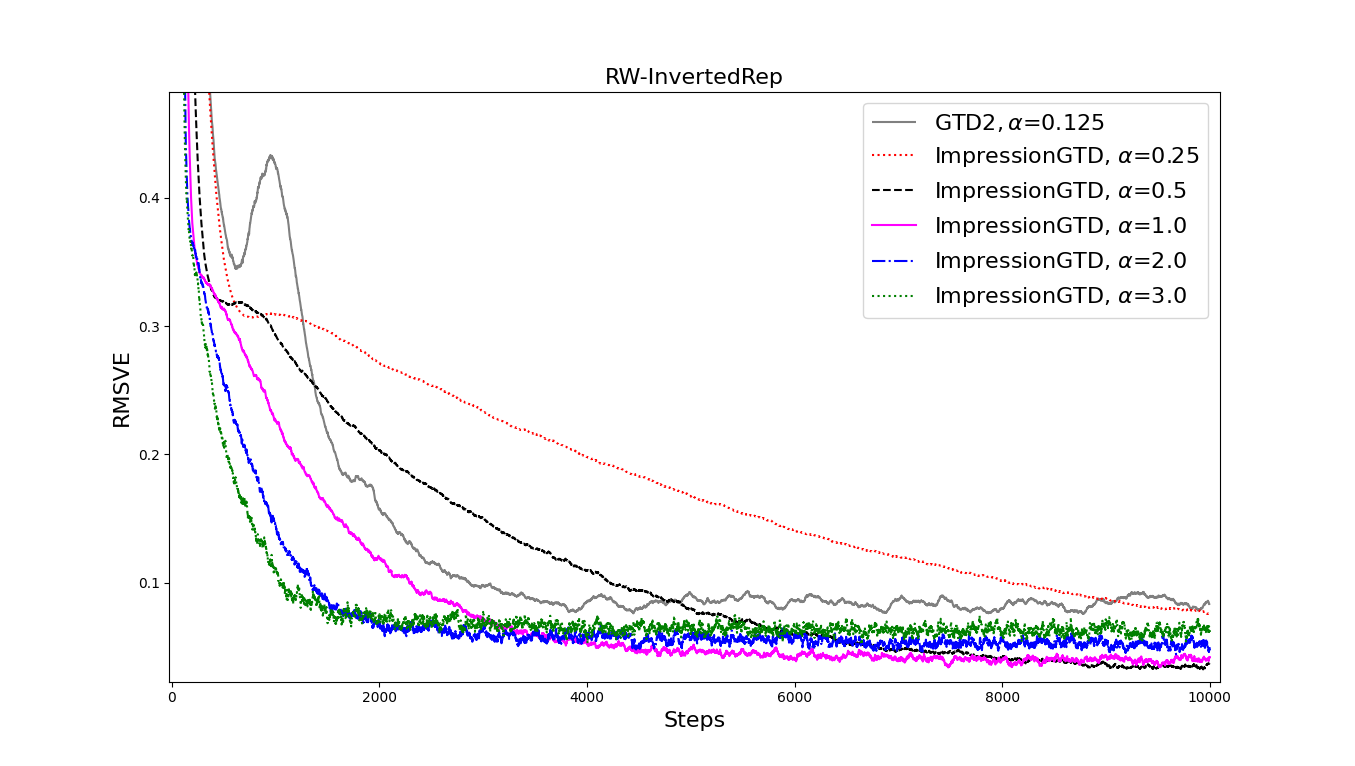}
\end{figure}

Figure \ref{fig:rwinv_alpha} shows the effect of step-size for Impression GTD. The batch size is 32. There is a convergence rate overturn similar to what we observed in the RW-tab problem. This suggests a decaying step-size  can further improve the convergence rate of Impression GTD.  

\subsection{RW-dep}\label{exp:rwdep}
In this random-walk problem, 
the representation for states 1 and 5 is the set of the two unit basis vectors. For states 2 and 4, the feature vectors are [1, 1, 0] and [0, 1, 1], respectively. State 3 is [1, 1, 1]. Finally, each feature vector is $\ell_2$ normalized.

First, we compared the algorithms in the MSPBE measure and the results are shown in Figure \ref{fig:rwdep_pbe}. Impression GTD used a batch size of 32 and a step-size of 0.05. The baseline algorithms used the same hyper-parameters as in the TDRC code base. TDRC performed better than the other baselines. However, the advantage of TDRC over TD/HTD is small. This is probably due to that TDRC mixes TD and TDC via regularizing the helper iterator \citep{martha2020gradient}. \footnote{Mixing the on-policy TD and the off-policy GTD algorithms is also the principle under which Hybrid TD (HTD) was designed \citep{HTD_leah,htd_marhta}.}. After about 2500 steps,  Impression GTD is much faster than TDRC and the others.

\begin{figure}[t]
\caption{RW-dep: Results using RMSPBE.}\label{fig:rwdep_pbe}
\centering
\includegraphics[width=0.9\textwidth]{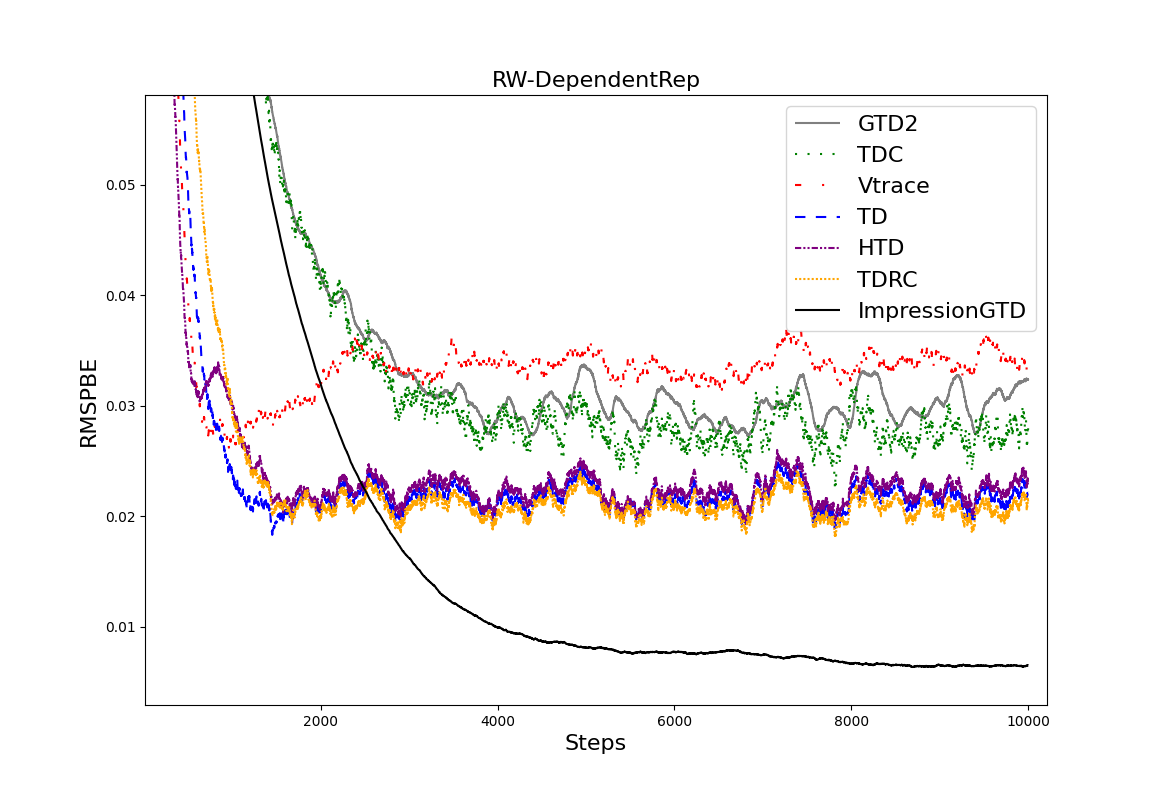}
\end{figure}

We then plotted the RMSVE metric in Figure \ref{fig:rwdep_rmsve}. It shows that Vtrace finds a solution that is far from the others, although under the MSPBE measure the solution is not very far, shown in Figure \ref{fig:rwdep_pbe}. This is another example that we should interpret the MSPBE measure carefully. GTD2 and TDC are faster than TDRC after about 2300 steps. GTD2 is also faster than TDC. This is surprising because usually GTD2 is slower than TDC in terms of the MSPBE. Impression GTD is still faster than GTD2 and the others after an initial learning time.

\begin{figure}[t]
\caption{RW-dep: Results using RMSPBE.}\label{fig:rwdep_rmsve}
\centering
\includegraphics[width=0.9\textwidth]{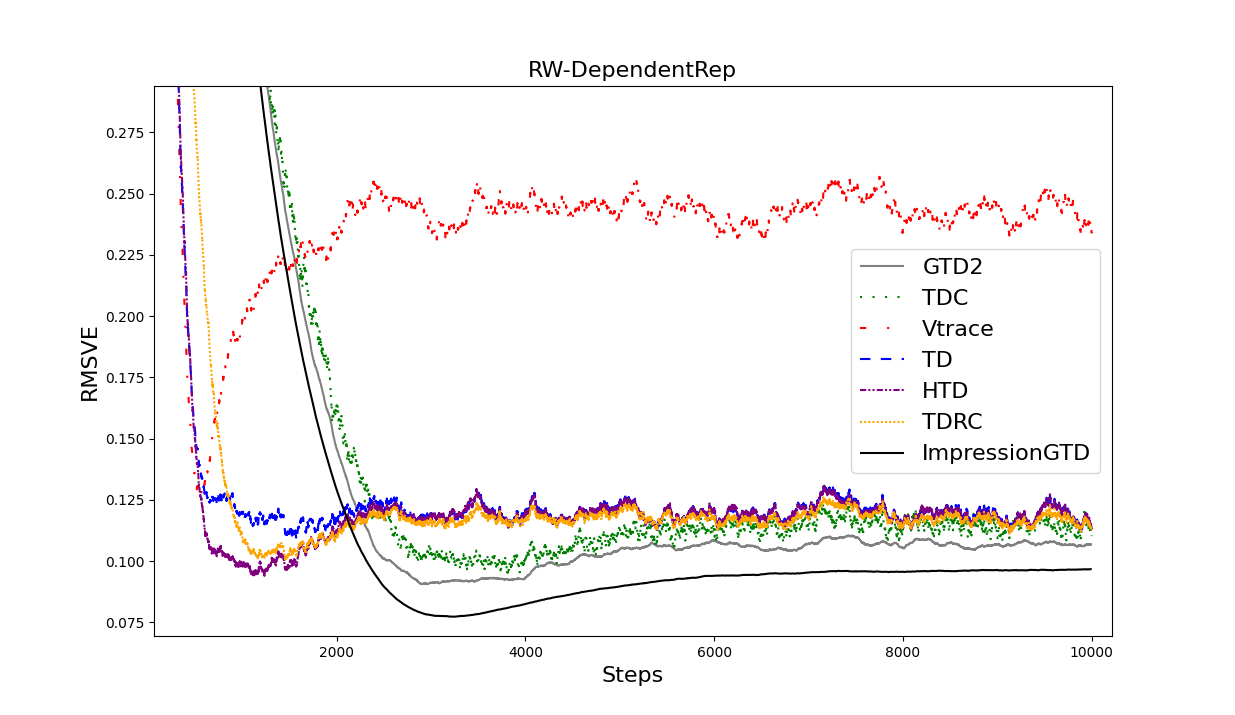}
\end{figure}

Figure \ref{fig:rwdep_batchsize} shows the effect of the batch size for Impression GTD. We also plotted the best performing GTD2. The step-size for all the ImpressoinGTD agents is 0.05. For this problem, the learning is not very hard in the beginning because of the generalization between the features. What’s interesting for this problem is that for all the algorithms, the learning deteriorates and the RMSVE metric makes a way back. This should be because the features are strongly correlated for this representation. Nonetheless,  ImpressoinGTD still learns much better than GTD2 whether in terms of the lowest RMSVE or the final plateau for all the batch sizes.  Bigger batch sizes (e.g., 32 and 64) perform slightly better than smaller ones  (e.g., 8, 16).

\begin{figure}[t]
\caption{RW-dep: Batch size effect for Impression GTD.}\label{fig:rwdep_batchsize}
\centering
\includegraphics[width=0.9\textwidth]{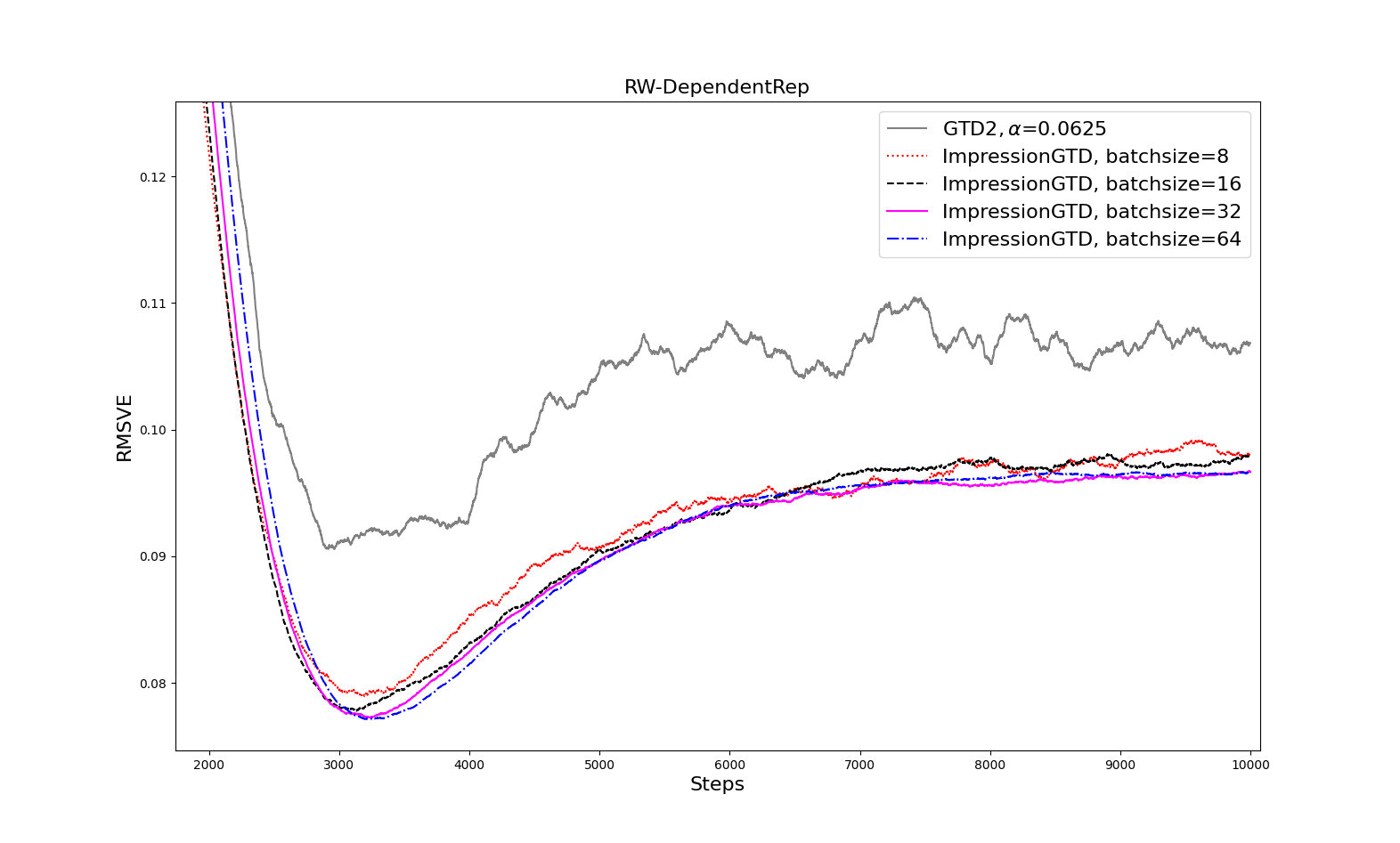}
\end{figure}

Figure \ref{fig:rwdep_stepsize} shows the effects of the step-size for Impression GTD. The batch size for all the Impression GTD agents is uniformly 32. A bigger step-size like 0.5 leads to fastest learning in the beginning. However, because this learning task benefits from the generalization in the representation, the other smaller step-sizes can also quickly minimize the learning error. We don’t know exactly why the lowest point for the blue line (step-size 0.5) in the beginning is higher than the others. Probably it is because a big step-size couldn’t go all the way to reach the bottom of the valley in the loss.  Smaller step-sizes like 0.1, 0.05 and 0.025 seems to have a lower low as we decrease the step-size. For example, the step-size 0.025 is slow before about 4000 steps. However, after that, the drop in the error is fast and the final solution is the best.

\begin{figure}[t]
\caption{RW-dep: Step-size effect for Impression GTD.}\label{fig:rwdep_stepsize}
\centering
\includegraphics[width=0.9\textwidth]{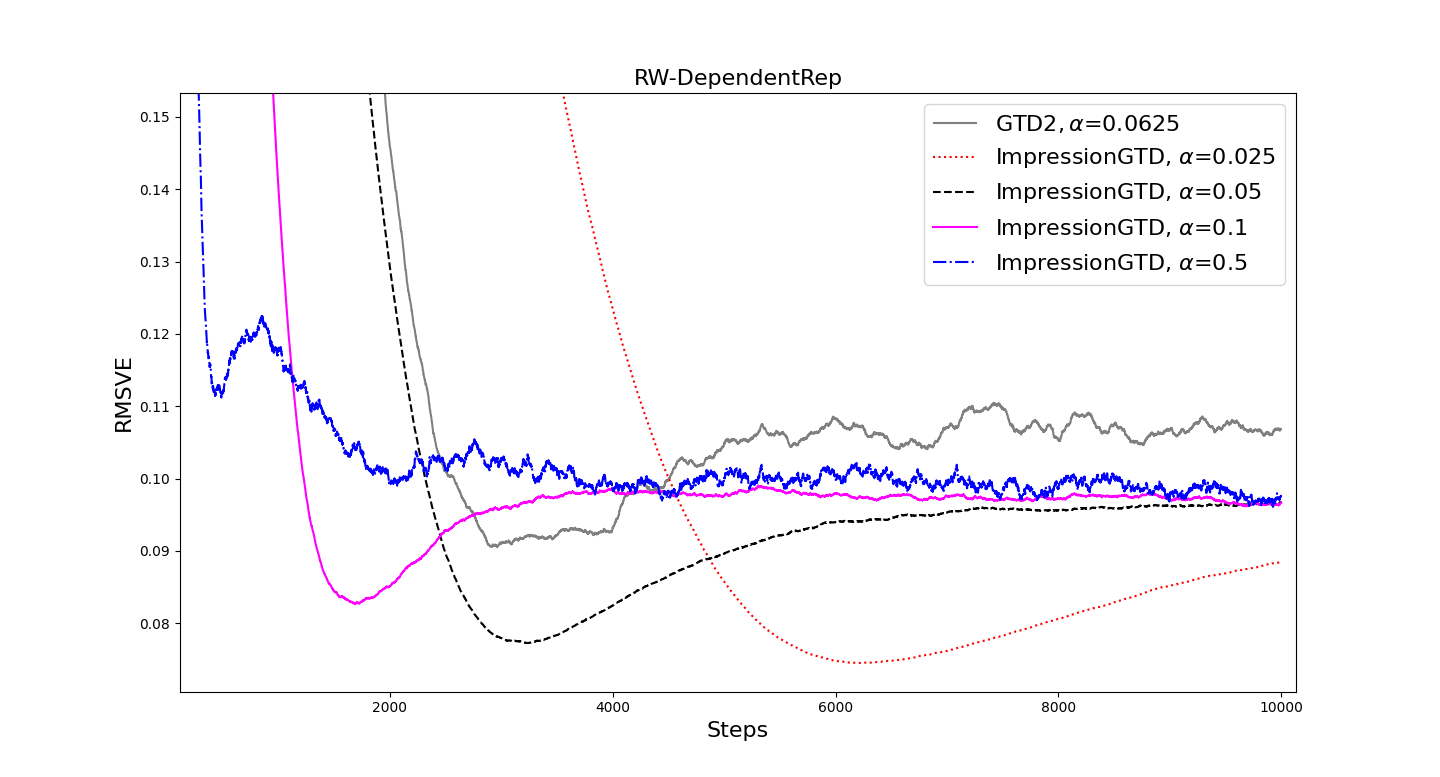}
\end{figure}

\subsection{Baird counterexample}\label{exp:baird}
We use the 7-state version of the problem by \citet{sutton2018reinforcement}. 
Although appearing simple, this problem in fact is very challenging for off-policy learning \citep{baird1995residual}. The discount factor is reduced from 0.99 to 0.9 to induce more contraction, and make it not so challenging in terms of the convergence rate. This rules out the possibility that GTD and TDC are slow because the problem is too challenging.  

The performance of Impression GTD algorithms is shown in Figure \ref{fig:baird_results}. For a clear visualization, we only show the curves of GTD, TDC and TDRC. The other baselines were not as fast as the chosen baselines. 
For TDRC, all the hyper-parameters are used the same as the TDRC paper, which were selected by the original authors from an extensive sweep search. The $\alpha$ was 0.03125, $\beta$ (the regularization factor) was 1.0, and $\eta$ was 1.0, too. We also tried bigger values of $\alpha$ (without changing $\beta$ or $\eta$), including 0.04 and 0.05. They had either much bigger variances or diverged. Impression GTD used a batch size of 10. All the algorithms are corrected by $\rho$, the importance sampling ratio.   

The Impression GTD didn’t start learning until 100 steps of following the behavior policy, filling the buffer with some content. Impression GTD agents learn very fast, with a steep drop in the value estimation error, all the way down to near zero. With a small $\alpha$ like 0.001, the algorithm converges slower, but it also drives the RMSVE down to near zero. 

The curves of Impression GTD exhibit the pattern of a  linear convergence of SGD to the optimal solution in Theorem \ref{thm:sgd_rate}, which derives the rate of Impression GTD in Theorem \ref{thm:rates_all}. Also see our discussions right before Section \ref{sec:experiments}, about why this time Impression GTD with a constant step-size converges to the optimal solution in a linear rate, instead of to a neighbourhood of it. In \citep{1_over_t_information_optiaml}, $B$ is defined to form a space of norm-bounded linear predictors, 
\[
\Theta = \{ \norm{\theta}\le B \}, 
\]
in which the solution of $\theta$ will be sought after. \citeauthor{1_over_t_information_optiaml} also assumes that the target signal $y$ is bounded: $\norm{y}\le Y$. For Baird counterexample, the true value function, or the target signal vector $y$ is zero. Also an optimal weight vector is zero. 
Thus $B=Y=0$ is a valid and tight choice for this problem, for which equation \ref{eq:worst_rate_1_over_t} does not state a valid worst case bound except an obvious fact. 

To see why Theorem \ref{thm:rates_all} correctly states the convergence of Impression GTD to the optimal solution, it suffices to examine the bias term, which has a numerator, $m(\sigma^2-\sigma^2_v) + L\alpha\sigma^2_v$. According to the definition of $\sigma^2_v$ in Lemma \ref{lem:ES_qLsmooth}, 
\[
\sigma^2_v=\min_{\theta} \EED\norm{g_t(\theta) -f'(\theta)}^2=\EED\norm{g_t(0) - f'(0)}^2=0,\]
because one optimal weight vector $\theta^*$ is equal to zero, which achieves zero TD error for every sample (an over-parameterization case). Note $f'(0)$ is zero too because $\theta^*=0$ is a stationary point of the NEU objective function. 

According to Lemma \ref{thm_item:Llambdasmoooth}, which defines the constants referred to in Theorem \ref{thm:rates_all},
\begin{align*}
\sigma^2 &= 16\left(\frac{\norm{\Sigma_A}^2}{m_1} +  \norm{A}^2\right) \left(\frac{\norm{\Sigma_A}^2}{m_2} \norm{\theta^*}^2 + \frac{\norm{\Sigma_b}^2}{m_2}\right)\\
&= 16\left(\frac{\norm{\Sigma_A}^2}{m_1} +  \norm{A}^2\right) \left(\frac{\norm{\Sigma_A}^2}{m_2} \cdot 0 + \frac{\norm{0}^2}{m_2}\right)\\
&=0, 
\end{align*}
because $\theta^*=0$. Also note that $\Sigma_b=0$, because $\Sigma_b(i)=\sqrt{Var(\phi(i)r)}$, according to the definition in Lemma \ref{thm_item:Llambdasmoooth}, and all the rewards are zero for this problem. This proves the bias term of Impression GTD in Theorem \ref{thm:rates_all} is zero for this example, and thus Impression GTD converges to the optimal solution with a linear rate for Baird counterexample. This matches the empirical results presented in Figure \ref{fig:baird_results}.  

\begin{figure}[t]
\caption{Baird counterexample results.}\label{fig:baird_results}
\centering
\includegraphics[width=1.06\textwidth]{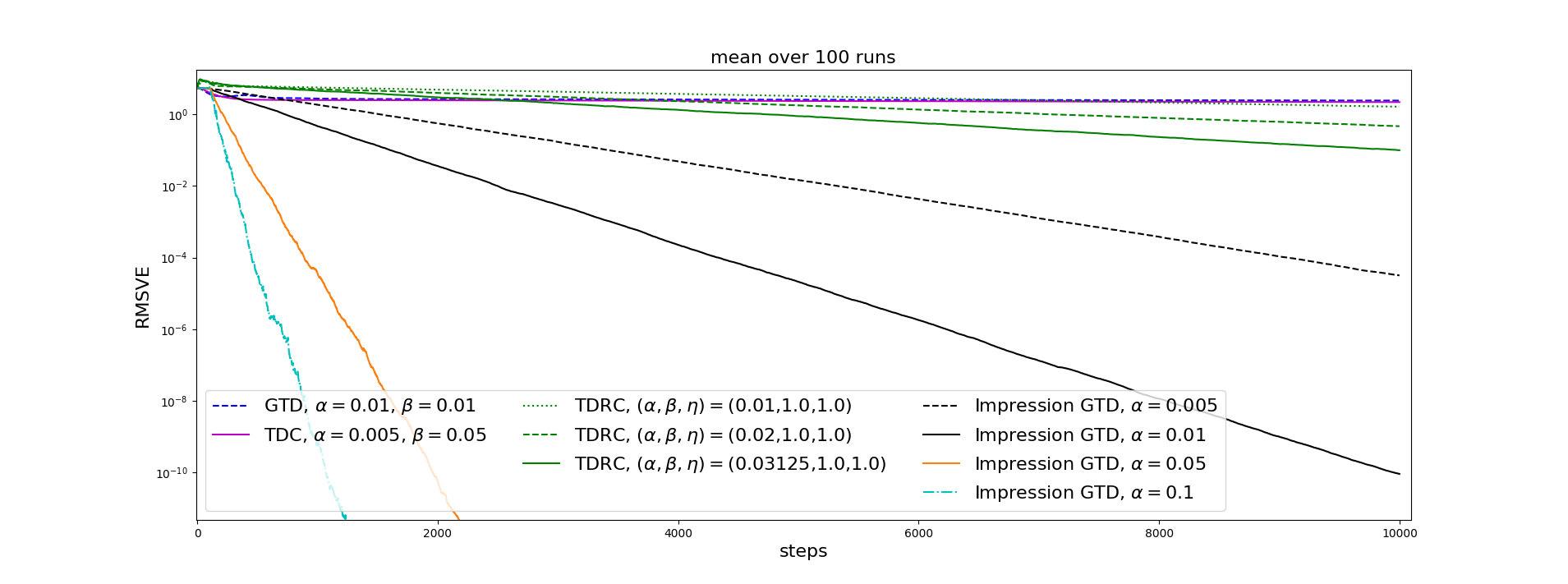}
\end{figure}

\end{document}